\documentclass[12pt, authoryear]{elsarticle}

\usepackage{graphicspath_arxiv}

\usepackage{amsmath, amssymb, color}
\usepackage[amsthm, hyperref]{ntheorem}
\numberwithin{equation}{section}
\usepackage{booktabs, graphicx, lineno, multicol, hyperref, cleveref}
\hypersetup{colorlinks = true, bookmarksnumbered = true}
\modulolinenumbers[5]
\raggedcolumns

\crefname{equation}{}{}

\newtheorem{theorem}{Theorem}
\newtheorem{lemma}{Lemma}
\newtheorem{proposition}{Proposition}
\newtheorem{assumption}{Assumption}
\theoremstyle{definition}
\newtheorem{defn}{Definition}
\newtheorem{example}{Example}
\theoremstyle{remark}
\newtheorem{remark}{Remark}

\newcommand\specialref{}

\def\supp{\mathrm{supp}}
\def\sign{\mathrm{sign}}

\allowdisplaybreaks[1]

\begin{document}

\begin{frontmatter}

\title{Boosting with Structural Sparsity: \\ A Differential Inclusion Approach}

\author[pku]{Chendi Huang}
\ead{cdhuang@pku.edu.cn}

\author[pku]{Xinwei Sun}
\ead{sxwxiaoxiaohehe@pku.edu.cn}

\author[pku]{Jiechao Xiong}
\ead{xiongjiechao@pku.edu.cn}

\author[pku,ust]{Yuan Yao\corref{cor}}
\cortext[cor]{Corresponding author}
\ead{yuany@ust.hk}

\address[pku]{School of Mathematical Science, Peking University, Beijing, 100871, China}
\address[ust]{Department of Mathematics and Division of Biomedical Engineering, Hong Kong University of Science and Technology, Clear Water Bay, Kowloon, Hong Kong SAR, China}

\begin{abstract}
    Boosting as gradient descent algorithms is one popular method in machine learning. In this paper a novel Boosting-type algorithm is proposed based on restricted gradient descent with structural sparsity control whose underlying dynamics are governed by differential inclusions. In particular, we present an iterative regularization path with structural sparsity where the parameter is sparse under some linear transforms, based on variable splitting and the Linearized Bregman Iteration. Hence it is called \emph{Split LBI}. Despite its simplicity, Split LBI outperforms the popular generalized Lasso in both theory and experiments. A theory of path consistency is presented that equipped with a proper early stopping, Split LBI may achieve model selection consistency under a family of Irrepresentable Conditions which can be weaker than the necessary and sufficient condition for generalized Lasso. Furthermore, some $\ell_2$ error bounds are also given at the minimax optimal rates. The utility and benefit of the algorithm are illustrated by several applications including image denoising, partial order ranking of sport teams, and world university grouping with crowdsourced ranking data.
\end{abstract}


\end{frontmatter}

\linenumbers

\section{Introduction}

In this paper, consider the recovery from linear noisy measurements of $\beta^\star\in \mathbb{R}^p$, which satisfies the following structural sparsity that the linear transformation $\gamma^\star := D \beta^\star$ for some $D\in \mathbb{R}^{m\times p}$ has most of its elements being zeros. For a design matrix $X\in \mathbb{R}^{n\times p}$, let 
\begin{equation}
    \label{eq:model}
    y = X\beta^\star + \epsilon,\ \gamma^\star = D\beta^\star \ \left( S = \mathrm{supp}\left( \gamma^\star \right),\ s = |S| \right),
\end{equation}
where $\epsilon\in \mathbb{R}^n $ has independent identically distributed components, each of which has a sub-Gaussian distribution with parameter $\sigma^2$ ($\mathbb{E} [\exp(t \epsilon_i)] \le \exp(\sigma^2 t^2/2)$). In literature the linear transform $D$ has various examples including the Fourier transform, the wavelet transform, or graph gradient operators etc. Here $\gamma^{\star}$ is \emph{sparse}, i.e. $s \ll m$. Given $(y, X, D)$, the purpose is to estimate $\beta^{\star}$ as well as $\gamma^{\star}$, and in particular, recovers the support of $\gamma^{\star}$.

There is a large literature on this problem. Perhaps the most popular approach is the following $\ell_1$-penalized convex optimization problem,
\begin{equation}
    \label{eq:genlasso}
    \arg\min_{\beta} \left( \frac{1}{2n} \left\| y - X\beta \right\|_2^2 + \lambda \left\| D \beta \right\|_1 \right).
\end{equation}
Such a problem can be at least traced back to \citet{ROF92} as a \emph{total variation regularization} for image denoising in applied mathematics; in statistics it is formally proposed by \citet{fusedLASSO} as \emph{fused Lasso}. As $D=I$ it reduces to the well-known \emph{Lasso} \citep{lasso} and different choices of $D$ include many special cases, it is often called \emph{generalized Lasso} \citep{tibshirani_solution_2011} in statistics.

Various algorithms are studied for solving \cref{eq:genlasso} at fixed values of the tuning parameter $\lambda$, most of which is based on the ADMM or Split Bregman using operator splitting ideas (see for examples \citet{SplitBregman,ye_split_2011,wahlberg_admm_2012,ramdas_fast_2014,zhu_augmented_2015} and references therein). To avoid the difficulty in dealing with the structural sparsity in $\| D \beta\|_1$, these algorithms exploit an augmented variable $\gamma$ to enforce sparsity while keeping it close to $D\beta$.

On the other hand, regularization paths are crucial for model selection by computing estimators as functions of regularization parameters. For example, \citet{efron_least_2004} studies the regularization path of standard Lasso with $D=I$, the algorithm in \citet{hoefling_path_2010} computes the regularization path of fused Lasso, and the dual path algorithm in \citet{tibshirani_solution_2011} can deal with generalized Lasso. Recently, \citet{arnold_efficient_2016} discussed various efficient implementations of the the algorithm in \citet{tibshirani_solution_2011}, and the related R package \texttt{genlasso} can be found in CRAN repository. All of these are based on homotopy method of solving convex optimization \cref{eq:genlasso}. 

Our departure here, instead of solving \cref{eq:genlasso}, is to look at an extremely simple yet novel iterative scheme which finds a new regularization path with structural sparsity. We are going to show that it works in a better way than \texttt{genlasso}, in both theory and experiments. 

\subsection{New Algorithm: Split LBI}

Define a loss function which splits $D\beta$ and $\gamma$,
\begin{equation}
    \label{eq:l-def}
    \ell\left( \beta, \gamma \right) := \frac{1}{2n} \left\| y - X \beta \right\|_2^2  + \frac{1}{2\nu} \left\| \gamma - D \beta \right\|_2^2\ \ (\nu > 0).
\end{equation}
Now consider the following iterative algorithm,
\begin{subequations}
    \label{eq:slbi-show}
    \begin{align}
        \label{eq:slbi-show-a}
        \beta_{k+1} &= \beta_k - \kappa \alpha \nabla_{\beta} \ell(\beta_k,\gamma_k),\\
        \label{eq:slbi-show-b}
        z_{k+1} &= z_k - \alpha \nabla_\gamma \ell(\beta_k,\gamma_k),\\
        \label{eq:slbi-show-c}
        \gamma_{k+1} &= \kappa \cdot \mathrm{prox}_{\|\cdot\|_1} (z_{k+1}), 
    \end{align}
\end{subequations}
where the initial choice $z_0 = \gamma_0 = 0 \in \mathbb{R}^m$, $\beta_0 = 0 \in \mathbb{R}^p$, parameters $\kappa>0,\ \alpha>0,\ \nu>0$, and the proximal map associated with a convex function $h$ is defined by $\mathrm{prox}_h(z) = \arg\min_x \|z-x\|^2/2 + h(x)$, which is reduced to the \emph{shrinkage} operator when $h$ is taken to be the $\ell_1$-norm, $\mathrm{prox}_{\|\cdot\|_1} (z)=\mathcal{S} \left( z, 1 \right)$ where
\begin{equation*}
    \mathcal{S}\left( z, \lambda \right) = \mathrm{sign}(z) \cdot \max\left( |z| - \lambda,\ 0 \right)\ (\lambda \geq 0).
\end{equation*}

The algorithm generates a sequence $(\beta_k,\gamma_k)_{k\in \mathbb{N}}$ which defines a discrete regularization path. Iteration \cref{eq:slbi-show-a} has appeared as  \emph{$L_2$-Boost} \citep{BuhYu02} in machine learning and can be traced back to the \emph{Landweber Iteration} in inverse problems \citep{YaoRosCap07} where early stopping regularization is needed against overfitting noise. On the other hand, \cref{eq:slbi-show-b} and \cref{eq:slbi-show-c}, generating a sparse regularization path on $\gamma_k$, is known as the \emph{Linearized Bregman Iteration} (LBI) firstly proposed in \citet{lbi}. Recently in sparse linear regression, \citet{osher_sparse_2016} shows that under nearly the same conditions as standard Lasso, LBI with early stopping may achieve sign consistency but with a less biased estimator than Lasso, and its limit dynamics will reach the bias-free \emph{oracle} estimator which is optimal over all estimators. Equipped with a variable splitting between $D\beta$ and $\gamma$, algorithm \cref{eq:slbi-show} thus combines the $L_2$-Boost of $\beta$ for prediction and LBI of $\gamma$ for sparse structure. Hence in this paper we call \cref{eq:slbi-show} the \emph{Split LBI} or \emph{Boosting with structural sparsity}.

The gap $\| \gamma - D \beta \|_2^2/\nu$ controls the affinity between $D\beta$ and $\gamma$. As $\nu\to 0$, $D\beta=\gamma$ which meets the generalized Lasso constraint; while for a finite $\nu>0$, $D\beta$ is not necessarily sparse. Such an increase in degree of freedom, however, leaves us a new space for improving the model selection consistency, as we shall see in the following experiment and in later part of this paper for a theoretical development. 
 
\subsection{Improved Model Selection in Experiments}

The following example shows that the iterative regularization path \cref{eq:slbi-show} can be more accurate than the regularization path of generalized Lasso, in terms of \emph{Area Under the Curve} (AUC)\protect\footnote{The ``Area Under the Curve'' is the area under the Receiver Operating Characteristic (ROC) Curve, whose definition can be seen for example in \citep{brown_receiver_2006}.} measurement of the order of parameters becoming nonzero in consistent with the ground truth sparsity pattern (higher value of AUC means better performance of variable selection of an algorithm such that true parameters becoming nonzero along the algorithmic regularization path earlier than the null parameters). The following simple experiment illustrates such phenomena by simulations.
\begin{figure}
    \centering
    \includegraphics[width = 0.35\textwidth]{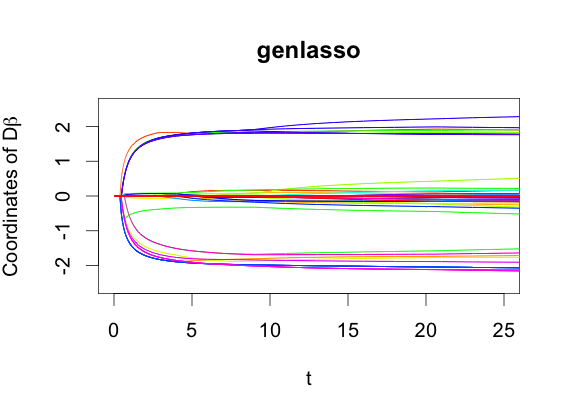}
    \includegraphics[width = 0.35\textwidth]{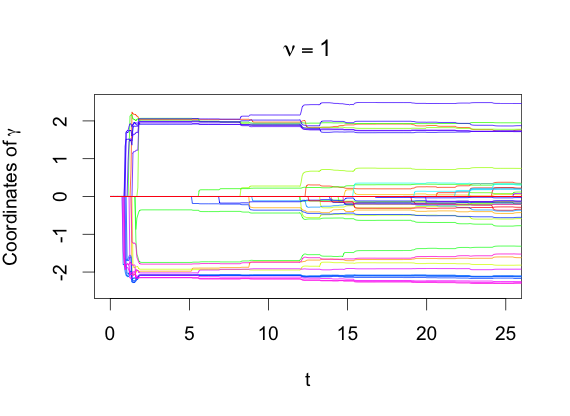}

    \includegraphics[width = 0.35\textwidth]{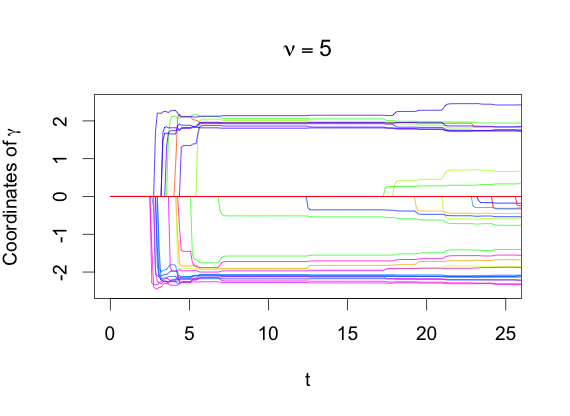}
    \includegraphics[width = 0.35\textwidth]{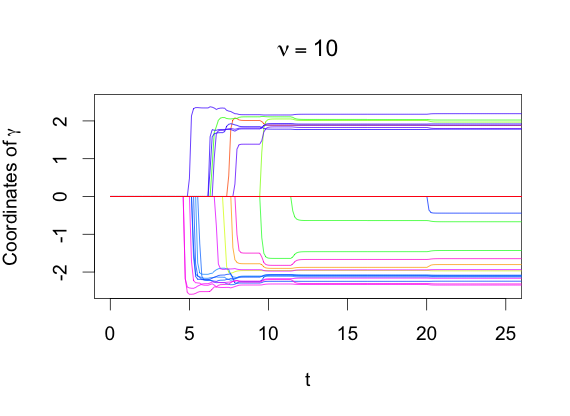}
    \caption{$\{D \beta_\lambda\}\ (t=1/\lambda)$ by \texttt{genlasso} and $\{\gamma_k\}\ (t = k \alpha)$ by Split LBI \cref{eq:slbi-show} with $\nu = 1,5,10$, for 1-D fused Lasso.}
    \label{fig:simu-1dfused-path}
\end{figure}

\begin{example}
    \label{thm:simu-lasso-1dfused-path-auc}
    Consider two problems: standard Lasso and 1-D fused Lasso. In both cases, set $n = p = 50$, and generate $X\in \mathbb{R}^{n\times p}$ denoting $n$ i.i.d. samples from $N(0,I_p)$, $\epsilon\sim N(0,I_n)$, $y = X \beta^{\star} + \epsilon$. $\beta_j^{\star} = 2$ (if $1\le j\le 10$), $-2$ (if $11\le j\le 15$), and $0$ (otherwise). For Lasso we choose $D = I$, and for 1-D fused Lasso we choose $D = (D_1; D_2) \in \mathbb{R}^{(p-1+p)\times p}$ such that $(D_1\beta)_j = \beta_j - \beta_{j+1}$ (for $1\le j\le p - 1$) and $D_2 =I_p$. \Cref{fig:simu-1dfused-path} shows the regularization paths by $\texttt{genlasso}$ ($\{D\beta_\lambda\}$) and by iteration \cref{eq:slbi-show} (linear interpolation of $\{\gamma_k\}$) with $\kappa = 200$ and $\nu\in\{1,5,10\}$, respectively. The generalized Lasso path is in fact piecewise linear with respect to $\lambda$ while we show it along $t=1/\lambda$ for a comparison. Note that the iterative paths exhibit a variety of different shapes depending on the choice of $\nu$. However, in terms of order of those curves entering into nonzero range, these iterative paths exhibit a \emph{better} accuracy than \texttt{genlasso}. \Cref{tab:simu-lasso-1dfused-auc} shows this by the mean AUC of $100$ independent experiments in each case, where the increase of $\nu$ improves the model selection accuracy of Split LBI paths and beats that of generalized Lasso.
\end{example}

\begin{table}[t]
    \caption{Mean AUC (with standard deviation) comparisons where Split LBI \cref{eq:slbi-show} beats \texttt{genlasso}. The first is for the standard Lasso, and the second is for the 1-D fused Lasso in \Cref{thm:simu-lasso-1dfused-path-auc}.}
    \label{tab:simu-lasso-1dfused-auc}
    \begin{minipage}{\textwidth}
        \centering
        \begin{tabular}{cccc}
            \toprule
            \multicolumn{1}{c}{\texttt{genlasso}} & \multicolumn{3}{c}{Split LBI}\\
            \cmidrule(lr){2-4}
            & $\nu = 1$ & $\nu = 5$ & $\nu = 10$ \\
            \cmidrule{2-4}
            $.9426$ & $.9845$ & $.9969$ & $\mathbf{.9982}$\\
            $(.0390)$ & $(.0185)$ & $(.0065)$ & $(\mathbf{.0043})$\\
            \bottomrule 
        \end{tabular}
    \end{minipage}
    \begin{minipage}{\textwidth}
        \centering
        \begin{tabular}{cccc}
            \toprule
            \multicolumn{1}{c}{\texttt{genlasso}} & \multicolumn{3}{c}{Split LBI}\\
            \cmidrule(lr){2-4}
            & $\nu = 1$ & $\nu = 5$ & $\nu = 10$ \\
            \cmidrule{2-4}
            $.9705$ & $.9955$ &  $.9996$ & $\mathbf{.9998}$\\
            $(.0212)$ & $(.0056)$ & $(.0014)$ & $(\mathbf{.0009})$\\
            \bottomrule
        \end{tabular}
    \end{minipage}
\end{table}

\emph{Why does Split LBI perform better in model selection than generalized Lasso?} Some limit dynamics of algorithm \cref{eq:slbi-show} actually shed light on the cause. 



\subsection{Limit Differential Inclusions of Split LBI}
\label{sec:slbiss}

Below we are going to derive several limit dynamics of Split LBI, which are differential inclusions and lead to explanations on how our algorithm might improve over generalized Lasso. 

First of all, noting by the following Moreau Decomposition 
\begin{align}
    \label{eq:var-subst}
    \rho\in \partial \left\| \gamma \right\|_1,\ z = \rho + \gamma / \kappa & &\Longleftrightarrow & &\gamma = \kappa \mathcal{S}(z,1),\ \rho = z - \mathcal{S}(z,1),
\end{align}
the Split LBI \eqref{eq:slbi-show} can be rewritten as,
\begin{subequations}
    \label{eq:slbi}
    \begin{align}
        \label{eq:slbi-a}
        \beta_{k+1} / \kappa &= \beta_k / \kappa - \alpha \nabla_\beta \ell\left( \beta_k, \gamma_k \right),\\
        \label{eq:slbi-b}
        \rho_{k+1} + \gamma_{k+1} / \kappa &= \rho_k + \gamma_k / \kappa - \alpha \nabla_\gamma \ell\left( \beta_k, \gamma_k \right),\\
        \label{eq:slbi-c}
        \rho_k &\in \partial \left\| \gamma_k \right\|_1,
    \end{align}
\end{subequations}
where $\rho_0 = \gamma_0 = 0 \in \mathbb{R}^m$, $\beta_0 = 0 \in \mathbb{R}^p$. 

Now taking $\rho(k\alpha) = \rho_k$, $\gamma(k\alpha) = \gamma_k$, $\beta(k\alpha) = \beta_k$, and $\alpha\to 0$, \cref{eq:slbi} is a forward Euler discretization of the following limit dynamics, called \emph{Split Linearized Bregman Inverse Scale Space (Split LBISS)} here. 
\begin{defn}[Split LBISS]
For $\alpha \to 0$, define the following differential inclusion as the limit dynamics of Split LBI,
\begin{subequations}
    \label{eq:slbiss}
    \begin{align}
        \label{eq:slbiss-a}
        \dot{\beta}(t) / \kappa &= - \nabla_\beta \ell\left( \beta(t), \gamma(t) \right),\\
        \label{eq:slbiss-b}
        \dot{\rho}(t) + \dot{\gamma}(t) / \kappa &= - \nabla_\gamma \ell\left( \beta(t), \gamma(t) \right),\\
        \label{eq:slbiss-c}
        \rho(t) &\in \partial \left\| \gamma(t) \right\|_1,
    \end{align}
\end{subequations}
where $\rho(t), \beta(t), \gamma(t)$ are right continuously differentiable, with $\dot{\rho}(t), \dot{\beta}(t), \dot{\gamma}(t)$ denoting the right derivatives in $t$ of $\rho(t), \beta(t), \gamma(t)$ respectively, and $\rho(0) = \gamma(0) = 0 \in \mathbb{R}^m$, $\beta(0) = 0 \in \mathbb{R}^p$. 
\end{defn}

Next taking $\kappa\to \infty$, we reach the following dynamics called \emph{Split Inverse Scale Space (Split ISS)} in this paper. 

\begin{defn}[Split ISS]
For $\kappa\to \infty$ and $\alpha\to 0$, define the differential inclusion, 
\begin{subequations}
    \label{eq:siss}
    \begin{align}
        \label{eq:siss-a}
        0 &= - \nabla_\beta \ell\left( \beta(t), \gamma(t) \right),\\
        \label{eq:siss-b}
        \dot{\rho}(t) &= - \nabla_\gamma \ell\left( \beta(t), \gamma(t) \right),\\
        \label{eq:siss-c}
        \rho(t) &\in \partial \left\| \gamma(t) \right\|_1,
    \end{align}
\end{subequations}
where $\rho(t)$ is right continuously differentiable, $\beta(t), \gamma(t)$ are right continuous, and $\rho(0) = \gamma(0) = 0 \in \mathbb{R}^m$, $\beta(0) = 0 \in \mathbb{R}^p$. Solving $\beta(t)$ in \cref{eq:siss-a} and plugging it into \cref{eq:siss-b}, \cref{eq:siss} can be reduced to 
\begin{subequations}
    \label{eq:siss-gamma}
    \begin{align}
    	\label{eq:siss-gamma-a}
        \dot{\rho}(t) &= - \Sigma^{1/2} ( \Sigma^{1/2} \gamma(t) - \Sigma^{\dag 1/2} D A^{\dag} X^{*} y),\\
        \rho(t) &\in \partial \left\| \gamma(t) \right\|_1,
    \end{align}
\end{subequations}
where $\Sigma$ and $A$ are given by 
\begin{equation}
    \label{eq:ASigma-def}
    \Sigma = \Sigma(\nu) := \left( I - D A^{\dag} D^T \right) / \nu,\ \text{and}\ A = A(\nu) = \nu X^{*} X + D^T D.
\end{equation}
\end{defn}

In fact by \cref{eq:siss-a} we have
\begin{equation*}
    \beta(t) = \arg\min_\beta \ell\left( \beta, \gamma(t) \right) = A^{\dag} \left( \nu X^{*} y + D^T \gamma(t) \right),
\end{equation*}
where $A = \nu X^{*} X + D^T D$. Substituting this for $\beta(t)$ in \cref{eq:siss-b} and noting \cref{eq:Sigma-DAdagX} ($DAX^{*} = \Sigma^{1/2} \Sigma^{\dag 1/2} D A X^{*}$), we thus get \cref{eq:siss-gamma-a}. 

\begin{remark}
    Note that \cref{eq:siss-gamma} coincides with the differential inclusion proposed in Chapter 8 of \citet{moeller_multiscale_2012} where the authors introduced it in a different way. The existence and uniquess of solutions of Split LBISS and Split ISS will be characterized precisely in \Cref{sec:pathproperty}.
\end{remark}

Now consider the particular case of the standard Lasso where $D = I$ and $\Sigma(\nu) = X^{*}(I + \nu X X^{*})^{-1} X$. Hence as $\nu \to 0$, we have $\Sigma(\nu) \to X^{*} X$ and \cref{eq:siss-gamma} leads to the standard Inverse Scale Space (ISS) dynamics studied in \citep{osher_sparse_2016} by identifying $\beta=\gamma$. 

\begin{proposition}
    \label{prop:siss-iss}
    Let $D=I$ and $\nu \to 0$, then $\gamma=\beta$ and \eqref{eq:siss} reduces to
    \begin{subequations}\label{eq:iss}
        \begin{align}
            \dot{\rho}(t) &= - X^{*} ( X \beta(t) - y),\\
            \rho(t) &\in \partial \left\| \beta(t) \right\|_1,
        \end{align}
    \end{subequations}
    with the same notations as above. 
\end{proposition}

A fundamental path consistency problem is the following. 

\textbf{Model Selection Consistency}: Under what conditions there exists a point $\bar{\tau}$ (or $\bar{k}$) such that $\supp(\gamma(\bar{\tau})) = S$ (or $\supp(\gamma_{\bar{k}}) = S$), or more specifically the so called \emph{sign-consistency} holds, $\sign(\gamma(\bar{\tau})) = \sign(\gamma^\star)$ (or $\sign(\gamma_{\bar{k}})=\sign(\gamma^\star)$, respectively)? 

Comparing the reduced Split ISS \cref{eq:siss-gamma} with the ISS \cref{eq:iss}, one can see that $\Sigma(\nu)$ plays a similar role as $X^{*} X$. For the special case that $D=I$ and $\nu\to 0$, \citet{osher_sparse_2016} shows that under nearly the same conditions as Lasso, ISS \cref{eq:iss} achieves model selection consistency but with the unbiased oracle estimator which is better than Lasso. Here an unbiased estimator means the expectation of the estimator equals to the ground truth and Lasso is well-known to be biased. In fact, under a so called \emph{Irrepresentable Condition (IRR)} on $X^* X$, ISS \cref{eq:iss} is guaranteed to evolve before the stopping time on the \emph{oracle subspace} whose coordinate index is within the support set $S$ of the true parameter, i.e. no false positive. Similarly the Lasso regularization path also has no false positive under the same condition. Moreover if the signal is strong enough, the Lasso may pick up an estimator which is \emph{sign-consistent} yet \emph{biased}, while the ISS path with an early stopping may reach the oracle estimator which is both \emph{sign-consistent} and \emph{unbiased}.  

For the comparison with generalized Lasso, the Irrepresentable Condition on $\Sigma(\nu)$ will replace that on $X^* X$, where the additional degree of freedom provided by $\nu$ enables us a chance to beat generalized Lasso. 

Model selection and estimation consistency of generalized Lasso \cref{eq:genlasso} has been studied in previous work. \citet{sharpnack_sparsistency_2012} considered the model selection consistency of the edge Lasso, with a special $D$ in \cref{eq:genlasso}, which has applications over graphs. \citet{liu_guaranteed_2013} provides an upper bound of estimation error by assuming the design matrix $X$ is a Gaussian random matrix. In particular, \citet{vaiter_robust_2013} proposes a general condition called \emph{Identifiability Criterion} (IC) for sign consistency. \citet{lee_model_2013} establishes a general framework for model selection consistency for penalized M-estimators, proposing an Irrepresentable Condition which is equivalent to IC from \citet{vaiter_robust_2013} under the specific setting of \cref{eq:genlasso}. In fact both of these conditions are sufficient and necessary for structural sparse recovery by generalized Lasso \cref{eq:genlasso} in a certain sense. 

In this paper, we shall present a new family of the Irrepresentable Condition depending on $\Sigma(\nu)$, under which model selection consistency can be established for both Split ISS \cref{eq:siss} and Split LBI \cref{eq:slbi-show}. In particular, this condition family can be strictly \emph{weaker} than IC as the parameter $\nu$ grows, which sheds light on the superb performance of Split LBI we observed in the experiment above. Therefore, the benefits of exploiting Split LBI \cref{eq:slbi-show} not only lie in its algorithmic simplicity, but also provide a possibility of theoretical improvement on model selection consistency. 

\begin{figure} 
    \centering
    \includegraphics[width = 0.6\textwidth]{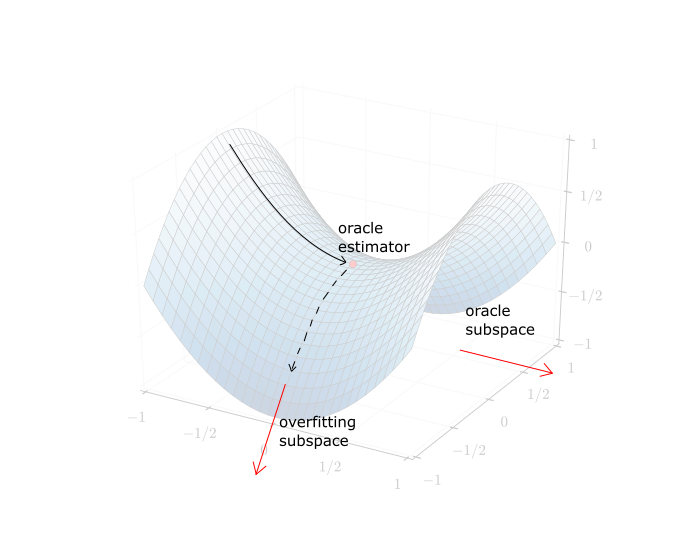}
    \caption{Illustration of global behaviour of dynamics in this paper.}
    \label{fig:saddle}
\end{figure}
Roughly speaking, the global picture of our theoretical development is illustrated in \Cref{fig:saddle}:
\begin{enumerate}
    \item Equipped with the Irrepresentable Condition on $\Sigma(\nu)$, all the dynamics (differential inclusions and the discrete iterations) evolves in a subspace of estimators whose support set lies in the that of the true parameter, whence the subspace is called the oracle subspace here; 
    \item Further enhanced by a restricted strongly convexity, along the paths of these dynamics the loss is rapidly decreasing at an exponential speed, firstly approaching a saddle point lying the oracle estimator then flowing away; 
    \item Early stopping regularization is designed here to stop the dynamics around the saddle point to pick up an estimator close to the oracle before escaping to overfitted solutions; 
    \item If the signal is strong enough such that the true parameters are all of sufficiently large magnitudes, such a good estimator is guaranteed to recover the sparsity pattern of the ground truth. 
\end{enumerate}
In the sequel, we are going to elaborate them in a precise way.





\subsection{Paper Organization} 

This paper is a long version of a conference report \citep{huang_split_2016} which states the main results about the discrete algorithm \cref{eq:slbi-show} without proofs together with part of the experiments.  
The full version here is organized as follows: \Cref{sec:comparison} presents the Irrepresentable Condition together with other assumptions for Split ISS and LBI, and shows that it can be strictly weaker than IC, the necessary and sufficient condition for model selection consistency of generalized Lasso; Some basic properties of dynamic paths are presented in \Cref{sec:pathproperty}, including the existence and uniqueness of differential inclusion solutions, as well as the non-increasing loss along the paths; \Cref{sec:pathconsistency} collects the path consistency results for both differential inclusions and the discrete algorithm; A brief description of proof ideas for these results are presented in \Cref{sec:ideas} with specific details left in appendices; \Cref{sec:exp} collects three more applications, including image denoising, partial order (group) estimate in sports and crowdsourced university ranking; Conclusion is given in \Cref{sec:conclusion}; Appendices collect all the remaining proofs in this paper. 

%

\subsection{Notation}
\label{sec:notation}

For matrix $Q$ with $m$ rows ($D$ for example) and $J\subseteq \{1,2,\ldots,m\}$, let $Q_J = Q_{J, \cdot}$ be the submatrix of $Q$ with rows indexed by $J$. However, for $Q\in \mathbb{R}^{n\times p}$ ($X$ for example) and $J\subseteq \{1,2,\ldots,p\}$, let $Q_J = Q_{\cdot, J}$ be the submatrix of $Q$ with columns indexed by $J$, abusing the notation.

$P_{L}$ denotes the projection matrix onto a linear subspace $L$, Let $L_1 + L_2 := \{ \xi_1 + \xi_2:\ \xi\in L_1,\ \xi\in L_2\}$ for subspaces $L_1, L_2$. For a matrix $Q$, let $Q^{\dag}$ denotes the Moore-Penrose pseudoinverse of $Q$, and we recall that $Q^{\dag} = (Q^T Q)^{\dag} Q^T$. Let $\lambda_{\max}(Q), \lambda_{\min}(Q), \lambda_{\min, +}(Q)$ denotes the largest singular value, the smallest singular value, the smallest nonzero singular value of $Q$, respectively. For symmetric matrices $P$ and $Q$, $Q\succ P$ (or $Q\succeq P$) means that $Q-P$ is positive (semi-)definite, respectively. Let $Q^{*} := Q^T / n$. Sometimes we use $\langle a, b\rangle := a^T b$, denoting the inner product between vectors $a,b$. Also, for tidiness in some situations, we write $(Q_1; Q_2) := (Q_1^T, Q_2^T)^T$.

\section{Assumptions and Comparisons}
\label{sec:comparison}

\subsection{Basic Assumptions}
\label{sec:basic-assump}

We need some convention, definitions and assumptions. For the identifiability of $\beta^\star$, we can assume that $\beta^{\star}$ and its estimators of interest are restricted in
\begin{equation*}
    \mathcal{L} := (\ker(X)\cap \ker(D))^{\perp} = \mathrm{Im} \left( X^T \right) + \mathrm{Im} \left( D^T \right),
\end{equation*}
since replacing $\beta^{\star}$ with ``the projection of $\beta^{\star}$ onto $\mathcal{L}$'' does not change the model. We also have $\beta^{\star} \in \mathcal{M}$, where $\mathcal{M}$ is the \emph{model subspace} defined as
\begin{equation*}
    \mathcal{M} := \left\{ \beta:\ D_{S^c} \beta = 0 \right\}.
\end{equation*}
Note that $\ell(\beta, \gamma)$ is quadratic, and we can define its Hessian matrix
\begin{equation}
    \label{eq:H-def}
    H = H(\nu) := \nabla^2 \ell\left( \beta, \gamma \right) \equiv \begin{pmatrix} X^{*} X + D^T D / \nu & - D^T / \nu \\ - D / \nu & I_m / \nu \end{pmatrix}
\end{equation}
(sometimes we use the notation $H(\nu)$ stressing the dependence on $\nu$). Now we assume that there exist constants $\lambda_D, \Lambda_D, \Lambda_X > 0$ satisfying
\begin{subequations}
    \label{eq:const-def}
    \begin{gather}
        \min \left( \lambda_{\min,+}\left( D \right),\ \lambda_{\min,+}\left( D_{S^c} \right) \right)\ge \lambda_D,\\
        \Lambda_{\max}\left( D \right) \le \Lambda_D,\\
        \Lambda_{\max}\left( X^{*} X \right) \le \Lambda_X^2.
    \end{gather}
\end{subequations}
Besides, we consider the following assumption. 

\begin{assumption}[\textnormal{Restricted Strong Convexity (RSC)}]
    \label{thm:rsc}
    There exists a constant $\lambda > 0$ such that
    \begin{equation}
        \label{eq:rsc}
        \beta^T X^{*} X \beta \ge \lambda \left\| \beta \right\|_2^2,\ \text{for any}\ \beta \in \mathcal{L} \cap \mathcal{M}.
    \end{equation}
\end{assumption}

\begin{remark}
    When $\mathcal{L} = \mathbb{R}^p$, i.e. there is only one $\beta^{\star}$ satisfying \cref{eq:model}, \Cref{thm:rsc} is the same as that proposed by \citet{lee_model_2013}. Specifically, when $D = I$, \Cref{thm:rsc} reduces to $X_S^{*} X_S \succeq \lambda I$, the usual RSC in Lasso.
\end{remark}

\begin{proposition}
    \label{thm:rsc-nu}
    If there exists $C > 0,\ \nu > 0$, such that
    \begin{equation}
        \label{eq:rsc-nu}
        \left( \beta^T, \gamma_S^T \right) \cdot H_{(\beta, S), (\beta, S)}(\nu) \cdot \begin{pmatrix} \beta\\ \gamma_S \end{pmatrix} \ge \frac{C}{1 + \nu} \left\| \begin{pmatrix} \beta\\ \gamma_S \end{pmatrix} \right\|_2^2\ \left( \beta\in \mathcal{L},\ \gamma_S \in \mathbb{R}^s \right).
    \end{equation}
    then \cref{eq:rsc} holds. Conversely, if \cref{eq:rsc} holds, then there exists $C > 0$, such that for all $\nu > 0$, \cref{eq:rsc-nu} holds.
\end{proposition}

\begin{remark}
    Traditional RSC for the partial Lasso $\min_{\beta, \gamma} (\ell(\beta, \gamma) + \lambda \| \gamma \|_1)$ requires $\ell$ to be restricted strongly convex, i.e. strongly convex restricted on $\mathcal{N} := \mathcal{L}\oplus \mathbb{R}^s\oplus \{0\}^{p-s}$ which is the sparse subspace corresponding to the support of $\gamma^{\star}$). \Cref{thm:rsc-nu} implies that, \Cref{thm:rsc} is necessary for $\ell$ to be restricted strongly convex for \emph{a specific} $\nu > 0$ (note that $\ell$ depends on $\nu$), and also sufficient for $\ell$ to be restricted strongly convex for \emph{all} $\nu > 0$.
\end{remark}

\begin{remark}
    Let us further note the rate $C/(1+\nu)$ in \cref{eq:rsc-nu}. When $\nu \rightarrow 0$, it approaches $C$, a constant independent with $\nu$. When $\nu \rightarrow +\infty$, the rate $C/(1+\nu) \sim \nu^{-1}$ is \emph{the best possible}, since $\| H \|_2 \lesssim \nu^{-1}$ by \cref{eq:H-upper}.
\end{remark}

\begin{assumption}[\textnormal{Irrepresentable Condition ($\nu$) (IRR($\nu$))}]
    \label{thm:irr-nu}
    There exists a constant $\eta \in (0, 1]$ such that
    \begin{equation}
        \label{eq:irr-nu}
        \sup_{\rho \in [-1,1]^s} \left\| H_{S^c, (\beta, S)}(\nu) H_{(\beta, S), (\beta, S)}(\nu)^{\dag} \cdot \begin{pmatrix} 0_p\\ \rho \end{pmatrix} \right\|_{\infty} \le 1 - \eta.
    \end{equation}
\end{assumption}

\begin{remark}
    \Cref{thm:irr-nu} actually concerns a family of assumptions with varing $\nu$. However, practically we only require that IRR($\nu$) holds for the specific $\nu$ used in the algorithm of Split LBI.
\end{remark}

\begin{remark}
    \Cref{thm:irr-nu} directly generalizes the Irrepresentable Condition from standard Lasso \citep{ZhaYu06} and OMP/BP \citep{Tropp04}, to the partial Lasso: $\min_{\beta,\gamma} (\ell \left( \beta, \gamma \right) + \lambda\|\gamma\|_1)$. This type conditions are firstly proposed by \citep{Tropp04} for Orthogonal Matching Pursuit (OMP) and Basis Pursuit (BP) in noise free case, in the name of Exact Recovery Condition; later \citet{CaiWan11} extends it to OMP in noisy measurement; \citet{ZhaYu06} establishes it for model selection consistency of Lasso under Gaussian noise while \citet{Wainwright09} extends it to the sub-gaussian; \citet{YuaLin07} and \citet{Zou06} also independently present this condition in other studies. Here following the standard Lasso case \citep{Wainwright09}, one version of the Irrepresentable Condition should be
    \begin{equation*}
        \left\| H_{S^c, (\beta, S)}(\nu) H_{(\beta, S), (\beta, S)}(\nu)^{\dag} \cdot \rho_{(\beta,S)}^{\star} \right\|_{\infty} \le 1 - \eta,\ \text{where}\ \rho_{(\beta,S)}^{\star} = \begin{pmatrix} 0_p\\ \rho_S^{\star} \end{pmatrix}.
    \end{equation*}
    $\rho_{(\beta,S)}^{\star}$ is the value of gradient (subgradient) of $\ell_1$ penalty function $\|\cdot\|_1$ on $(\beta^{\star}; \gamma_S^{\star})$. Here $\rho_{\beta}^{\star} = 0_p$, because $\beta$ is not assumed to be sparse and hence is not penalized. \Cref{thm:irr-nu} slightly strengthens this by a supremum over $\rho$, for uniform sparse recovery independent to a particular sign pattern of $\gamma^\star$. 
\end{remark}

\subsection{Some Equivalent Assumptions}

Recall that in order to obtain path consistency results of standard LBISS and LBI in \citet{osher_sparse_2016}, they propose Restricted Strong Convexity (RSC) and Irrpresentable Condition (IRR) based on their $X^{*} X$, and these assumptions are actually the same as those for Lasso. In a contrast, for Split LBISS and Split LBI, we can propose assumptions based on $\Sigma(\nu)$, i.e. $\Sigma_{S,S}(\nu)$ is positive definite, and $\|\Sigma_{S^c,S}(\nu) \Sigma_{S,S}(\nu)^{-1}\|_{\infty} \le 1 - \eta$. These assumptions actually prove to be equivalent with \Cref{thm:rsc} and \ref{thm:irr-nu} as follows.

\begin{proposition}
    \label{thm:rsc-nu-prime}
    If There exists $C > 0,\ \nu_0 > 0$ such that \cref{eq:rsc-nu} holds for $\nu = \nu_0$, then there exists $C' > 0$ such that for all $\nu > 0$,
    \begin{equation}
        \label{eq:rsc-nu-prime}
        \Sigma_{S,S}(\nu) \succeq \frac{C'}{1 + \nu} I.
    \end{equation}
    Conversely, if there exists $C' > 0,\ \nu_0 > 0$ such that \cref{eq:rsc-nu-prime} holds for $\nu = \nu_0$, then there exists $C > 0$ such that for all $\nu > 0$, \cref{eq:rsc-nu} holds.
\end{proposition}

\begin{proposition}
    \label{thm:irr-nu-prime}
    Under \Cref{thm:rsc}, the left hand side of \cref{eq:irr-nu} in \Cref{thm:irr-nu} becomes $\| \Sigma_{S^c, S}(\nu) \Sigma_{S,S}(\nu)^{-1} \|_{\infty}$, and \cref{eq:irr-nu} is equivalent to
    \begin{equation}
        \label{eq:irr-nu-prime}
        \left\| \Sigma_{S^c, S}(\nu) \Sigma_{S,S}(\nu)^{-1} \right\|_{\infty} \le 1 - \eta.
    \end{equation}
\end{proposition}

\begin{remark}
    From \Cref{thm:rsc-nu} and \ref{thm:irr-nu-prime}, $\Sigma$ seems to be closely related to $H$, which is truly the case. In fact, $\Sigma$ is the Schur complement of $H_{\beta,\beta}$ in $H$.
\end{remark}

\subsection{Comparison Theorem on the Irrepresentable Condition}

We present a comparison theorem showing that IRR($\nu$) can be weaker than IC, a necessary and sufficient for model selection consistency of generalized Lasso \citep{vaiter_robust_2013}. Define $\mathrm{irr}(\nu)$ as the left hand side of \cref{eq:irr-nu} (or equivalently the left hand side of \cref{eq:irr-nu-prime}, due to \Cref{thm:irr-nu-prime}), and
\begin{equation*}
    \mathrm{irr}(0) := \lim_{\nu\rightarrow 0} \mathrm{irr}(\nu),\ \mathrm{irr}(\infty) := \lim_{\nu \rightarrow +\infty} \mathrm{irr}(\nu).
\end{equation*}
Let $W$ be a matrix whose columns form an orthogonal basis of $\ker(D_{S^c})$, and define
\begin{gather*}
    \Omega^S := \left( D_{S^c}^{\dag} \right)^T \left( X^{*} X W \left( W^T X^{*} X W \right)^{\dag} W^T - I \right) D_S^T,\\
    \mathrm{ic}_0 := \left\| \Omega^S \right\|_{\infty},\ \mathrm{ic}_1 := \min_{u\in \ker\left( D_{S^c}^T \right)} \left\| \Omega^S \mathrm{sign} \left( D_S \beta^{\star} \right) - u \right\|_{\infty}.
\end{gather*}
\citet{vaiter_robust_2013} proved the sign consistency of the generalized Lasso estimator of \cref{eq:genlasso} for specifically chosen $\lambda$, under the assumption $\mathrm{ic}_1 < 1$. As we shall see later, the same conclusion holds for our algorithm under the assumption $\mathrm{irr}(\nu) \le 1 - \eta$. \emph{Which assumption is weaker to be satisfied?} The following theorem, with proof in \ref{sec:proof-ic-compare}, answers this.

\begin{theorem}[\textnormal{Comparisons between IRR($\nu$) and IC}]
    \label{thm:ic-compare}
    \
    \begin{enumerate}
        \item
            $\mathrm{ic}_0 \ge \mathrm{ic}_1$.
        \item
            $\mathrm{irr}(0)$ exists, and $\mathrm{irr}(0) = \mathrm{ic}_0$.
        \item
            $\mathrm{irr}(\infty)$ exists, and $\mathrm{irr}(\infty) = 0$ if and only if $\ker(X) \subseteq \ker(D_S)$.
    \end{enumerate}
\end{theorem}

\begin{figure}
    \centering
    \includegraphics[width = 0.7\textwidth]{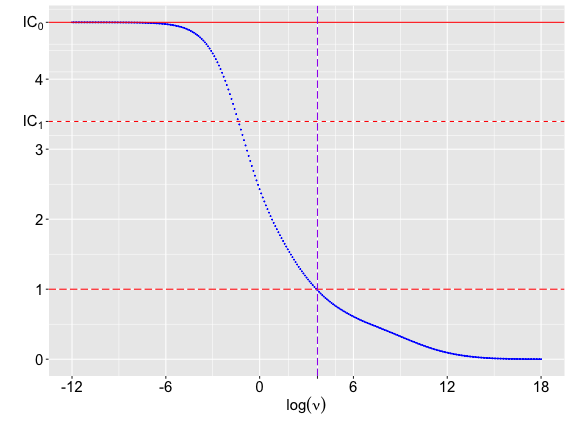}
    \caption{A comparison between our family of Irrepresentable Conditions (IRR($\nu$)) and IC in \citet{vaiter_robust_2013}, with log-scale horizontal axis. As $\nu$ grows, $\mathrm{irr}(\nu)$ can be significantly smaller than $\mathrm{ic}_0$ and $\mathrm{ic}_1$, so that our model selection condition is easier to be met!}
    \label{fig:simu-1dfused-ic-compare}
\end{figure}

From this comparison theorem with a design matrix $X$ of full column rank, as $\nu$ grows, $\mathrm{irr}(\nu)<\mathrm{ic}_1\leq \mathrm{ic}_0$, hence \Cref{thm:irr-nu} is weaker than IC. Now recall the setting of \Cref{thm:simu-lasso-1dfused-path-auc} where $\ker(X)=0$ generically. In \Cref{fig:simu-1dfused-ic-compare}, the (solid and dashed) horizontal red lines denote $\mathrm{ic}_0, \mathrm{ic}_1$, and we see the blue curve denoting $\mathrm{irr}(\nu)$ approaches $\mathrm{ic}_0$ when $\nu \rightarrow 0$ and approaches $0$ when $\nu \rightarrow +\infty$, which illustrates \Cref{thm:ic-compare} (here each of $\mathrm{ic}_0, \mathrm{ic}_1, \mathrm{irr}(\nu)$ is the mean of $100$ values calculated under $100$ generated $X$'s). Although $\mathrm{irr}(0) = \mathrm{ic}_0$ is slightly larger than $\mathrm{ic}_1$, $\mathrm{irr}(\nu)$ can be significantly smaller than $\mathrm{ic}_1$ if $\nu$ is not tiny. On the right side of the vertical line, $\mathrm{irr}(\nu)$ drops below $1$, indicating that \Cref{thm:irr-nu} is satisfied while IC fails.

\begin{remark}
    Despite that \Cref{thm:ic-compare} suggests to adopt a large $\nu$, $\nu$ can not be arbitrarily large, elsewise $C/(1+\nu)$ in \cref{eq:rsc-nu} is small and $\ell$ becomes ``flat'', which will deteriorates the estimator in terms of $\ell_2$ error to be shown later.
\end{remark}

\section{Basic Properties of Paths}
\label{sec:pathproperty}

The following theorem establishes the solution existence as well as uniqueness of Split ISS and Split LBISS, in almost the same way as \citet{osher_sparse_2016}. The proof is given in \ref{sec:proof-slb-basic-properties}.

\begin{theorem}[\textnormal{Existence and uniqueness of solutions}]
    \label{thm:slbiss-exi-uni}
    \
    \begin{enumerate}
        \item
            As for Split ISS \cref{eq:siss}, assume that $\rho(t)$ is right continuously differentiable and $\beta(t), \gamma(t)$ is right continuous. Then a solution exists for $t\ge 0$, with piecewise linear $\rho(t)$ and piecewise constant $\beta(t), \gamma(t)$. Besides, $\rho(t)$ is unique. If additionally $\Sigma_{S(t), S(t)} \succ 0$ for $0\le t\le \tau$, where $\Sigma$ is defined in \cref{eq:ASigma-def} and $S(t) := \mathrm{supp}(\gamma(t))$, then $\beta(t), \gamma(t)$ are unique for $0\le t\le \tau$.
        \item
            As for Split LBISS \cref{eq:slbiss}, assume that $\rho(t), \beta(t)$ are right continuously differentiable. Then there is a unique solution for $t\ge 0$.
    \end{enumerate}
\end{theorem}

The following theorem states that along the solution path of either differential inclusions or iterative algorithms, the loss function is always non-increasing. Its proof is provided in \ref{sec:proof-slb-basic-properties}.

\begin{theorem}[\textnormal{Non-increasing loss along the paths}]
    \label{thm:slb-l-dec}
    Consider the loss function $\ell$ defined in \cref{eq:l-def}.
    \begin{enumerate}
        \item
            For a solution $(\rho(t), \beta(t), \gamma(t))$ of Split ISS \cref{eq:siss}, $\ell(\beta(t), \gamma(t))$ is non-increasing in $t$.
        \item
            For a solution $(\rho(t), \beta(t), \gamma(t))$ of Split LBISS \cref{eq:slbiss}, $\ell(\beta(t), \gamma(t))$ is non-increasing in $t$.
        \item
            For a solution $(\rho_k, \beta_k, \gamma_k)$ of Split LBI \cref{eq:slbi}, $\ell(\beta_k, \gamma_k)$ is non-increasing in $k$, if
            \begin{equation}
                \label{eq:kappastep-upper}
                \kappa \alpha \| H \|_2 \le 2.
            \end{equation}
            Moreover, $\left\| H \right\|_2 \le 2 \left( 1 + \nu \Lambda_X^2 + \Lambda_D^2 \right) / \nu$ holds, so \cref{eq:kappastep-upper} holds if
            \begin{equation}
                \label{eq:kappastep-upper-2}
                \kappa \alpha \le \nu / (1 + \nu \Lambda_X^2 + \Lambda_D^2).
            \end{equation}
    \end{enumerate}
\end{theorem}

\section{Path Consistency of Split LBISS and Split LBI}
\label{sec:pathconsistency}

\subsection{Consistency of Split LBISS}
\label{sec:slbiss-cstc}

The following theorem, with proof in \ref{sec:proof-slbiss-cstc}, says that under \Cref{thm:rsc} and \ref{thm:irr-nu}, Split LBISS will automatically evolve in the ``\emph{oracle}'' subspace (unknown to us) restricted within the support set of $(\beta^\star,\gamma^\star)$ before leaving it, and if the signal parameters is strong enough, sign consistency will be reached. Moreover, $\ell_2$ error bounds on $\gamma(t)$ and $\beta(t)$ are given.

\begin{theorem}[Consistency of Split LBISS]
    \label[theorem]{thm:slbiss-cstc}
    Under \Cref{thm:rsc} and \ref{thm:irr-nu}, define $\lambda_H = C / (1 + \nu)$ (from \cref{eq:rsc-nu}) and suppose that $\kappa$ is large so that
    \begin{multline}
        \label{eq:kappa-lower}
        \kappa \ge \frac{4}{\eta} \left( 1 + \frac{1}{\lambda_D} + \frac{\Lambda_X}{\lambda_1 \lambda_D} \right) \left( 1 + \sqrt{\frac{2 \left( 1 + \nu \Lambda_X^2 + \Lambda_D^2 \right)}{\lambda_H \nu}} \right) \\
        \cdot \left( (1 + \Lambda_D) \left\| \beta^{\star} \right\|_2 + \frac{2\sigma}{\lambda_H} \left( \frac{\Lambda_X}{\lambda_D} + \frac{\Lambda_X}{\lambda_D^2} + \frac{\lambda_H \lambda_D^2 + \Lambda_X^2}{\lambda_1 \lambda_D^2} \right) \right),
    \end{multline}
    Let
    \begin{equation}
        \label{eq:tau-bar-def}
        \bar{\tau} := \frac{\eta}{8\sigma} \cdot \frac{\lambda_D}{\Lambda_X} \sqrt{\frac{n}{\log m}}.
    \end{equation}
    Then with probability not less than $1 - 6/m - 3\exp(-4n/5)$, we have all the following properties.
    \begin{enumerate}
        \item
            \emph{No-false-positive}: The solution has no false-positive, i.e. $\mathrm{supp}(\gamma(t)) \subseteq S$, for $0 \le t \le \overline{\tau}$.
        \item
            \emph{Sign consistency of $\gamma(t)$}: Once the signal is strong enough such that
            \begin{equation}
                \label{eq:slbiss-gammamin-cond}
                \gamma_{\min}^{\star} := \left( D_S \beta^{\star} \right)_{\min} \ge \frac{16\sigma}{\eta \lambda_H} \cdot \frac{\Lambda_X \Lambda_D}{\lambda_D^2} \left( 2\log s + 5 + \log (8 \Lambda_D) \right) \sqrt{\frac{\log m}{n}},
            \end{equation}
            then $\gamma(t)$ has sign consistency at $\bar{\tau}$, i.e. $ \mathrm{sign}( \gamma(\bar{\tau}) ) = \mathrm{sign} ( D \beta^{\star} )$. 
        \item
            \emph{$\ell_2$ consistency of $\gamma(t)$}:
            \begin{equation*}
                \left\| \gamma\left( \bar{\tau} \right) - D \beta^{\star} \right\|_2 \le \frac{42\sigma}{\eta \lambda_H} \cdot \frac{\Lambda_X}{\lambda_D} \sqrt{\frac{s\log m}{n}}.
            \end{equation*}
        \item
            \emph{$\ell_2$ ``consistency'' of $\beta(t)$}:
            \begin{multline*}
                \left\| \beta\left( \bar{\tau} \right) - \beta^{\star} \right\|_2 \le \frac{42\sigma}{\eta \lambda_H} \cdot \frac{\lambda_1 \Lambda_X (1 + \lambda_D) + \Lambda_X^2}{\lambda_1 \lambda_D^2} \sqrt{\frac{s\log m}{n}}\\
                + \frac{2\sigma}{\lambda_1} \sqrt{\frac{r'\log m}{n}} + \nu \cdot 2\sigma \cdot \frac{\lambda_1 \Lambda_X + \Lambda_X^2}{\lambda_1 \lambda_D^2},
            \end{multline*}
            where
            \begin{equation}
                \label{eq:r-prime-def}
                r' = \dim (\{ X \beta:\ \beta \in \ker(D) \}),
            \end{equation}
            which is very small in most cases. 
    \end{enumerate}
\end{theorem}

Despite that the sign consistency of $\gamma(t)$ can be established here, usually one can not expect $D\beta(t)$ recovers the sparsity pattern of $\gamma^\star$ due to the variable splitting. As shown in the last term of the $\ell_2$ error bound of $\beta(t)$, increasing $\nu$ will sacrifice its accuracy, as to achieve the minimax optimal $\ell_2$ error rate one needs $\nu=O(\sqrt{(s \log m) /n})$. However, one can remedy this by projecting $\beta(t)$ on to a subspace using the support set of $\gamma(t)$, and obtain a good estimator $\tilde{\beta}(t)$ with both sign consistency and $\ell_2$ consistency at the minimax optimal rates. This leads to the following theorem.

\begin{theorem}[Consistency of revised Split LBISS]
    \label[theorem]{thm:slbiss-rev-cstc}
    Under \Cref{thm:rsc} and \ref{thm:irr-nu}, define $\lambda_H = C / (1 + \nu)$ (from \cref{eq:rsc-nu}) and suppose that $\kappa$ satisfies \cref{eq:kappa-lower}. Define $\bar{\tau}$ the same as in \Cref{thm:slbiss-cstc}, and define
    \begin{equation*}
        S(t) := \mathrm{supp}(\gamma(t)),\ P_{S(t)} := P_{\ker\left( D_{S(t)^c} \right)} = I - D_{S(t)^c}^{\dag} D_{S(t)^c},\ \tilde{\beta}(t) := P_{S(t)} \beta(t).
    \end{equation*}
    If $S(t)^c = \varnothing$, define $P_{S(t)} = I$. Then we have the following properties.
    \begin{enumerate}
        \item
            \emph{Sign consistency of $\tilde{\beta}(t)$}: Once \cref{eq:slbiss-gammamin-cond} holds, then with probability not less than $1 - 8/m - 3\exp(-4n/5)$, there holds $\mathrm{sign}( D \tilde{\beta}(\bar{\tau}) ) = \mathrm{sign}( D \beta^{\star} )$.
        \item
            \emph{$\ell_2$ consistency of $\tilde{\beta}(t)$}: With probability not less than $1 - 8/m - 2r'/m^2 - 3\exp(-4n/5)$, we have
            \begin{multline*}
                \left\| \tilde{\beta}\left( \bar{\tau} \right) - \beta^{\star} \right\|_2 \le \frac{80\sigma}{\eta \lambda_H} \cdot \frac{\Lambda_X \left( \Lambda_D + \lambda_D^2 \right)}{\lambda_D^3} \sqrt{\frac{s\log m}{n}}\\
                + \frac{2\sigma}{\lambda_H} \left( \frac{\Lambda_X}{\lambda_D^2} + \frac{\lambda_H \lambda_D^2 + \Lambda_X^2}{\lambda_1 \lambda_D^2} \right) \sqrt{\frac{r' \log m}{n}} + 2 \left\| D_{S(\bar{\tau})^c}^{\dag} D_{S(\bar{\tau})^c\cap S} \beta^{\star} \right\|_2,
            \end{multline*}
            where $r'$ is defined in \cref{eq:r-prime-def}. If additionally $S(\bar{\tau}) \supseteq S$, then the last term on the right hand side drops.
    \end{enumerate}
\end{theorem}

\begin{remark}
    In most cases $r'$ is very small, so the dominant $\ell_2$ error rate is $O(\sqrt{(s\log m)/n})$ (as long as $\nu$ is upper bounded by constant), which is minimax optimal \citep{lee_model_2013,liu_guaranteed_2013}.
\end{remark}

\subsection{Consistency of Split LBI}
\label{sec:slbi-cstc}

Based on theorems on consistency of Split LBISS, one can naturally derive similar results for Split LBI with large $\kappa$ and small $\alpha$.

\begin{theorem}[Consistency of Split LBI]
    \label[theorem]{thm:slbi-cstc}
    Under \Cref{thm:rsc} and \ref{thm:irr-nu}, define $\lambda_H = C / (1 + \nu)$ (from \cref{eq:rsc-nu}). Suppose that $\kappa$ is large and $\alpha$ is small, so that
    \begin{equation}
        \label{eq:step-upper}
        \kappa \alpha \|H\|_2 < 2,
    \end{equation}
    $\kappa$ satisfies \cref{eq:kappa-lower} with $\lambda_H$ replaced by $\lambda_H' := \lambda_H (1 - \kappa \alpha \|H\|_2 / 2) > 0$, and
    \begin{equation*}
        5\alpha < \bar{\tau} := \frac{\eta}{8\sigma} \cdot \frac{\lambda_D}{\Lambda_X} \sqrt{\frac{n}{\log m}}.
    \end{equation*}
    Let $\overline{k} := \lfloor \bar{\tau} / \alpha \rfloor$. Then with probability not less than $1 - 6/m - 3\exp(-4n/5)$, we have all the following properties.
    \begin{enumerate}
        \item
            \emph{No-false-positive}: The solution has no false-positive, i.e. $\mathrm{supp}(\gamma_k) \subseteq S$, for $0\le k\alpha \le \overline{\tau}$.
        \item
            \emph{Sign consistency of $\gamma_k$}: Once the signal is strong enough such that
            \begin{multline}
                \label{eq:slbi-gammamin-cond}
                \gamma_{\min}^{\star} := \left( D_S \beta^{\star} \right)_{\min}\\
                \ge \frac{16\sigma}{\eta \lambda_H' \left( 1 - 5\alpha / \bar{\tau} \right)} \cdot \frac{\Lambda_X \Lambda_D}{\lambda_D^2} \left( 2\log s + 5 + \log (8 \Lambda_D) \right) \sqrt{\frac{\log m}{n}},
            \end{multline}
            then $\gamma_k$ has sign consistency at $\bar{k}$, i.e. $ \mathrm{sign}( \gamma_{\bar{k}} ) = \mathrm{sign} ( D \beta^{\star} )$. 
        \item
            \emph{$\ell_2$ consistency of $\gamma_k$}:
            \begin{equation*}
                \left\| \gamma_{\bar{k}} - D \beta^{\star} \right\|_2 \le \frac{42\sigma}{\eta \lambda_H' \left( 1 - \alpha / \bar{\tau} \right)} \cdot \frac{\Lambda_X}{\lambda_D} \sqrt{\frac{s\log m}{n}}.
            \end{equation*}
        \item
            \emph{$\ell_2$ ``consistency'' of $\beta_k$}:
            \begin{multline*}
                \left\| \beta_{\bar{k}} - \beta^{\star} \right\|_2 \le \frac{42\sigma}{\eta \lambda_H' \left( 1 - \alpha / \bar{\tau} \right)} \cdot \frac{\lambda_1 \Lambda_X (1 + \lambda_D) + \Lambda_X^2}{\lambda_1 \lambda_D^2} \sqrt{\frac{s\log m}{n}}\\
                + \frac{2\sigma}{\lambda_1} \sqrt{\frac{r'\log m}{n}} + \nu \cdot 2\sigma \cdot \frac{\lambda_1 \Lambda_X + \Lambda_X^2}{\lambda_1 \lambda_D^2},
            \end{multline*}
            where $r'$ is defined in \cref{eq:r-prime-def}.
    \end{enumerate}
\end{theorem}

Similarly after the projection one can get $\tilde{\beta}_k$ such that $D \tilde{\beta}_k$ is sparse and the corresponding $\ell_2$ error bound is improved.

\begin{theorem}[Consistency of revised Split LBI]
    \label[theorem]{thm:slbi-rev-cstc}
    Under \Cref{thm:rsc} and \ref{thm:irr-nu}, define $\lambda_H = C / (1 + \nu)$ (from \cref{eq:rsc-nu}). Suppose that $\kappa, \alpha$ satisfy the same conditions as in \Cref{thm:slbi-cstc}; $\lambda_H', \bar{\tau}$ is defined the same as in \Cref{thm:slbi-cstc}. Define
    \begin{equation*}
        S_k := \mathrm{supp}(\gamma_k),\ P_{S_k} := P_{\ker\left( D_{S_k^c} \right)} = I - D_{S_k^c}^{\dag} D_{S_k^c},\ \tilde{\beta}_k := P_{S_k} \beta_k.
    \end{equation*}
    If $S_k^c = \varnothing$, define $P_{S_k} = I$. Then we have the following properties.
    \begin{enumerate}
        \item
            \emph{Sign consistency of $\tilde{\beta}_k$}: Once \cref{eq:slbi-gammamin-cond} holds, then with probability not less than $1 - 8/m - 3\exp(-4n/5)$, there holds $\mathrm{sign}( D \tilde{\beta}_{\bar{k}} ) = \mathrm{sign}( D \beta^{\star} )$.
        \item
            \emph{$\ell_2$ consistency of $\tilde{\beta}_k$}: With probability not less than $1 - 8/m - 2r'/m^2 - 3\exp(-4n/5)$, we have
            \begin{multline*}
                \left\| \tilde{\beta}_{\bar{k}} - \beta^{\star} \right\|_2 \le \frac{80\sigma}{\eta \lambda_H' \left( 1 - \alpha / \bar{\tau} \right)} \cdot \frac{\Lambda_X \left( \Lambda_D + \lambda_D^2 \right)}{\lambda_D^3} \sqrt{\frac{s\log m}{n}}\\
                + \frac{2\sigma}{\lambda_H'} \left( \frac{\Lambda_X}{\lambda_D^2} + \frac{\lambda_H' \lambda_D^2 + \Lambda_X^2}{\lambda_1 \lambda_D^2} \right) \sqrt{\frac{r' \log m}{n}} + 2 \left\| D_{S_{\bar{k}}^c}^{\dag} D_{S_{\bar{k}}^c\cap S} \beta^{\star} \right\|_2,
            \end{multline*}
            where $r'$ is defined in \cref{eq:r-prime-def}. If additionally $S_{\bar{k}} \supseteq S$, then the last term on the right hand side drops.
    \end{enumerate}
\end{theorem}

\section{Proof Ideas for SLBISS Path Consistency Theorems}
\label{sec:ideas}

\begin{proof}[Sketchy proof of \Cref{thm:slbiss-cstc}]
    The Split LBISS dynamics always start within the oracle subspace ($\gamma_{S^c}(t) = 0$), and by \Cref{thm:slbiss-nfp} we prove that under the Irrepresentable Condition the exit time of the oracle subspace is no earlier than some $\bar{\tau} \lesssim \sqrt{n / \log m}$ (i.e. the no-false-positive condition holds before $\bar{\tau}$), with high probability.

    Before $\bar{\tau}$, the dynamics follow the identical path of the following \textit{oracle dynamics} of the Split LBISS restricted in the oracle subspace
    \begin{subequations}
        \label{eq:slbiss-orc}
        \begin{align}
            \label{eq:slbiss-orc-a}
            \rho_{S^c}'(t) &= \gamma_{S^c}'(t) \equiv 0,\\
            \label{eq:slbiss-orc-b}
            \dot{\beta}'(t) / \kappa &= - X^{*} \left( X \beta'(t) - y \right) - D^T \left( D \beta'(t) - \gamma'(t) \right) / \nu,\\
            \label{eq:slbiss-orc-c}
            \dot{\rho}_S'(t) + \dot{\gamma}_S'(t) / \kappa &= - \left( \gamma_S'(t) - D_S \beta'(t) \right) / \nu,\\
            \label{eq:slbiss-orc-d}
            \rho_S'(t) &\in \partial \left\| \gamma_S'(t) \right\|_1,
        \end{align}
    \end{subequations}
    where $\rho_S'(0) = \gamma_S'(0) = 0 \in \mathbb{R}^s,\ \beta'(0) = 0 \in \mathbb{R}^p$. \Cref{thm:slb-l-dec} shows that the loss is always dropping along the paths. Hence to monitor the distance of an estimator to the \textit{oracle estimator}
    \begin{equation}
        \label{eq:orc-def}
        \left( \beta^o, \gamma^o \right) \in \arg\min_{\substack{\beta, \gamma\\ \beta\in \mathcal{L},\ \gamma_{S^c} = 0}} \ell\left( \beta, \gamma \right) \subseteq \arg\min_{\substack{\beta, \gamma\\ \gamma_{S^c} = 0}} \ell\left( \beta, \gamma \right)
    \end{equation}
    which\protect\footnote{The property of the right hand side of \cref{eq:orc-def} is based on $\ell(P_\mathcal{L} \beta^o, \gamma^o) = \ell(\beta^o, \gamma^o)$.} is an optimal estimate of the true parameter $(\beta^{\star}, \gamma^{\star})$ (with error bounds in \Cref{thm:orc-minus-star}), we define a \textit{potential function}
    \begin{equation*}
        \Psi(t) := D^{\rho_S'(t)} \left( \gamma_S^o, \gamma_S'(t) \right) + d(t)^2 / (2\kappa),
    \end{equation*}
    where
    \begin{equation}
        \label{eq:d-def-slbiss}
        d_\beta(t) := \beta'(t) - \beta^o,\ d_\gamma(t) := \gamma'(t) - \gamma^o,\ d(t) := \sqrt{\left\| d_{\gamma, S}(t) \right\|_2^2 + \left\| d_\beta(t) \right\|_2^2},
    \end{equation}
    and the Bregman distance
    \begin{multline*}
        D^{\rho_S'(t)}\left( \gamma_S^o, \gamma_S'(t) \right) := \left\| \gamma_S^o \right\|_1 - \left\| \gamma_S'(t) \right\|_1 - \left\langle \gamma_S^o - \gamma_S'(t),\ \rho_S'(t) \right\rangle\\
        = \left\| \gamma_S^o \right\|_1 - \left\langle \gamma_S^o, \rho_S'(t) \right\rangle.
    \end{multline*}
    Equipped with this potential function, the original differential inclusion is reduced to the following differential inequality, called as generalized Bihari's inequality (\Cref{thm:slbiss-orc-gbi}) whose proof will be given in \ref{sec:proof-orc}. 
    \begin{lemma}[\textnormal{Generalized Bihari's inequality}]
        \label{thm:slbiss-orc-gbi}
        Under \Cref{thm:rsc}, for all $t\ge 0$ we have
        \begin{equation*}
            \frac{\mathrm{d}}{\mathrm{d}t} \Psi(t) \le - \lambda_H F^{-1}\left( \Psi(t) \right),
        \end{equation*}
        where $\gamma_{\min}^o := \min(|\gamma_j^o|:\ \gamma_j^o\neq 0)$ and
        \begin{align*}
            F(x) &:= \frac{x}{2\kappa} +
            \begin{cases}
                0,& 0\le x < (\gamma_{\min}^o)^2,\\
                2x/\gamma_{\mathrm{min}}^o,& (\gamma_{\min}^o)^2 \le x < s (\gamma_{\min}^o)^2,\\
                2\sqrt{sx},& x \ge s (\gamma_{\min}^o)^2,
            \end{cases}\\
            F^{-1}(x) &:= \inf(y:\ F(y)\ge x)\ (y\ge 0).
        \end{align*}
    \end{lemma}
    Such an inequality, together with the Restricted Strong Convexity condition (RSC), leads to an exponential decrease of the potential above enforcing the convergence to the oracle estimator. Then we can show that as long as the signal is strong enough with all the magnitudes of entries of $\gamma^{\star}$ being large enough ($\gtrsim (\log s) \sqrt{(\log m) / n}$), the dynamics stopped at $\bar{\tau}$, exactly selects all nonzero entries of $\gamma^o$ (\cref{eq:slbiss-orc-cstc-sign} in \Cref{thm:slbiss-orc-cstc}), hence also of $\gamma^{\star}$ with high probability, achieving the sign consistency.

    Even without the strong signal condition, with RSC we can also show that the dynamics, at $\bar{\tau}$, returns a good estimator of $\gamma^o$ (\cref{eq:slbiss-orc-cstc-l2} in \Cref{thm:slbiss-orc-cstc}), hence also of $\gamma^{\star}$, having an $\ell_2$ error $\simeq \sqrt{(s\log m) / n}$ (minimax optimal rate) with high probability. Combining the $\ell_2$ bounds of $\beta'(t) - \beta^o$ (from \cref{eq:slbiss-orc-cstc-l2} in \Cref{thm:slbiss-orc-cstc}) and $\beta^o - \beta^{\star}$ (\Cref{thm:orc-minus-star}), we obtain the result concerning the $\ell_2$ bound of $\beta'(t) - \beta^{\star}$, at $\bar{\tau}$, similarly with the minimax optimal rate.

    A detailed proof of \Cref{thm:slbiss-cstc} can be found in \ref{sec:proof-slbiss-cstc}. \qed
\end{proof}

\begin{remark}
    It is an interesting open problem how to relax the Irrepresentable Condition to achieve a minimax optimal estimator at weaker conditions such as \citep{BicRitTsy09}.
\end{remark}

\begin{proof}[Proof sketch of \Cref{thm:slbiss-rev-cstc}]
    By \Cref{thm:slbiss-cstc}, the exit time of the oracle subspace is no earlier than some $\bar{\tau} \lesssim \sqrt{n / \log m}$, i.e. the no-false-positive condition holds before $\bar{\tau}$, or say $S(t) \subseteq S$ for $t\le \bar{\tau}$, with high probability. The definition of $\tilde{\beta}(t)$ enforces
    \begin{equation*}
        D_{S^c} \tilde{\beta}\left( \bar{\tau} \right) = 0 = D_{S^c} \beta^{\star}.
    \end{equation*}
    Using the error bounds of $\beta'(t) - \beta^o$ (from \cref{eq:slbiss-orc-cstc-sign} in \Cref{thm:slbiss-orc-cstc}) and $\beta^o - \beta^{\star}$ (\Cref{thm:orc-minus-star}), we obtain
    \begin{equation*}
        \| D_S \tilde{\beta}\left( \bar{\tau} \right) - D_S \beta^{\star} \|_{\infty} < \gamma_{\min}^{\star} = (D_S \beta^{\star})_{\min} \Longrightarrow \mathrm{sign}\left( D_S \tilde{\beta}\left( \bar{\tau} \right) \right) = \mathrm{sign}(D_S \beta^{\star}),
    \end{equation*}
    as long as the magnitudes of entries of $\gamma^\star$ are all large enough, achieving the sign consistency. Also we can obtain the $\ell_2$ bound of $\tilde{\beta}(t) - \beta^{\star}$.

    A detailed proof of \Cref{thm:slbiss-rev-cstc} can be found in \ref{sec:proof-slbiss-cstc}. \qed
\end{proof}

\section{Experiments}
\label{sec:exp}

In this section, we show three additional applications using the algorithm proposed in this paper. The first application is about traditional image denoising using TV-regularization or fused Lasso. The remaining twos are new applications in partial order ranking: the second one is the basketball team ranking in partial order and the third one is the grouping of world universities in crowdsourced ranking. For reproducible research, Matlab source codes are released at the following website:
\begin{center}
\url{https://github.com/yuany-pku/split-lbi}.
\end{center}

\subsection{Parameter Setting}

Parameter $\kappa$ should be large enough according to \cref{eq:kappa-lower}. Moreover, step size $\alpha$ should be small enough to ensure the stability of Split LBI. When $\nu, \kappa$ are determined, $\alpha$ can actually be determined by $\alpha = \nu / (\kappa (1 + \nu \Lambda_X^2 + \Lambda_D^2))$ (see \cref{eq:kappastep-upper-2}).

\subsection{Application: Image Denoising}
\label{sec:simu-2dfused}

Consider the image denoising problem in \citet{tibshirani_solution_2011}. The original image is resized to $50 \times 50$, and reset with only four colors, as in the top left image in \Cref{fig:simu-2dfused-recon-auc}. Some noise is added by randomly changing some pixels to be white, as in the bottom left. Let $G = (V,E)$ is the 4-nearest-neighbor grid graph on pixels, then $\beta = (\beta_R;\beta_G;\beta_B)\in \mathbb{R}^{3|V|}$ since there are 3 color channels (RGB channels). $X = I_{3|V|}$ and $D = \mathrm{diag}(D_G,D_G,D_G)$, where $D_G \delta\in \mathbb{R}^{|E|\times|V|}$ is the gradient operator on graph $G$ defined by $(D_G x)(e_{ij}) = x_i-x_j,\ e_{ij} \in E$. Set $\nu = 180,\ \kappa = 100$. The regularization path of Split LBI is shown in \Cref{fig:simu-2dfused-recon-auc}, where as $t$ evolves, images on the path gradually select visually salient features before picking up the random noise.

Now compare the AUC (Area Under the Curve) of \texttt{genlasso} and Split LBI algorithm with different $\nu$. For simplicity we show the AUC corresponding to the red color channel. Here $\nu\in \{1,20,40,60,\ldots,300\}$. As shown in the right panel of \Cref{fig:simu-2dfused-recon-auc}, with the increase of $\nu$, Split LBI beats \texttt{genlasso} with higher AUC values.

\begin{figure}
    \centering
    \begin{minipage}{0.6\textwidth}
        \includegraphics[width = 0.25\textwidth]{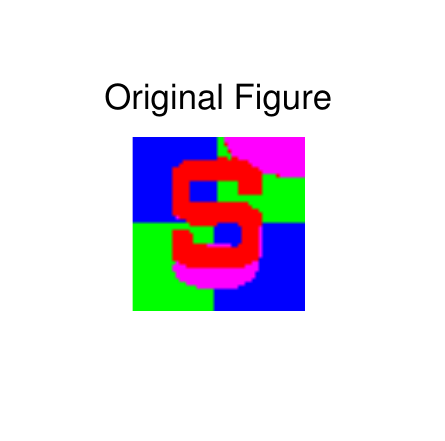}
        \hspace{0.1\textwidth}
        \includegraphics[width = 0.25\textwidth]{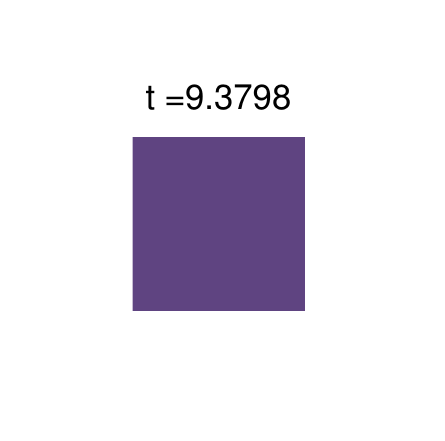}
        \includegraphics[width = 0.25\textwidth]{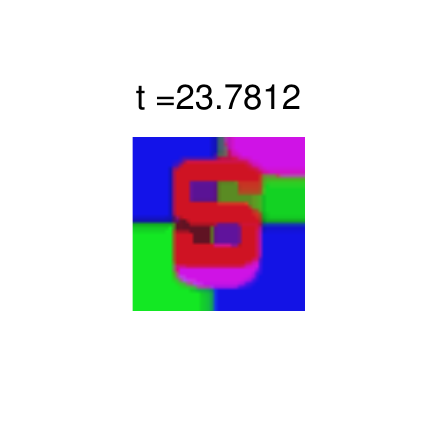}

        \includegraphics[width = 0.25\textwidth]{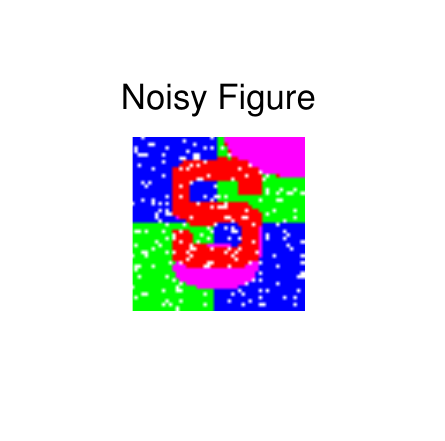}
        \hspace{0.1\textwidth}
        \includegraphics[width = 0.25\textwidth]{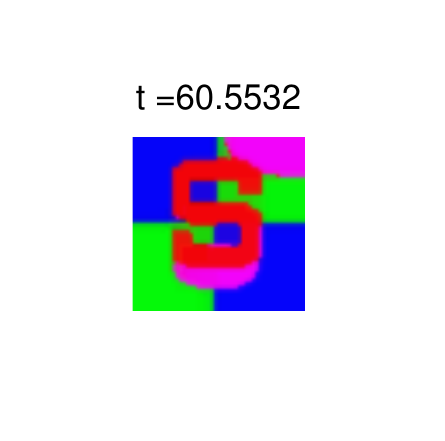}
        \includegraphics[width = 0.25\textwidth]{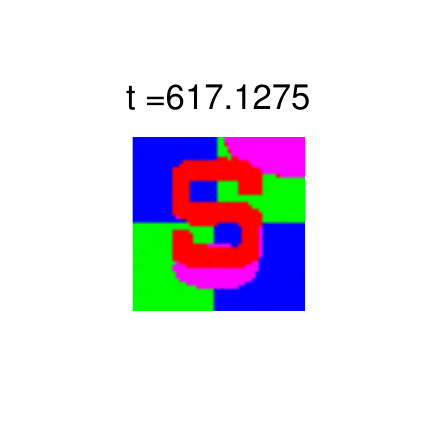}
    \end{minipage}
    \begin{minipage}{0.35\textwidth}
        \includegraphics[width = \textwidth]{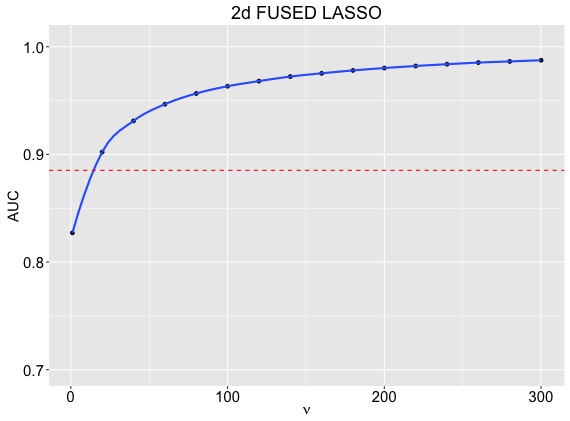} 
    \end{minipage}
    \caption{Left is image denoising results by Split LBI. Right shows the AUC of Split LBI (blue solid line) increases and exceeds that of \texttt{genlasso} (dashed red line) as $\nu$ increases.}
    \label{fig:simu-2dfused-recon-auc}
\end{figure}

\subsection{Application: Partial Order Ranking for Basketball Teams}
\label{sec:real-basketball}

\begin{figure}
    \centering
    \includegraphics[width = 0.8\textwidth]{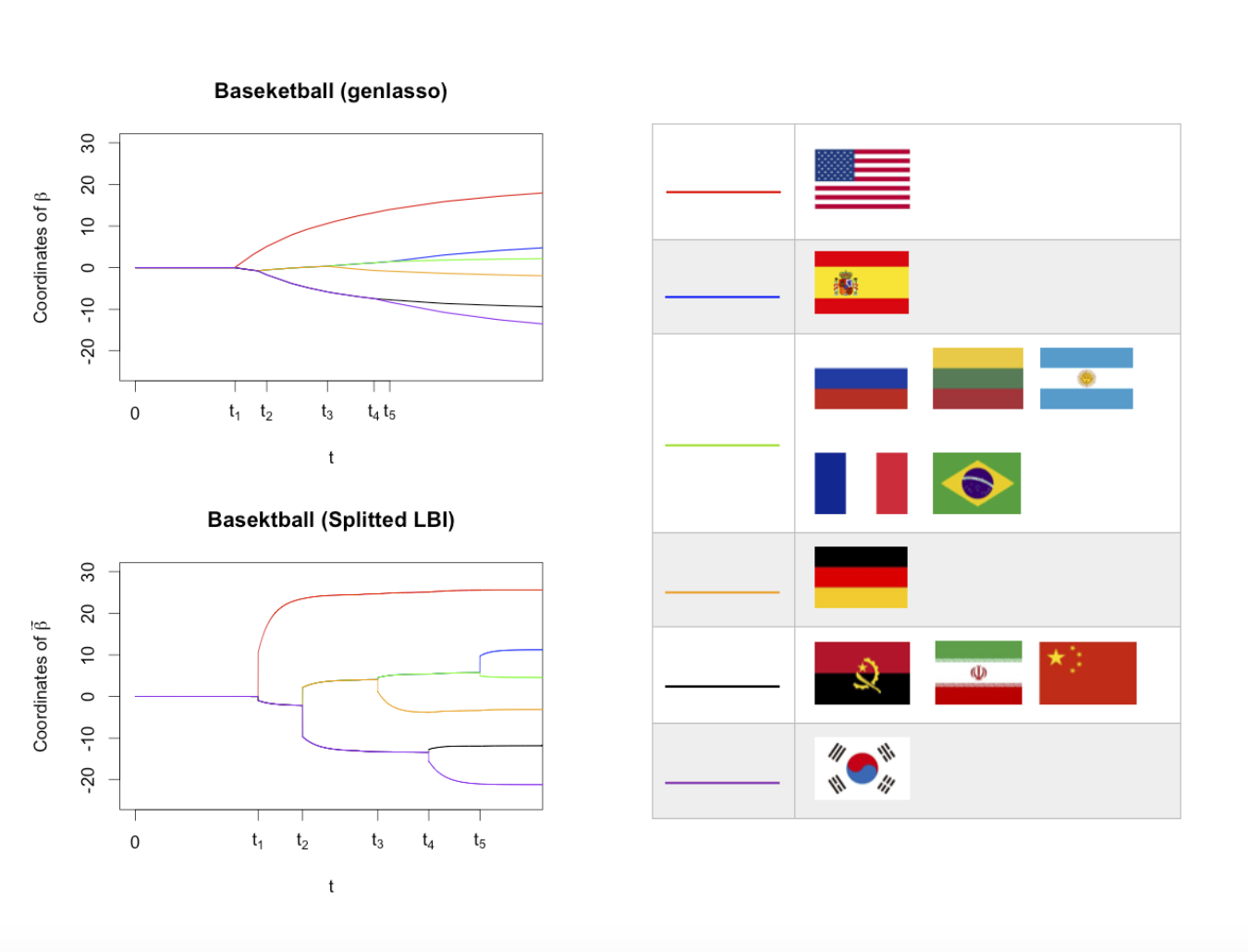}
    \includegraphics[width = 0.8\textwidth]{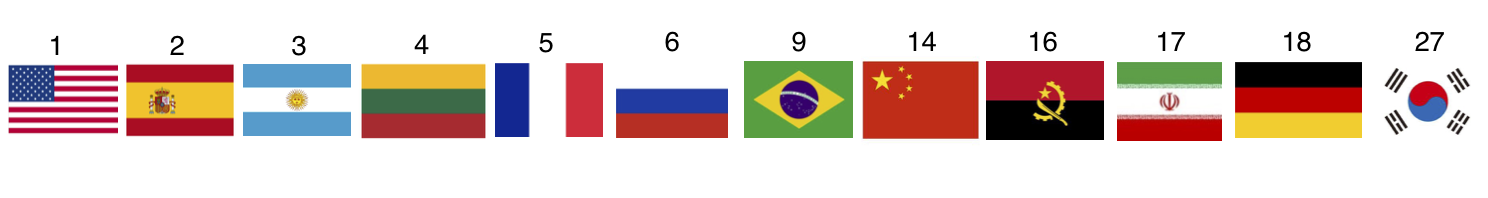}
    \caption{Partial order ranking for basketball teams. Top left shows $\{\beta_\lambda\}\ (t = 1/\lambda)$ by \texttt{genlasso} and $\tilde{\beta}_k\ (t = k\alpha)$ by Split LBI. Top right shows the same grouping result just passing $t_5$. Bottom is the FIBA ranking of all teams.}
    \label{fig:real-basketball-group}
\end{figure}

Here we consider a new application on the ranking of $p=12$ FIBA basketball teams into partial orders. The teams are listed in \Cref{fig:real-basketball-group}. We collected $n = 134$ pairwise comparison game results mainly from various important championship such as Olympic Games, FIBA World Championship and FIBA Basketball Championship in 5 continents from 2006--2014 (8 years is not too long for teams to keep relatively stable levels while not too short to have enough samples). For each sample indexed by $k$ and corresponding team pair $(i_k, j_k)$, $y_k = s_{i_k} - s_{j_k}$ is the score difference between team $i_k$ and $j_k$. We assume a model $y_k = \beta^\star_{i_k} - \beta^\star_{j_k} + \epsilon_k$ where $\beta^{\star}\in \mathbb{R}^p$ measures the strength of these teams. So the design matrix $X\in \mathbb{R}^{n\times p}$ is defined by its $k$-th row: $x_{k,i_k} = 1,\ x_{k,j_k} = -1,\ x_{k,l} = 0\ (l \neq i_k,j_k)$. In sports, teams with similar strength generally meet more often than those in different levels. Thus we hope to find a coarse grained partial order ranking by adding a structural sparsity on $D\beta^\star$ where $D=c X$ ($c$ scales the smallest nonzero singular value of $D$ to be 1).

The top left panel of \Cref{fig:real-basketball-group} shows $\{\beta_\lambda\}$ by \texttt{genlasso} and $\tilde{\beta}_k$ by Split LBI with $\nu = 1$ and $\kappa = 100$. Both paths give the same partial order at early stages, though the Split LBI path looks qualitatively better. For example, the top right panel shows the same partial order after the change point $t_5$. It is interesting to compare it against the FIBA ranking in September, 2014, shown in the bottom. Note that the average basketball level in Europe is higher than that of in Asia and Africa, hence China can get more FIBA points than Germany based on the dominant position in Asia, so is Angola in Africa. But their true levels might be lower than Germany, as indicated in our results. Moreover, America (FIBA points $1040.0$) itself forms a group, agreeing with the common sense that it is much better than any other country. Spain, having much higher FIBA ranking points ($705.0$) than the 3rd team Argentina ($455.0$), also forms a group alone. It is the only team that can challenge America in recent years, and it enters both finals against America in 2008 and 2012.

\subsection{Application: Grouping in Crowdsourced Ranking of World Universities}

\begin{figure}[!h]
    \centering
    \includegraphics[width = 0.8\columnwidth]{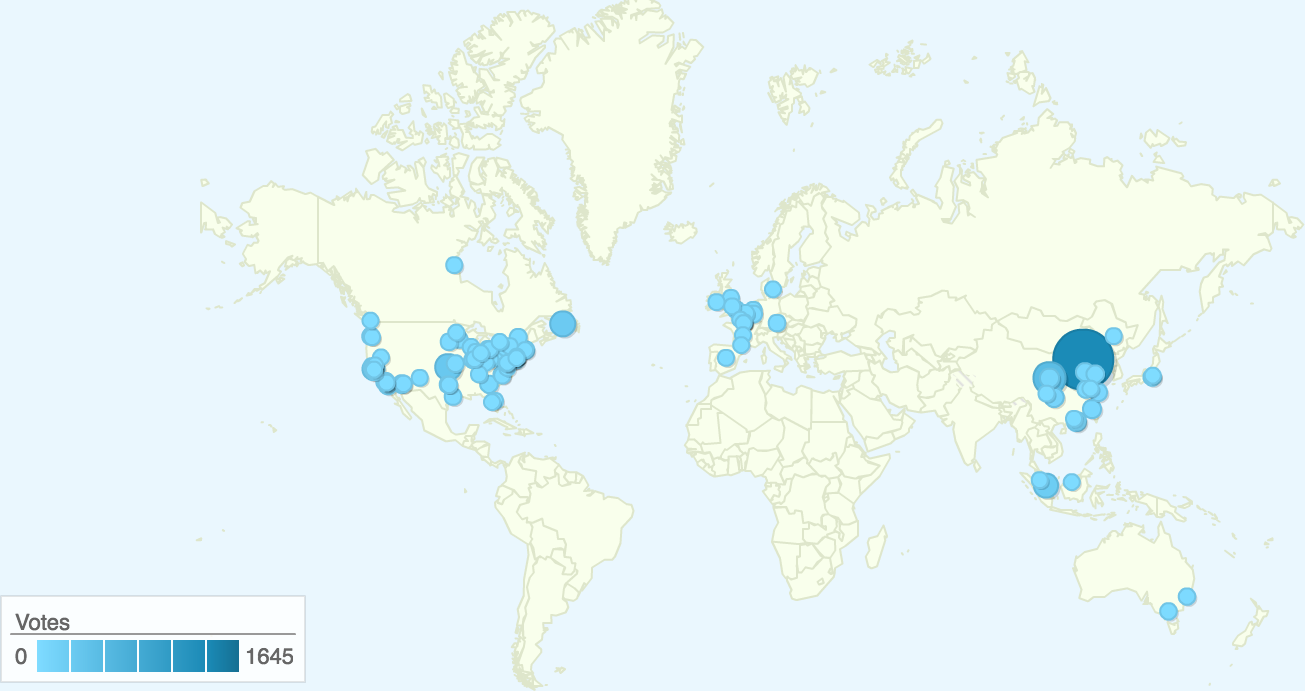}
    \caption{The map of voter distribution.}
    \label{fig:voter-distribution}
\end{figure}

Crowdsourcing technique has been recently used to rank universities by Internet voters, e.g. {\tt{CrowdRank}}. In the following a crowdsourcing experiment has been conducted for ranking $p = 261$ universities in the world on the platform \url{http://www.allourideas.org/worldcollege}. The majority of the participants are undergraduates or alumni from Peking University, mostly majoring in applied mathematics and statistics while some with engineering background. Voters are widely distributed around the world, with one fifth of all from Beijing, see \Cref{fig:voter-distribution}. Every voter is presented with a randomly chosen pair of universities, and asked with the question ``\emph{which university would you rather attend?}''. Then the voter is allowed to choose either of the two universities, or simply ``\emph{I can't decide}''. Our collection consists of about eight thousand votes. To make our result more robust, we remove some indecisive votes or outliers using the technique from \cite{xu_robust_2014} and are left with $n = 6,125$ paired comparison samples in the cleaned dataset for the study in this paper. For each sample indexed by $k$ and corresponding university pair $(i_k, j_k)$, if the voter considers $i_k$ to be better than $j_k$, then $y_k = 1$, otherwise $y_k = - 1$. We assume a model $y_k = \beta^\star_{i_k} - \beta^\star_{j_k} + \epsilon_k$ where $\beta^{\star}\in \mathbb{R}^p$ measures the strength of these universities. So the design matrix $X\in \mathbb{R}^{n\times p}$ is defined by its $k$-th row: $x_{k,i_k} = 1,\ x_{k,j_k} = -1,\ x_{k,l} = 0\ (l \neq i_k,j_k)$. $D$ is denoted as the total variation matrix with complete graph, i.e. $\Vert D\beta \Vert_{1} = \Sigma_{i < j} \vert \beta_{i} - \beta_{j} \vert_{1}$ for any $\beta \in R^{p}$. Split LBI is then implemented to obtain $\tilde{\beta}_k$.

\begin{figure}[!h]
    \centering
    \begin{minipage}[t]{0.48\linewidth}
        \includegraphics[width = \linewidth]{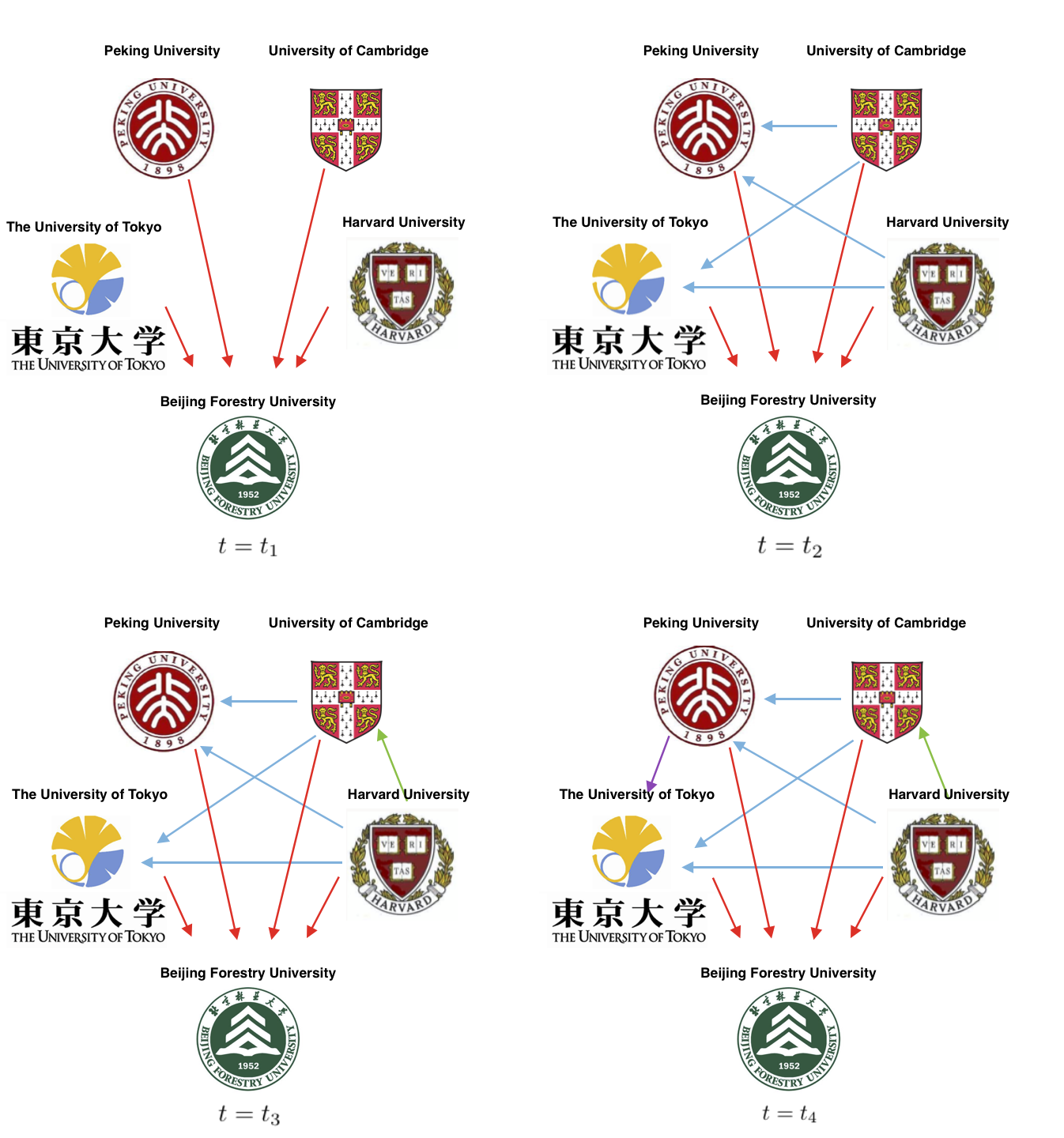}
    \end{minipage}
    \begin{minipage}[t]{0.48\linewidth}
        \includegraphics[width = \linewidth]{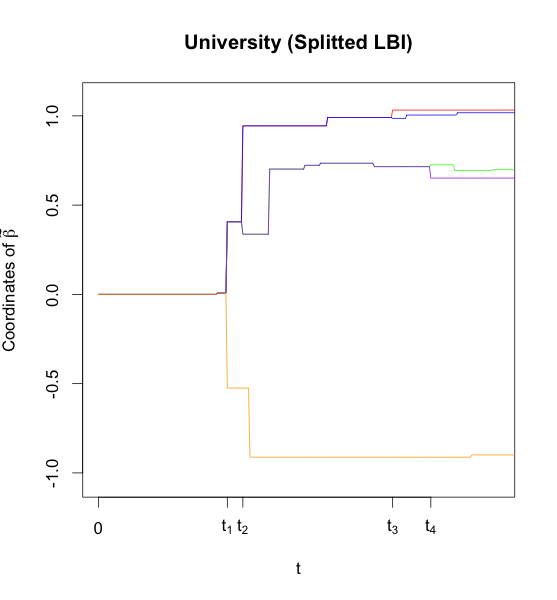}
    \end{minipage}
    \caption{World university ranking. Right shows $\tilde{\beta}_k\ (t = k\alpha)$ (on the $5$ entries corresponding to $5$ selected universities) by Split LBI. Left shows the corresponding graph $G_t$ at $t = t_1, \ldots, t_4$. $i \rightarrow j$ if the learned entry for $i$ is better than $j$.}
    \label{fig:real-university-group}
\end{figure}

Similar to the the previous application on basketball team ranking, for each $k$, entries of $\tilde{\beta}_k$ with same values form a group. For each $k$, consider the directed graph $G_t = (V, E_t)\ (t = k\alpha)$ with $V = \{1, \ldots, p\}$ and $E_t$ consisting of directed edges $(i,j)$'s with $\tilde{\beta}_{k,i} \neq \tilde{\beta}_{k,j}$ ($i\to j$ if $\tilde{\beta}_{k,i} < \tilde{\beta}_{k,j}$). We pick up $5$ universities for a simple illustration. See \Cref{fig:real-university-group} for the path and corresponding graphs for $t = t_1, \ldots, t_4$. No edge is selected at $t = 0$. At $t = t_1$, Beijing Forestry University is left behind. At $t = t_2$, we see that Harvard University and The University of Cambridge form the 1st group; Peking University and The University of Tokyo form the 2nd group; Beijing Forestry University becomes the last group. Continuing at $t = t_3, t_4$, further refinements within the 1st and 2rd groups are made. Note that at $t = t_{4}$, Peking University is more preferred to The University of Tokyo, yet below Harvard and Cambridge, reflecting the preference of voters from Peking University.

Now back to the whole set of $p = 261$ universities, we pick up a particular time at which the universities are separated into $10$ groups. Some reasonable results can be observed. See \Cref{tab:real-university-group-1} for the 1st group consisting of $17$ top universities. Most of them are first tier universities in USA, together with two top universities in UK (University of Cambridge, University of Oxford). California is clearly a favorite place for these voters, having five institutes included in the first group. 
 
\begin{table}[!h]
    \centering
    \begin{tabular}{p{0.48\textwidth} p{0.45\textwidth}}
        \hline
        Harvard University & Princeton University \\
        \hline
        Stanford University & University of California, Berkeley \\
        \hline
        Yale University & Cornell University \\
        \hline
        University of California, Los Angeles & University of Cambridge (UK) \\
        \hline
        California Institute of Technology & University of Oxford (UK) \\
        \hline
        Columbia University & University of Pennsylvania \\
        \hline
        Carnegie Mellon University & University of California, San Diego \\
        \hline
        University of Michigan & New York University \\
        \hline
        Johns Hopkins University & \\
        \hline 
    \end{tabular}
    \caption{Universities in the 1st group.}
    \label{tab:real-university-group-1}
\end{table}

The 2nd group universities are listed in See \Cref{tab:real-university-group-2}. It includes top universities in Asia (Peking University, Tsinghua University, University of Tokyo, University of Hong Kong, and Hong Kong University of Science and Technology), Europe (Swiss Federal Institute of Technology/ETH, Imperial College London, University College London, and London School of Economics and Political Science), and North America.  A surprising result is that MIT is listed in this second group, while most of the authoritative ranking systems clearly place it in the first tier. This phenomenon is probably due to the sampling bias in our crowdsourcing experiment: a large portion of the voters are of statistics major and MIT does not have a statistics department or program. Hence such voters will not choose MIT when considering graduate programs.

\begin{table}[!h]
    \centering
    \begin{tabular}{p{0.46\textwidth} p{0.46\textwidth}}
        \hline
        Massachusetts Institute of Technology (MIT) & University of Southern California\\
        \hline
        University of British Columbia (Canada) & University of Wisconsin-Madison \\
        \hline
        Peking University (China) & Northwestern University \\
        \hline
        University of Chicago & Swiss Federal Institute of Technology (Switzerland) \\
        \hline
        Brown University & Georgia Institute of Technology \\
        \hline
        Imperial College London (UK) & University of Washington \\
        \hline
        University of Toronto (Canada) &  University of California, Santa Barbara\\
        \hline
        Duke University & University of Tokyo (Japan) \\
        \hline
        The University of Hong Kong (Hong Kong) & Purdue University \\
        \hline
        University of Texas at Austin & Dartmouth College\\
        \hline
        University of California, Irvine & University of California, Santa Cruz \\
        \hline
        University of California, Davis & Tsinghua University (China) \\
        \hline
        University of Maryland, College Park & London School of Economics and Political Science (UK)\\
        \hline
        Boston University & Hong Kong University of Science and Technology (Hong Kong) \\
        \hline
        University College London (UK) & Rice University \\
        \hline 
    \end{tabular}
    \caption{The Universities in the 2nd Group}
    \label{tab:real-university-group-2}
\end{table}

Information about other groups can be found on website: \url{https://github.com/yuany-pku/split-lbi/tree/master/examples/university}.

\section{Conclusion}
\label{sec:conclusion}

In this paper, we introduce a novel iterative regularization path with structural sparsity such that parameters are sparse under certain linear transform. Variable splitting is exploited to lift the parameters into a high dimensional space with separate parameters for data fitting and sparse model selection. A statistical benefit of such a splitting lies in its improved model selection consistency under weaker conditions than the traditional generalized Lasso, shown in both theory and experiments. For the statistical analysis of such an algorithm, several limit dynamics as differential inclusions are introduced which sheds light on the consistency properties of the regularization paths. Finally some applications are given with real world data, including image denoising, partial order ranking of basket ball teams, and grouping of world universities by crowdsourced ranking. These results show that the benefit of the proposed algorithm lies in both its simplicity in computing the regularization path iteratively and its solid theoretical guarantee on path consistency. Hence it can be regarded as a generalization of $L_2$Boost in machine learning or Landweber iteration in inverse problems with structural sparsity control.  

\appendix

\section{Further Notations throughout the Appendix}

Apart from \Cref{sec:notation} and \ref{sec:basic-assump}, we need more notations throughout the appendix. Let the compact singular value decomposition (compact SVD) of $D$ be
\begin{equation}
    \label{eq:D-svd}
    D = U \Lambda V^T\ \left( \Lambda\in \mathbb{R}^{r\times r},\ \Lambda \succ 0,\ U\in \mathbb{R}^{m\times r},\ V\in \mathbb{R}^{p\times r} \right),
\end{equation}
and $(V, \tilde{V})$ be an orthogonal square matrix. Let the compact SVD of $X\tilde{V}/\sqrt{n}$ be
\begin{equation}
    \label{eq:XVtilde-svd}
    X \tilde{V} / \sqrt{n} = U_1 \Lambda_1 V_1^T\ \left( \Lambda_1\in \mathbb{R}^{r'\times r'},\ \Lambda_1 \succ 0,\ U_1\in \mathbb{R}^{n\times r'},\ V_1\in \mathbb{R}^{\left( p-r \right)\times r'} \right),
\end{equation}
and let $(V_1, \tilde{V}_1)$ be an orthogonal square matrix. $r'$ in \cref{eq:XVtilde-svd} is the rank of $X \tilde{V}$, which meets the definition \cref{eq:r-prime-def}.

We have $\lambda_D I \preceq \Lambda \preceq \Lambda_D I$. If $\ker(D) \subseteq \ker(X)$ (for example, $D$ has full column rank), then $X \tilde{V} = 0$, and $\tilde{V},U_1,\Lambda_1,V_1, \tilde{V}_1$ all drop.

Generally, $r' \le p - r$. If \Cref{thm:rsc} holds, noting for any $\xi\in \mathbb{R}^{p-r}$, $\tilde{V} \xi \in \ker(D) \subseteq \mathcal{M}$, and
\begin{equation*}
    \left( \tilde{V} \xi \right)^T X^{*} X \left( \tilde{V} \xi \right) \ge \lambda \left\| \tilde{V} \xi \right\|_2^2 = \lambda \left\| \xi \right\|_2^2,
\end{equation*}
we have $V_1 \Lambda_1^2 V_1^T = \tilde{V}^T X^{*} X \tilde{V} \succeq \lambda I$. Since $V_1$ is a tall matrix, we further know it is square (elsewise $V_1 \Lambda_1^2 V_1^T$ is not invertible), i.e. $r' = p - r$. Besides, we have $\lambda_1 := \lambda_{\min}(\Lambda_1) \ge \sqrt{\lambda}$ (when $\Lambda_1$ drops, $\lambda_1 := +\infty$).

From now, we also write $\lambda_H = C / (1 + \nu),\ \lambda_\Sigma = C' / (1 + \nu)$ according to \Cref{thm:rsc-nu} and \ref{thm:rsc-nu-prime}.

\section{Some Useful Technical Lemmas}

\begin{lemma}[\textnormal{Concentration inequalities}]
    \label{thm:tail}
    Suppose that $\epsilon\in \mathbb{R}^n $ has independent identically distributed components, each of which has a sub-Gaussian distribution with parameter $\sigma^2$, i.e. $\mathbb{E} [\exp(t \epsilon_i)] \le \exp(\sigma^2 t^2/2)$, then
    \begin{align}
        \label{eq:tail-subgau}
        & \mathbb{P} \left( \frac{\left\| B \epsilon \right\|_{\infty}}{\sigma} \ge z \right) \le 2q \exp\left( - \frac{z^2}{2\left\|B\right\|_2^2} \right)\ \left( B\in \mathbb{R}^{q\times n},\ z \ge 0 \right),\\
        \label{eq:tail-subexp}
        & \mathbb{P} \left( \frac{\left\| \epsilon \right\|_2^2}{n \sigma^2} \ge 1 + z \right) \le \exp\left( - \frac{n \left( z - \log \left( 1 + z \right) \right)}{2} \right)\ (z\ge 0).
    \end{align}
    Moreover, by \cref{eq:tail-subgau} we have that for $B\in \mathbb{R}^{q\times n}$, with probability not less than $1 - 2q / m^2$,
    \begin{equation}
        \label{eq:tail-subgau-2}
        \left\| B \epsilon \right\|_{\infty} \le 2\sigma\cdot \left\| B \right\|_2 \sqrt{\log m}.
    \end{equation}
    By \cref{eq:tail-subexp} we have that with probability not less than $1 - \exp(- 4n/5)$,
    \begin{equation}
        \label{eq:tail-subexp-2}
        \left\| \epsilon \right\|_2 \le 2\sigma \sqrt{n}.
    \end{equation}
\end{lemma}

\begin{proof}
    As for \cref{eq:tail-subgau}, let $B = (B_{i,j})_{q\times n}$ and $1\le i\le q$, it is well-known that
    \begin{equation*}
        B_{i,\cdot} \epsilon = B_{i,1} \epsilon_1 + B_{i,2} \epsilon_2 + \cdots + B_{i,n} \epsilon_n
    \end{equation*}
    is also sub-Gaussian, with parameter $b_i^2 = (B_{i,1}^2 + \cdots + B_{i,n}^2) \sigma^2$. Thus
    \begin{equation*}
        \mathbb{P} \left( \left\| B \epsilon \right\|_{\infty} \ge z \right) \le q \cdot \max_{1\le i\le q} \mathbb{P} \left( \left| B_{i,\cdot} \epsilon \right| \ge z \right) \le 2q \exp\left( - \frac{z^2}{2b_i^2} \right) \le 2q \exp\left( - \frac{z^2}{2\left\|B\right\|_2^2} \right).
    \end{equation*}
    As for \cref{eq:tail-subexp}, note that for $0\le \zeta < 1/2$,
    \begin{multline*}
        \mathbb{P} \left( \frac{\left\| \epsilon \right\|_2^2}{n\sigma^2} \ge 1 + z \right) \le \mathbb{P} \left( \exp\left( \frac{\zeta \left\| \epsilon \right\|_2^2}{\sigma^2} \right) \ge \exp\left( \zeta n (1+z) \right) \right)\\
        \le \exp\left( - \zeta n (1+z) \right) \mathbb{E} \left[ \exp\left( \frac{\zeta \left\| \epsilon \right\|_2^2}{\sigma^2} \right) \right]\\
        = \exp\left( - \zeta n (1+z) \right) \left( \mathbb{E} \left[ \exp\left( \frac{\zeta \epsilon_1^2}{\sigma^2} \right) \right] \right)^n \le \exp\left( - \zeta n (1+z) \right) \cdot \left( \frac{1}{1-2\zeta} \right)^{n/2}.
    \end{multline*}
    Take $\zeta = z / (2(1+z))\in [0, 1/2)$, and \cref{eq:tail-subexp} follows.
\end{proof}

\begin{lemma}[\textnormal{Transformations and bounds for quadratic forms}]
    \label{thm:schur}
    If
    \begin{equation*}
        K = \begin{pmatrix} P & Q\\ Q^T & R \end{pmatrix} \succeq 0,
    \end{equation*}
    then
    \begin{equation}
        \label{eq:K-ge-1}
        \left( u^T, v^T \right) \begin{pmatrix} P & Q\\ Q^T & R \end{pmatrix} \begin{pmatrix} u\\ v\end{pmatrix} \ge \max\left( u^T \left( P - Q R^{\dag} Q^T \right) u,\ v^T \left( R - Q^T P^{\dag} Q \right) v \right).
    \end{equation}
    Moreover, for $0\le \zeta \le 1$, the following two statements are equivalent:
    \begin{gather}
        \label{eq:P-ge-R-schur}
        P - Q R^{\dag} Q^T\succeq \zeta P,\\
        \label{eq:R-ge-P-schur}
        R - Q^T P^{\dag} Q\succeq \zeta R.
    \end{gather}
    And if \cref{eq:P-ge-R-schur,eq:R-ge-P-schur} hold, then by \cref{eq:K-ge-1} we easily obtain
    \begin{multline}
        \label{eq:K-ge-2}
        \left( u^T, v^T \right) \begin{pmatrix} P & Q\\ Q^T & R \end{pmatrix} \begin{pmatrix} u\\ v\end{pmatrix} \ge \max\left( \zeta u^T P u,\ \zeta v^T R v \right)\\
        \ge \xi \left( \lambda_{\min}(P) \|u\|_2^2 + \lambda_{\min}(R) \|v\|_2^2 \right) \ge \frac{\xi}{1/\lambda_{\min(P)} + 1/\lambda_{\min}(R)} \left\| \begin{pmatrix} u\\ v \end{pmatrix} \right\|_2^2.
    \end{multline}
\end{lemma}

\begin{proof}
    Theorem 1.19 in \citet{zhang_schur_2006} tells that $P P^{\dag} Q = Q$, so it is easy to verify
    \begin{equation*}
        K = \begin{pmatrix} I & 0\\ Q^T P^{\dag} & I \end{pmatrix} \begin{pmatrix} P & 0\\ 0& R - Q^T P^{\dag} Q \end{pmatrix} \begin{pmatrix} I & P^{\dag} Q\\ 0 & I \end{pmatrix}.
    \end{equation*}
    Thus
    \begin{multline*}
        \left( u^T, v^T \right) K \begin{pmatrix} u\\ v\end{pmatrix} = \begin{pmatrix} u + P^{\dag} R v\\ v \end{pmatrix}^T \begin{pmatrix} P & 0 \\ 0 & R - Q^T P^{\dag} Q \end{pmatrix} \begin{pmatrix} u + P^{\dag} R v\\ v \end{pmatrix}\\
        \ge v^T \left( R - Q^T P^{\dag} Q \right) v.
    \end{multline*}
    Similarly we can obtain another inequality.

    If \cref{eq:P-ge-R-schur} holds, then
    \begin{multline*}
        P^{\dag 1/2} Q R^{\dag 1/2} \cdot R^{\dag 1/2} Q^T P^{\dag 1/2} \preceq (1-\zeta) P^{\dag 1/2} P P^{\dag 1/2} \preceq (1-\zeta) I\\
        \Longrightarrow R^{\dag 1/2} Q^T P^{\dag 1/2} \cdot P^{\dag 1/2} Q R^{\dag 1/2} \preceq (1 - \zeta) I\\
        \Longrightarrow R^{1/2} R^{\dag 1/2} Q^T P^{\dag} Q R^{\dag 1/2} R^{1/2} \preceq (1-\zeta) R.
    \end{multline*}
    By Theorem 1.19 in \citet{zhang_schur_2006} we have $Q R^{\dag 1/2} R^{1/2} = Q$, thus $Q^T P^{\dag} Q \preceq (1-\zeta) R$, i.e. \cref{eq:R-ge-P-schur} holds. Similarly \cref{eq:R-ge-P-schur} implies \cref{eq:P-ge-R-schur}.
\end{proof}

\begin{lemma}[\textnormal{Representation of $\mathcal{L}$}]
    \label{thm:beta-decomposition}
    Adopt notations from \cref{eq:D-svd,eq:XVtilde-svd}. $\beta\in \mathcal{L}$ if and only if
    \begin{equation*}
        \beta = V \delta + \tilde{V} V_1 \xi,\ \text{where}\ \delta = V^T \beta,\ \xi = V_1^T \tilde{V}^T \beta.
    \end{equation*}
\end{lemma}

\begin{proof}
    Note that
    \begin{equation*}
        I = V V^T + \tilde{V} \tilde{V}^T = V V^T + \tilde{V} \left( V_1 V_1^T + \tilde{V}_1 \tilde{V}_1^T \right) \tilde{V}^T.
    \end{equation*}
    Right multiplying $\beta$ on both side leads to
    \begin{equation}
        \label{eq:beta-decomposition}
        \beta = V \delta + \tilde{V} V_1 \xi + \tilde{V} \tilde{V}_1 \left( \tilde{V}_1^T \tilde{V}^T \beta \right).
    \end{equation}
    It suffices to show $\ker \left( \tilde{V}_1^T \tilde{V}^T \right) = \mathcal{L}$, which is equivalent to
    \begin{equation*}
        \mathcal{L}' := \mathrm{Im}\left( \tilde{V} \tilde{V}_1 \right) = \mathcal{L}^{\perp} \left( = \ker(X)\cap \ker(D) \right).
    \end{equation*}
    For any $\beta\in \mathcal{L}'$, we have $X \beta = 0,\ D \beta = 0$ since $X \tilde{V} \tilde{V}_1 = 0,\ D \tilde{V} = 0$, so $\beta\in \mathcal{L}^{\perp}$. Conversely, if $\beta\in \mathcal{L}^{\perp}$, left multiplying $D$ on both sides of \cref{eq:beta-decomposition} leads to $\delta = 0$. Then left multiplying $X$ on both sides of \cref{eq:beta-decomposition} further leads to $\xi = 0$. Now \cref{eq:beta-decomposition} tells that $\beta\in \mathcal{L}'$. So $\mathcal{L}' = \mathcal{L}^{\perp}$.
\end{proof}

\begin{lemma}
    \label{thm:Adag}
    Adopt the notation from \cref{eq:D-svd,eq:XVtilde-svd}. Define
    \begin{equation*}
        B := \Lambda^2 + \nu V^T X^{*} \left( I - U_1 U_1^T \right) X V.
    \end{equation*}
    We have
    \begin{equation}
        \label{eq:DAdag}
        D A^{\dag} = U \Lambda B^{-1} V^T \left( I - \frac{1}{\sqrt{n}} X^T U_1 \Lambda_1^{-1} V_1^T \tilde{V}^T \right).
    \end{equation}
    Consequently,
    \begin{equation}
        \label{eq:Sigma}
        \Sigma = \left( I - D A^{\dag} D^T \right) / \nu = \left( I - U \Lambda B^{-1} \Lambda U^T \right) / \nu.
    \end{equation}
\end{lemma}

\begin{proof}
    Note that
    \begin{multline*}
        \begin{pmatrix} V^T \\ \tilde{V}^T \end{pmatrix} A \left( V, \tilde{V} \right) = \begin{pmatrix} \Lambda^2 + \nu V^T X^{*} X V & \nu V^T X^T U_1 \Lambda_1 V_1^T / \sqrt{n} \\ \nu V_1 \Lambda_1 U_1^T X V / \sqrt{n} & \nu V_1 \Lambda_1^2 V_1^T \end{pmatrix}\\
        = Q M Q^T,\ \text{where}\ Q := \begin{pmatrix} I_r & V^T X^T U_1 \Lambda_1^{-1} V_1^T / \sqrt{n} \\ 0 & I_{p-r} \end{pmatrix},\ M := \begin{pmatrix} B & 0\\ 0 & \nu V_1 \Lambda_1^2 V_1^T \end{pmatrix}
    \end{multline*}
    We can directly verify that $( Q M Q^T )^{\dag} = ( Q^T )^{-1} M^{\dag} Q^{-1}$, thus
    \begin{equation*}
        D A^{\dag} = D \left( V, \tilde{V} \right) \left( \begin{pmatrix} V^T \\ \tilde{V}^T \end{pmatrix} A \left( V, \tilde{V} \right) \right)^{\dag} \begin{pmatrix} V^T \\ \tilde{V}^T \end{pmatrix} = \left( U \Lambda, 0 \right) \left( Q^T \right)^{-1} M^{\dag} Q^{-1} \begin{pmatrix} V^T \\ \tilde{V}^T \end{pmatrix},
    \end{equation*}
    which comes to be the right hand side of \cref{eq:DAdag}. Now it is easy to verify \cref{eq:Sigma}.
\end{proof}

\section{Proof on Basic Path Properties of Split ISS, Split LBISS and Split LBI}
\label{sec:proof-slb-basic-properties}

\begin{proof}[Proof of \Cref{thm:slbiss-exi-uni}]
    For Split ISS, by \cref{eq:siss-a} and the fact that $\beta(t)\in \mathcal{L} = \mathrm{Im}(X^T) + \mathrm{Im}(D^T) = \mathrm{Im}(A) = \mathrm{Im}(A^{\dag})$, we can solve $\beta(t) = A^{\dag} (\nu X^{*} y + D^T \gamma(t))$ which is determined by $\gamma(t)$. Plugging it into \cref{eq:siss-b} we have
    \begin{equation*}
        \dot{\rho}(t) + \dot{\gamma}(t) / \kappa = - \Sigma \gamma(t) + D A^{\dag} X^{*} y.
    \end{equation*}
    Taking $M = I_{p+m} - (\sqrt{\nu/n} X^T, D^T)^{\dag} (\sqrt{\nu/n} X^T, D^T)$ in Theorem 1.19 in \citet{zhang_schur_2006} leads to
    \begin{equation}
        \label{eq:Sigma-DAdagX}
        D A^{\dag} X^{*} = \Sigma \Sigma^{\dag} \left( D A^{\dag} X^{*} \right) = \Sigma^{1/2} \Sigma^{\dag 1/2} \left( D A X^{*} \right).
    \end{equation}
    The inclusion becomes
    \begin{equation*}
        \dot{\rho}(t) + \dot{\gamma}(t) / \kappa = - \Sigma^{1/2} \left( \Sigma^{1/2} \gamma(t) - \Sigma^{\dag 1/2} D A^{\dag} X^{*} y \right),
    \end{equation*}
    which is a standard ISS (on $\gamma(t)$) and has been sufficiently discussed in \citet{osher_sparse_2016} (let $X, y$ in that paper take $\sqrt{n} \Sigma^{1/2}$ and $\sqrt{n} \Sigma^{\dag 1/2} D A^{\dag} X^{*} y$ in this paper). Specifially, there exists a solution with piecewise linear $\rho(t)$ and piecewise constant $\beta(t), \gamma(t)$. Besides, $\rho(t)$ is unique. If additionally, when $\Sigma_{S(t), S(t)} \succ 0$, we have that $\Sigma_{\cdot, S(t)}$ has full column rank, and $\gamma(t)$ (hence $\beta(t)$) is unique.

    For Split LBISS, letting $z(t) = \rho(t) + \gamma(t) / \kappa$ and noting \cref{eq:var-subst}, the Split LBISS \cref{eq:slbiss} is equivalent to
    \begin{equation*}
        \begin{pmatrix} \dot{\beta}(t) \\ \dot{z}(t) \end{pmatrix} = - \begin{pmatrix} - \kappa X^{*} \left( X \beta(t) - y \right) - \kappa D^T \left( D \beta(t) - \kappa \mathcal{S}(z(t), 1) \right) / \nu \\ - \left( \kappa \mathcal{S}(z(t), 1) - D \beta(t) \right) / \nu \end{pmatrix}.
    \end{equation*}
    The Picard-Lindel\"{o}f Theorem implies that this ODE has a unique solution $(\beta(t), z(t))$, so there exists a unique solution to the Split LBISS \cref{eq:slbiss}. \qed
\end{proof}

\begin{proof}[Proof of \Cref{thm:slb-l-dec}]
    For Split ISS, one can easily imitates the technique in the proof of Theorem 2.1 in \citet{osher_sparse_2016} to show that $(\beta(t), \gamma(t))$ is the solution of the following optimization problem.
    \begin{equation}
        \label{eq:siss-opt}
        \begin{aligned}
            \min_{\beta, \gamma} & & &\ell\left( \beta, \gamma \right) \\
            \text{subject to} & & &
            \begin{cases}
                \gamma_j \ge 0, & \text{if}\ \rho_j(t) = 1,\\
                \gamma_j \le 0, & \text{if}\ \rho_j(t) = -1,\\
                \gamma_j = 0, & \text{if}\ \rho_j(t) \in (-1,1).
            \end{cases}
        \end{aligned}
    \end{equation}
    for any $t > 0$, due to the continuity of $\rho(\cdot)$, there is a small neighborhood of $t$, on which every $\tau$ satisfies
    \begin{equation*}
        \begin{cases}
            \rho_j\left( \tau \right) > -1\ \text{hence}\ \gamma_j\left( \tau \right)\ge 0, & \text{if}\ \rho_j(t) = 1,\\
            \rho_j\left( \tau \right) < 1\ \text{hence}\ \gamma_j\left( \tau \right)\ge 0, & \text{if}\ \rho_j(t) = - 1,\\
            \rho_j\left( \tau \right) \in (-1,1)\ \text{hence}\ \gamma_j\left( \tau \right) = 0, & \text{if}\ \rho_j(t) \in (-1,1).
        \end{cases}
    \end{equation*}
    That is to say, $(\beta(\tau), \gamma(\tau))$ satisfies the constraints in \cref{eq:siss-opt}, so the value of $\ell(\beta(\tau), \gamma(\tau))$ is not less than $\ell(\beta(t), \gamma(t))$ (the solution of \cref{eq:siss-opt}). This implies that any $t\ge 0$ is a local minimal point of a right continuous function $\ell(\beta(\cdot), \gamma(\cdot))$. Then by standard techniques in mathematical analysis, we have that $\ell(\beta(t), \gamma(t))$ is non-increasing.

    For Split LBISS, by \cref{eq:slbiss-c}, we have $\dot{\gamma}_j(t) \cdot \dot{\rho}_j(t) \equiv 0$ for each $j$, so $\ell$ is non-increasing since
    \begin{multline*}
        \frac{\mathrm{d}}{\mathrm{d}t} \ell(\beta(t), \gamma(t)) = \left\langle \begin{pmatrix} \dot{\beta}(t)\\ \dot{\gamma}(t) \end{pmatrix},\ \begin{pmatrix} \nabla_\beta \ell\left( \beta(t), \gamma(t) \right) \\ \nabla_\gamma \ell\left( \beta(t), \gamma(t) \right) \end{pmatrix} \right\rangle\\
        = \left\langle \begin{pmatrix} \dot{\beta}(t)\\ \dot{\gamma}(t) \end{pmatrix},\ \begin{pmatrix} - \dot{\beta}(t) / \kappa \\ - \dot{\rho}(t) - \dot{\gamma}(t) / \kappa \end{pmatrix} \right\rangle = \frac{1}{\kappa} \left\| \begin{pmatrix} \dot{\beta}(t) \\ \dot{\gamma}(t) \end{pmatrix} \right\|_2^2\le 0.
    \end{multline*}

    For Split LBI, noting $(\rho_{k+1} - \rho_k)(\gamma_{k+1} - \gamma_k) = \| \rho_{k+1} \|_1 - \langle \rho_{k+1}, \gamma_k \rangle + \| \gamma_{k+1} \|_1 - \langle \rho_k, \gamma_{k+1} \rangle \ge 0$, we have
    \begin{multline*}
        - \alpha \nabla \ell\left( \beta_k, \gamma_k \right)^T \begin{pmatrix} \beta_{k+1} - \beta_k \\ \gamma_{k+1} - \gamma_k \end{pmatrix}\\
        = \left( \begin{pmatrix} 0 \\ \rho_{k+1} - \rho_k \end{pmatrix} + \frac{1}{\kappa} \begin{pmatrix} \beta_{k+1} - \beta_k \\ \gamma_{k+1} - \gamma_k \end{pmatrix} \right) \begin{pmatrix} \beta_{k+1} - \beta_k \\ \gamma_{k+1} - \gamma_k \end{pmatrix} \ge \frac{1}{\kappa} \left\| \begin{pmatrix} \beta_{k+1} - \beta_k \\ \gamma_{k+1} - \gamma_k \end{pmatrix} \right\|_2^2.
    \end{multline*}
    By $\kappa \alpha \| H \|_2 < 2$, we have
    \begin{multline*}
        \ell \left( \beta_{k+1}, \gamma_{k+1} \right) - \ell\left( \beta_k, \gamma_k \right)\\
        = \nabla \ell\left( \beta_k, \gamma_k \right)^T \begin{pmatrix} \beta_{k+1} - \beta_k \\ \gamma_{k+1} - \gamma_k \end{pmatrix} + \frac{1}{2} \left( \beta_{k+1}^T - \beta_k^T, \gamma_{k+1}^T - \gamma_k^T \right) H \begin{pmatrix} \beta_{k+1} - \beta_k \\ \gamma_{k+1} - \gamma_k \end{pmatrix}\\
        \le - \frac{1}{\kappa \alpha} \left\| \begin{pmatrix} \beta_{k+1} - \beta_k \\ \gamma_{k+1} - \gamma_k \end{pmatrix} \right\|_2^2 + \frac{\left\| H \right\|_2}{2} \cdot \left\| \begin{pmatrix} \beta_{k+1} - \beta_k \\ \gamma_{k+1} - \gamma_k \end{pmatrix} \right\|_2^2\le 0.
    \end{multline*}

    Moreover, it is easy to verify that
    \begin{multline}
        \label{eq:H-upper}
        \left( \beta^T, \gamma^T \right) H \begin{pmatrix} \beta\\ \gamma \end{pmatrix} = \frac{1}{n} \left\| X \beta \right\|_2^2 + \frac{1}{\nu} \left\| D \beta - \gamma \right\|_2^2 \le \frac{2}{n} \left\| X \beta \right\|_2^2 + \frac{2}{\nu} \left\| D \beta \right\|_2^2 + 2 \left\| \gamma \right\|_2^2\\
        \le \frac{2 \left( 1 + \nu \Lambda_X^2 + \Lambda_D^2 \right)}{\nu} \left\| \begin{pmatrix} \beta\\ \gamma \end{pmatrix} \right\|_2^2\ \left( \begin{pmatrix} \beta\\ \gamma \end{pmatrix} \in \mathbb{R}^{m+p} \right),\\
        \Longrightarrow \left\| H \right\|_2 \le \frac{2 \left( 1 + \nu \Lambda_X^2 + \Lambda_D^2 \right)}{\nu}. 
    \end{multline}
\end{proof}

\section{Proof on Equivalence of Assumptions}
\label{sec:proof-equiv-assump}

\begin{proof}[Proof of \Cref{thm:rsc-nu}]
    If there exists $C > 0,\ \nu > 0$ such that \cref{eq:rsc-nu} holds, then for $\beta\in \mathcal{L}\cap \mathcal{M}$, taking $\gamma_S = D_S \beta$ and noting $D_{S^c} \beta = 0$, the left hand side of \cref{eq:rsc-nu} is
    \begin{equation*}
        \frac{1}{n} \left\| X \beta \right\|_2^2 + \frac{1}{\nu} \left\| \gamma_S - D_S \beta \right\|_2^2 + \frac{1}{\nu} \left\| D_{S^c} \beta \right\|_2^2 = \frac{1}{n} \left\| X \beta \right\|_2^2.
    \end{equation*}
    which should be not less than $(C/(1+\nu))\|\beta\|_2^2$. Thus \Cref{thm:rsc} holds for $\lambda = C / (1+\nu) > 0$.

    If \Cref{thm:rsc} holds for some $\lambda > 0$, for any $\beta\in \mathcal{L}$, let $\beta = \beta' + \beta''$ where $\beta'\in \mathcal{L} \cap \mathcal{M}$ and $\beta''\in \mathcal{L} \cap \mathcal{M}^{\perp}$. Since $D_{S^c} \beta' = 0,\ \beta'' \in \mathcal{M}^{\perp} = \mathrm{Im}(D_{S^c}^T)$, we have
    \begin{equation*}
        \beta^T D_{S^c}^T D_{S^c} \beta = \beta''^T D_{S^c}^T D_{S^c} \beta'' \ge \lambda_D^2 \left\| \beta'' \right\|_2^2.
    \end{equation*}
    For constant $\nu_0 = 2 \lambda_D^2 / (\lambda + 2 \Lambda_X^2) > 0$ we have
    \begin{align*}
        & \beta^T \left( \nu_0 X^{*} X + D_{S^c}^T D_{S^c} \right) \beta = \nu_0 \cdot \beta^T X^{*} X \beta + \beta^T D_{S^c}^T D_{S^c} \beta\\
        \ge{} & \nu_0 \left( 2 \left( \beta' / 2 \right)^T X^{*} X \left( \beta' / 2 \right) - \left( - \beta'' \right)^T X^{*} X \left( - \beta'' \right) \right) + \lambda_D^2 \left\| \beta'' \right\|_2^2\\
        \ge{} & \frac{\lambda \nu_0}{2} \left\| \beta' \right\|_2^2 + \left( \lambda_D^2 - \nu_0 \Lambda_X^2 \right) \left\| \beta'' \right\|_2^2 = \frac{\lambda \nu_0}{2} \left( \left\| \beta' \right\|_2^2 + \left\| \beta'' \right\|_2^2 \right) = \frac{\lambda \nu_0}{2} \left\| \beta \right\|_2^2.
    \end{align*}
    The left hand side of \cref{eq:rsc-nu}, denoted by $L$ or $L(\nu)$, satisfies
    \begin{multline*}
        L \ge \frac{1}{n} \left\| X \beta \right\|_2^2 + \frac{1}{\nu} \left\| D_{S^c} \beta \right\|_2^2 \\
        \ge \frac{1}{\max(\nu_0, \nu)} \beta^T \left( \nu_0 X^{*} X + D_{S^c}^T D_{S^c} \right) \beta \ge \frac{\lambda \nu_0}{2(\nu_0 + \nu)} \left\| \beta \right\|_2^2.
    \end{multline*}
    Furthermore, by the inequality above and Cauchy's inequality,
    \begin{multline*}
        L \ge \frac{\lambda \nu_0}{2(\nu_0 + \nu)} \left\| \beta \right\|_2^2 + \frac{1}{\nu} \left\| \gamma_S - D_S \beta \right\|_2^2\\
        \ge \frac{\lambda \nu_0}{2 \Lambda_D^2 (\nu_0 + \nu)} \left\| D_S \beta \right\|_2^2 + \frac{1}{\nu} \left\| \gamma_S - D_S \beta \right\|_2^2 \ge \frac{1}{2\Lambda_D^2(\nu_0 + \nu)/(\lambda \nu_0) + \nu} \left\| \gamma_S \right\|_2^2,
    \end{multline*}
    Consequently,
    \begin{equation*}
        \left\| \begin{pmatrix} \beta\\ \gamma_S \end{pmatrix} \right\|_2^2 \le \left( \frac{2(\nu_0 + \nu)}{\lambda \nu_0} + \left( \frac{2 \Lambda_D^2 (\nu_0 + \nu)}{\lambda \nu_0} + \nu \right) \right) L \le \frac{1 + \nu}{C} L,
    \end{equation*}
    where
    \begin{equation}
        \label{eq:C-def}
        C = \frac{\lambda\cdot \min(\nu_0, 1)}{2 + 2\Lambda_D^2 + \lambda \nu_0} \left( \nu_0 = \frac{2 \lambda_D^2}{\lambda + 2 \Lambda_X^2} > 0 \right)
    \end{equation}
    is a constant. Thus \cref{eq:rsc-nu} holds for all $\nu > 0$. \qed
\end{proof}

\begin{proof}[Proof of \Cref{thm:rsc-nu-prime}]
    Let $L = L(\nu)$ denotes the left hand side of \cref{eq:rsc-nu}. Suppose there exists $C > 0,\ \nu_0 > 0$ such that \cref{eq:rsc-nu} holds for $\nu = \nu_0$. Since
    \begin{equation}
        \label{eq:H-nu-nu0}
        H_{(\beta,S), (\beta,S)}(\nu) \ge \min\left( 1, \frac{\nu_0}{\nu} \right) H_{(\beta,S), (\beta,S)}(\nu_0) \ge \frac{\nu_0}{\nu_0 + \nu} H_{(\beta,S), (\beta,S)}(\nu_0),
    \end{equation}
    we can find $C_1 > 0$ such that \cref{eq:rsc-nu} holds for $C = C_1$ and all $\nu > 0$. Now
    \begin{equation}
        \label{eq:H-SS}
        H_{(\beta,S), (\beta, S)} = Q M Q^T,\ \text{where}\ Q := \begin{pmatrix} I_p & 0 \\ - D_S A^{\dag} & I_s \end{pmatrix},\ M := \begin{pmatrix} A / \nu & 0 \\ 0 & \Sigma_{S,S} \end{pmatrix}.
    \end{equation}
    So
    \begin{equation*}
        L = \begin{pmatrix} \beta - A^{\dag} D_S^T \gamma_S \\ \gamma_S \end{pmatrix}^T \begin{pmatrix} A / \nu & 0 \\ 0 & \Sigma_{S,S} \end{pmatrix} \begin{pmatrix} \beta - A^{\dag} D_S^T \gamma_S \\ \gamma_S \end{pmatrix}.
    \end{equation*}
    $L \ge (C_1/(1+\nu)) \| (\beta; \gamma_S) \|_2^2$ implies $\gamma_S^T \Sigma_{S,S} \gamma_S \ge (C_1/(1+\nu)) \| \gamma_S \|_2^2$ (letting $\beta = A^{\dag} D_S^T \gamma_S \in \mathcal{L}$). So \cref{eq:rsc-nu-prime} holds for $C' = C_1$ and all $\nu > 0$.

    Suppose there exists $C' > 0,\ \nu_0 > 0$ such that \cref{eq:rsc-nu-prime} holds for $\nu = \nu_0$. For any $\beta\in \mathcal{L}$, represent $\beta = V \delta + \tilde{V} V_1 \xi$ by \Cref{thm:beta-decomposition}, then
    \begin{multline*}
        \beta^T \left( \nu_0 X^{*} X + D^T D \right) \beta\\
        = \left( \delta^T, \xi^T \right) \begin{pmatrix} \Lambda^2 + \nu_0 V^T X^{*} X V & \nu_0 V^T X^{*} X \tilde{V} V_1\\ \nu_0 V_1^T \tilde{V}^T X^{*} X V & \nu_0 V_1^T \tilde{V}^T X^{*} X \tilde{V} V_1 \end{pmatrix} \begin{pmatrix} \delta \\ \xi \end{pmatrix}\\
        = \left( \delta^T, \xi^T \right) \begin{pmatrix} \Lambda^2 + \nu_0 V^T X^{*} X V & \nu_0 V^T X^T U_1 \Lambda_1 / \sqrt{n}\\ \nu_0 \Lambda_1 U_1^T X V / \sqrt{n} & \nu_0 \Lambda_1^2 \end{pmatrix} \begin{pmatrix} \delta \\ \xi \end{pmatrix}.
    \end{multline*}
    For $P = \Lambda^2 + \nu_0 V^T X^{*} X V,\ Q = \nu_0 V^T X^T U_1 \Lambda_1 / \sqrt{n},\ R = \nu_0 \Lambda_1^2$,
    \begin{equation*}
        P - Q R^{\dag} Q^T = \Lambda^2 + \nu_0 V^T X^{*} (I - U_1 U_1^T) X V \succeq \lambda_D^2 I \succeq \frac{\lambda_D^2}{\nu_0 \Lambda_X^2 + \Lambda_D^2} \cdot P.
    \end{equation*}
    By \Cref{thm:schur} we have
    \begin{multline*}
        \beta^T \left( \nu_0 X^{*} X + D^T D \right) \beta \ge \frac{\lambda_D^2}{\nu_0 \Lambda_X^2 + \Lambda_D^2} \cdot \frac{1}{1/\lambda_{\min}(P) + 1/\lambda_{\min}(R)} \left\| \begin{pmatrix} \delta\\ \xi \end{pmatrix} \right\|_2^2\\
        \ge \frac{\lambda_D^2}{\nu_0 \Lambda_X^2 + \Lambda_D^2} \cdot \frac{1}{1/\lambda_D^2 + 1/(\nu_0 \lambda_1^2)} \left\| \beta \right\|_2^2.
    \end{multline*}
    Note that
    \begin{equation*}
        H_{S,S}(\nu_0) - H_{S, \beta}(\nu_0) {H_{\beta, \beta}(\nu_0)}^{\dag} H_{\beta, S}(\nu_0) = \Sigma_{S,S}(\nu_0) \succeq \frac{C'}{1 + \nu_0} I = \frac{C' \nu_0}{1 + \nu_0} H_{S,S}(\nu_0).
    \end{equation*}
    By \Cref{thm:schur} we have
    \begin{multline*}
        H_{\beta, \beta}(\nu_0) - H_{\beta, S}(\nu_0) {H_{S,S}(\nu_0)}^{\dag} H_{S, \beta}(\nu_0) \succeq \frac{C' \nu_0}{1 + \nu_0} H_{\beta, \beta}(\nu_0)\\
        \Longrightarrow \left( 1 - \frac{C' \nu_0}{1 + \nu_0} \right) \left( \nu_0 X^{*} X + D_{S^c}^T D_{S^c} \right) \succeq \frac{C' \nu_0}{1 + \nu_0} D_S^T D_S\\
        \Longrightarrow \nu_0 X^{*} X + D_{S^c}^T D_{S^c} \succeq \frac{C' \nu_0}{1 + \nu_0} \left( \nu_0 X^{*} X + D^T D \right).
    \end{multline*}
    Thus
    \begin{multline*}
        L(\nu_0) = \frac{1}{2n} \left\| X \beta \right\|_2^2 + \frac{1}{2\nu_0} \left\| \gamma_S - D_S \beta \right\|_2^2 + \left\| D_{S^c} \beta \right\|_2^2\\
        \ge \frac{1}{2\nu_0} \beta^T \left( \nu_0 X^{*} X + D_{S^c}^T D_{S^c} \right) \beta \ge \frac{C'}{2(1 + \nu_0)} \beta^T \left( \nu_0 X^{*} X + D^T D \right) \beta\ge C_1' \left\| \beta \right\|_2^2.
    \end{multline*}
    Where $C_1' > 0$ is a constant. Besides,
    \begin{equation*}
        L(\nu_0) \ge \gamma_S^T \left( H_{S,S}(\nu_0) - H_{S, \beta}(\nu_0) H_{\beta, \beta}(\nu_0)^{\dag} H_{\beta, S}(\nu_0) \right) \gamma_S \ge \frac{C'}{1+\nu_0} \left\| \gamma_S \right\|_2^2.
    \end{equation*}
    Thus we can find $C_2' > 0$ such that $L(\nu_0) \ge C_2' \left\| (\beta; \gamma_S) \right\|_2^2$. Combining with \cref{eq:H-nu-nu0}, we can find $C>0$ such that \cref{eq:rsc-nu} holds for all $\nu > 0$. \qed
\end{proof}

\begin{proof}[Proof of \Cref{thm:irr-nu-prime}]
    Under \Cref{thm:rsc}, by \Cref{thm:rsc-nu} and \ref{thm:rsc-nu-prime} we have $\Sigma_{S,S} \succ 0$. By \cref{eq:H-SS}, we know
    \begin{multline*}
        \mathrm{rank}\left( H_{(\beta,S),(\beta,S)} \right) = \mathrm{rank} \left( \begin{pmatrix} A / \nu & 0 \\ 0 & \Sigma_{S,S} \end{pmatrix} \right)\\
        = \mathrm{rank}(A) + \mathrm{rank}\left( \Sigma_{S,S} \right) = \mathrm{rank}\left( H_{\beta, \beta} \right) + \mathrm{rank}\left( H_{S,S} \right).
    \end{multline*}
    Then by Theorem 1.21 in \citet{zhang_schur_2006}, we have that
    \begin{equation*}
        {H_{(\beta,S), (\beta,S)}}^{\dag} = \begin{pmatrix} \nu A^{\dag} + A^{\dag} D_S^T \Sigma_{S,S}^{-1} D_S A^{\dag} & A^{\dag} D_S^T \Sigma_{S,S}^{-1} \\ \Sigma_{S,S}^{-1} D_S A^{\dag} & \Sigma_{S,S}^{-1} \end{pmatrix}.
    \end{equation*}
    By $H_{S^c, (\beta,S)} = (- D_{S^c} / \nu, 0)$ and $- D_{S^c} A^{\dag} D_S / \nu = \Sigma_{S^c, S}$, we have
    \begin{equation}
        \label{eq:ic-equiv}
        H_{S^c, (\beta,S)} {H_{(\beta,S), (\beta,S)}}^{\dag} = \left( - D_{S^c} A^{\dag} + \Sigma_{S^c, S} \Sigma_{S,S}^{-1} D_S A^{\dag},\ \Sigma_{S^c, S} \Sigma_{S,S}^{-1} \right).
    \end{equation}
    The rest is easy. \qed
\end{proof}

\section{Proof of the Comparison Theorem}
\label{sec:proof-ic-compare}

\begin{proof}[Proof of \Cref{thm:ic-compare}]
    By definition, we have $\mathrm{ic}_0 \ge \| \Omega^S \mathrm{sign}(D_S \beta^{\star}) \|_{\infty} \ge \mathrm{ic}_1$. Now we prove $\mathrm{irr}(0)$ exists and $\mathrm{irr}(0) = \mathrm{ic}_0$. Let $M := \Lambda^{-1} V^T X^{*} (I - U_1 U_1^T) X V \Lambda^{-1}$. When $\nu$ is small, by \cref{eq:Sigma},
    \begin{multline*}
        \nu \Sigma = I - U \Lambda B^{-1} \Lambda U^T = I - U \left( I + \nu M \right)^{-1} U^T\\
        = I - U \left( I - \nu M + O\left( \nu^2 \right) \right) U^T = I - U U^T + \nu U M U^T + O\left( \nu^2 \right)\\
        \Longrightarrow \nu \Sigma_{S^c, S} = - U_{S^c} U_S^T + \nu U_{S^c} M U_S^T + O\left( \nu^2 \right),\ \nu \Sigma_{S,S} = I - U_S U_S^T + \nu U_S M U_S^T + O\left( \nu^2 \right).
    \end{multline*}
    Let $F := I - U_S U_S^T$ and $F = U' \Lambda' U'^T$ be the ``compact'' eigendecomposition of $F$ ($\Lambda' \succ 0$). Let $G := U_S M U_S^T$. Suppose $(U', \tilde{U}')$ is an orthogonal square matrix, and
    \begin{equation*}
        K = \begin{pmatrix} K_1 & K_2\\ K_2^T & K_3 \end{pmatrix} := \begin{pmatrix} U'^T\\ \tilde{U}'^T \end{pmatrix} G \left( U', \tilde{U}' \right).
    \end{equation*}
    By $F + \nu G \succ 0$, we have $K_3\succ 0$. Now
    \begin{equation*}
        F + \nu G = \left( U', \tilde{U}' \right) \begin{pmatrix} \Lambda' + \nu K_1 & \nu K_2 \\ \nu K_2^T & \nu K_3 \end{pmatrix} \begin{pmatrix} U'^T \\ \tilde{U}'^T \end{pmatrix}.
    \end{equation*}
    Define $Q_\nu = K_3 - \nu K_2^T (\Lambda' + \nu K_1)^{-1} K_2,\ R_\nu = K_2^T (\Lambda' + \nu K_1)^{-1}$, and we can calculate
    \begin{equation*}
        (F + \nu G)^{-1} = \left( U', \tilde{U}' \right) \begin{pmatrix} (\Lambda' + \nu K_1)^{-1} + \nu R_\nu^T Q_\nu^{-1} R_\nu & - R_\nu^T Q_\nu^{-1}\\ - Q_\nu^{-1} R_\nu & Q_\nu / \nu \end{pmatrix} \begin{pmatrix} U'^T \\ \tilde{U}'^T \end{pmatrix}.
    \end{equation*}
    Note that $Q_\nu \rightarrow K_3, R_\nu \rightarrow K_2^T \Lambda'^{-1}$, and note that
    \begin{multline}
        \label{eq:ker-ut-sub-im}
        U_{S^c}^T U_{S^c} U_S^T \tilde{U}' = \left( I - U_S^T U_S \right) U_S^T \tilde{U}' = U_S^T \left( I - U_S U_S^T \right) \tilde{U}' = U_S^T U' \Lambda' \cdot U'^T \tilde{U}' = 0\\
        \Longrightarrow \left( U_{S^c} U_S^T \tilde{U}' \right)^T U_{S^c} U_S^T \tilde{U}' = 0 \Longrightarrow U_{S^c} U_S^T \tilde{U}' = 0.
    \end{multline}
    Combining it with the representation of $(F + \nu G)^{-1}$,
    \begin{multline*}
        - U_{S^c} U_S^T \Sigma_{S,S}^{-1} \doteq U_{S^c} U_S^T (F + \nu G)^{-1}\\
        = - \left( U_{S^c} U_S^T U', 0 \right) \begin{pmatrix} (\Lambda' + \nu K_1)^{-1} + \nu R_\nu^T Q_\nu^{-1} R_\nu & - R_\nu^T Q_\nu^{-1}\\ - Q_\nu^{-1} R_\nu & \star \end{pmatrix} \begin{pmatrix} U'^T \\ \tilde{U}'^T \end{pmatrix}\\
        \rightarrow \left( - U_{S^c} U_S^T U' \Lambda'^{-1}, U_{S^c} U_S^T U' \Lambda'^{-1} K_2 K_3^{-1} \right) \begin{pmatrix} U'^T\\ \tilde{U}' \end{pmatrix} = - U_{S^c} U_S^T U' \Lambda'^{-1} \left( U'^T - K_2 K_3^{-1} \tilde{U}'^T \right).
    \end{multline*}
    Besides,
    \begin{equation*}
        \nu U_{S^c} M U_S^T \Sigma_{S,S}^{-1} \doteq U_{S^c} M U_S^T \cdot \nu \left( F + \nu G \right)^{-1} \rightarrow U_{S^c} M U_S^T \tilde{U}' K_3^{-1} \tilde{U}'^T.
    \end{equation*}
    So when $\nu \rightarrow 0$,
    \begin{multline*}
        \Sigma_{S^c, S} \Sigma_{S,S}^{-1} \rightarrow - U_{S^c} U_S^T U' \Lambda'^{-1} \left( U'^T - K_2 K_3^{-1} \tilde{U}'^T \right) + U_{S^c} M U_S^T \tilde{U}' K_3^{-1} \tilde{U}'^T\\
        = - U_{S^c} U_S^T U' \Lambda'^{-1} U'^T + U_{S^c} \left( U_S^T U' \Lambda'^{-1} U'^T U_S + I \right) M U_S^T \tilde{U}' K_3^{-1} \tilde{U}'^T\\
        = - D_{S^c} V \Lambda^{-1} U_S^T U' \Lambda'^{-1} U'^T + D_{S^c} V \Lambda^{-1} \left( I + U_S^T U' \Lambda'^{-1} U'^T U_S \right) M U_S^T \tilde{U}' K_3^{-1} \tilde{U}'^T.
    \end{multline*}
    The infinity norm of the right hand side is $\mathrm{irr}(0)$. On the other hand,
    \begin{equation*}
        \mathrm{ic}_0 = \left\| D_{S^c} \left( D_{S^c}^T D_{S^c} \right)^{\dag} \left( X^{*} X W \left( W^T X^{*} X W \right)^{\dag} W^T - I \right) D_S^T \right\|_{\infty}.
    \end{equation*}
    In order to prove $\mathrm{irr}(0) = \mathrm{ic}_0$, it suffices to show
    \begin{multline*}
        \left( X^{*} X W \left( W^T X^{*} X W \right)^{\dag} W^T - I \right) D_S^T = - D_{S^c}^T D_{S^c} V \Lambda^{-1} U_S^T U' \Lambda'^{-1} U'^T\\
        + D_{S^c}^T D_{S^c} V \Lambda^{-1} \left( I + U_S^T U' \Lambda'^{-1} U'^T U_S \right) M U_S^T \tilde{U}' K_3^{-1} \tilde{U}'^T.
    \end{multline*}
    The first term of the right hand side is
    \begin{multline*}
        - V \Lambda U_{S^c}^T U_{S^c} U_S^T U' \Lambda'^{-1} U'^T = - V \Lambda \left( I - U_S^T U_S \right) U_S^T U' \Lambda'^{-1} U'^T\\
        = - V \Lambda U_S^T \left( I - U_S U_S^T \right) U' \Lambda'^{-1} U'^T = - V \Lambda U_S^T U' \Lambda' U'^T U' \Lambda'^{-1} U'^T = - D_S^T U' U'^T,
    \end{multline*}
    while by the fact that
    \begin{multline*}
        \left( I - U_S^T U_S \right) \left( I + U_S^T U' \Lambda'^{-1} U'^T U_S \right)\\
        = I - U_S^T U_S + U_S^T U' \Lambda'^{-1} U'^T U_S - U_S^T U_S U_S^T U' \Lambda'^{-1} U'^T U_S\\
        = I - U_S^T U_S + U_S^T \left( I - U_S U_S^T \right) U' \Lambda'^{-1} U'^T U_S\\
        = I - U_S^T U_S + U_S^T U' U'^T U_S = I - U_S^T \tilde{U}' \tilde{U}'^T U_S,
    \end{multline*}
    the second term becomes
    \begin{multline*}
        V \Lambda U_{S^c}^T U_{S^c} \left( I + U_S^T U' \Lambda'^{-1} U'^T U_S \right) M U_S^T \tilde{U}' K_3^{-1} \tilde{U}'^T\\
        = V \Lambda \left( I - U_S^T U_S \right) \left( I + U_S^T U' \Lambda'^{-1} U'^T U_S \right) M U_S^T \tilde{U}' K_3^{-1} \tilde{U}'^T\\
        = V \Lambda \left( I - U_S^T \tilde{U}' \tilde{U}'^T U_S \right) M U_S^T \tilde{U}' K_3^{-1} \tilde{U}'^T\\
        = V \Lambda M U_S^T \tilde{U}' K_3^{-1} \tilde{U}'^T - V \Lambda U_S^T \tilde{U}' \cdot \tilde{U}'^T U_S M U_S^T \tilde{U}' \cdot K_3^{-1} \tilde{U}'^T\\
        = X^{*} \left( I - U_1 U_1^T \right) XV \Lambda^{-1} U_S^T \tilde{U}' K_3^{-1} \tilde{U}' - D_S^T \tilde{U}' \cdot K_3 \cdot K_3^{-1} \tilde{U}'^T\\
        = X^{*} \left( I - U_1 U_1^T \right) XV \Lambda^{-1} U_S^T \tilde{U}' K_3^{-1} \tilde{U}'^T - D_S^T \tilde{U}' \tilde{U}'^T.
    \end{multline*}
    So it suffices to show
    \begin{equation*}
        X^{*} X W \left( W^T X^{*} X W \right)^{\dag} W^T D_S^T = X^{*} \left( I - U_1 U_1^T \right) XV \Lambda^{-1} U_S^T \tilde{U}' K_3^{-1} \tilde{U}'^T,
    \end{equation*}
    which is equivalent to
    \begin{equation}
        \label{eq:ic-0-ic0-remain}
        X^{*} \left( X W W^T X^{*} \right)^{\dag} X W W^T D_S^T = X^{*} \left( I - U_1 U_1^T \right) XV \Lambda^{-1} U_S^T \tilde{U}' K_3^{-1} \tilde{U}'^T.
    \end{equation}
    First we prove
    \begin{equation}
        \label{eq:ker-ut-eq-im}
        \ker\left( U_{S^c} \right) = \mathrm{Im} \left( U_S^T \tilde{U}' \right).
    \end{equation}
    In fact, by \cref{eq:ker-ut-sub-im} we have $\mathrm{Im}(U_S^T \tilde{U}') \subseteq \ker(U_{S^c})$. For any $\zeta\in \ker(U_{S^c})$, we have $(I - U_S^T U_S) \zeta = U_{S^c}^T U_{S^c} \zeta = 0$. Let
    \begin{equation*}
        \zeta = U_S^T \zeta_1 + \zeta_2,\ \zeta_2 \in \mathrm{ker}(U_S),
    \end{equation*}
    then
    \begin{equation*}
        0 = (I - U_S^T U_S) (U_S^T \zeta_1 + \zeta_2) = \zeta_2 + (I - U_S^T U_S) U_S^T \zeta_1 = \zeta_2 + U_S^T (I - U_S U_S^T) \zeta_1,
    \end{equation*}
    which implies $\zeta_2\in \mathrm{Im}(U_S^T)$. But $\zeta_2\in \ker(U_S)$, then $\zeta_2 = 0$, and $0 = (I - U_S^T U_S) U_S^T \zeta_1 = U_S^T (I - U_S U_S^T) \zeta_1 = U_S^T U' \Lambda' U'^T \zeta_1$. Assume that $\zeta_1 = U' \zeta_3 + \tilde{U}' \tilde{\zeta}_3$, then $U_S^T U' \Lambda' \zeta_3 = 0$. Thus
    \begin{multline*}
        0 = U_S U_S^T U' \Lambda' \zeta_3 = \left( I - U' \Lambda' U'^T \right) U' \Lambda' \zeta_3 = U' \Lambda' \left( I - \Lambda' \right) \zeta_3\Longrightarrow \left( I - \Lambda' \right) \zeta_3 = 0\\
        \Longrightarrow U_S U_S^T U' \zeta_3 = U' \left( I - \Lambda' \right) \zeta_3 = 0 \Longrightarrow \left( U_S^T U' \zeta_3 \right)^T U_S^T U' \zeta_3 = 0\Longrightarrow U_S^T U' \zeta_3 = 0\\
        \Longrightarrow \beta = U_S^T \zeta_1 = U_S^T U' \zeta_3 + U_S^T \tilde{U}' \tilde{\zeta}_3 = U_S^T \tilde{U}' \tilde{\zeta}_3 \in \mathrm{Im}\left( U_S^T \tilde{U}' \right).
    \end{multline*}
    So \cref{eq:ker-ut-eq-im} holds. Now for any $\beta\in \mathbb{R}^p$, let $\beta = V \delta + \tilde{V} \tilde{\delta}$, then $\beta\in \ker(D_{S^c})$ if and only if $U_{S^c} \Lambda \delta = 0$, which means $\delta \in \Lambda^{-1} \ker(U_{S^c}) = \mathrm{Im}(\Lambda^{-1} U_S^T \tilde{U}')$. So
    \begin{equation*}
        \ker\left( D_{S^c} \right) = \mathrm{Im} \left( J \right) + \mathrm{Im} \left( \tilde{V} \right),\ \text{where}\ J := V \Lambda^{-1} U_S^T \tilde{U}'.
    \end{equation*}
    Since $\tilde{V}^T V = 0$, the linear subspaces spanned by $J$ and $\tilde{V}$ are orthogonal, and we have
    \begin{equation*}
        W W^T = J \left( J^T J \right)^{\dag} J^T + \tilde{V} \tilde{V}^T.
    \end{equation*}
    Noting $\tilde{V}^T V = 0,\ \tilde{V}^T X^{*} (I - U_1 U_1^T) = 0$, we have
    \begin{align*}
        & X^{*} \left( X W W^T X^{*} \right)^{\dag} XW W^T D_S^T \tilde{U}' \tilde{U}'^T K_3\\
        = {}& X^{*} \left( X W W^T X^{*} \right)^{\dag} X J \left( J^T J \right)^{\dag} J^T V \Lambda U_S^T \tilde{U}' \tilde{U}'^T K_3\\
        = {}& X^{*} \left( X W W^T X^{*} \right)^{\dag} X J \left( J^T J \right)^{\dag} \cdot \tilde{U}'^T U_S U_S^T \tilde{U}' \cdot \tilde{U}'^T U_S M U_S^T \tilde{U}'\\
        = {}& X^{*} \left( X W W^T X^{*} \right)^{\dag} X J \left( J^T J \right)^{\dag} \cdot \tilde{U}'^T U_S M U_S^T \tilde{U}'\\
        = {}& X^{*} \left( X W W^T X^{*} \right)^{\dag} X J \left( J^T J \right)^{\dag} J^T X^{*} \left( I - U_1 U_1^T \right) X V \Lambda^{-1} U_S^T \tilde{U}'\\
        = {}& X^{*} \left( X W W^T X^{*} \right)^{\dag} \left( X W W^T X^{*} \right) \left( I - U_1 U_1^T \right) X J.
    \end{align*}
    Since $(X W W^T X^{*})^{\dag} (X W W^T X^{*})$ is the projection matrix onto the linear subspace $\mathrm{Im}(XW) = \mathrm{Im}(X\tilde{V}) + \mathrm{Im}(XJ) = \mathrm{Im}(U_1) + \mathrm{Im}(XJ)$, and $(I - U_1 U_1^T) XJ = XJ - U_1 \cdot U_1^T XJ$ lies in this subspace, the last term above becomes $X^{*} \left( I - U_1 U_1^T \right) XJ$. Therefore, we get
    \begin{multline*}
        X^{*} \left( X W W^T X^{*} \right)^{\dag} XW W^T D_S^T \tilde{U}' K_3 = X^{*}\left( I - U_1 U_1^T \right) XJ\\
        \Longleftrightarrow X^{*} \left( X W W^T X^{*} \right)^{\dag} XW W^T D_S^T \tilde{U}' = X^{*} \left( I - U_1 U_1^T \right) X V \Lambda^{-1} U_S^T \tilde{U}' K_3^{-1}.
    \end{multline*}
    Now to prove \cref{eq:ic-0-ic0-remain}, it suffices to show
    \begin{multline*}
        X^{*} \left( X W W^T X^{*} \right)^{\dag} XW W^T D_S^T \left(I - \tilde{U}' \tilde{U}'^T \right) = 0 \Longleftarrow WW^T D_S^T U' U'^T = 0\\
        \Longleftarrow J\left( J^T J \right)^{\dag} J^T D_S^T U' = 0 \Longleftarrow J^T D_S^T U' = 0 \Longleftarrow \tilde{U}'^T U_S \Lambda^{-1} V^T \cdot V \Lambda U_S^T U' = 0\\
        \Longleftarrow \tilde{U}'^T U_S U_S^T U' = 0 \Longleftarrow \tilde{U}'^T \left( I - U' \Lambda' U'^T \right) U' = 0,
    \end{multline*}
    which is surely true since $\tilde{U}'^T U' = 0$. Then $\mathrm{irr}(0) = \mathrm{ic}_0$ is proved.

    Now we turn to $\mathrm{irr}(\infty)$. Let $M = U'' \Lambda'' U''^T$ be the compact eigendecomposition of $M$, and $(U'', \tilde{U}'')$ is an orthogonal square matrix. Then
    \begin{multline*}
        \nu \Sigma = I - U \left( I + \nu M \right)^{-1} U^T\\
        = I - U \left( U'', \tilde{U}'' \right) \left( \begin{pmatrix} U''^T\\ \tilde{U}''^T \end{pmatrix} \left( I + \nu M \right) \left( U'', \tilde{U}'' \right) \right)^{-1} \begin{pmatrix} U''^T \\ U''^T \end{pmatrix} U^T\\
        = I - U \left( U'', \tilde{U}'' \right) \begin{pmatrix} I + \nu \Lambda'' & 0 \\ 0 & I \end{pmatrix}^{-1} \begin{pmatrix} U''^T \\ \tilde{U}''^T \end{pmatrix} U^T\\
        = I - U U'' \left( I + \nu \Lambda'' \right)^{-1} U''^T U^T - U \tilde{U}'' \tilde{U}''^T U^T \rightarrow I - U \tilde{U}'' \tilde{U}''^T U^T
    \end{multline*}
    when $\nu \rightarrow +\infty$. Besides, $\nu \Sigma_{S,S} \rightarrow I - U_S \tilde{U}'' \tilde{U}''^T U_S^T$, and this limit $\succeq \nu \Sigma_{S,S} \succ 0$ for any $\nu > 0$. Thus $\Sigma_{S^c, S} \Sigma_{S,S}^{-1}$ has limit when $\nu \rightarrow +\infty$.

    Now we study when $\mathrm{irr}(\infty) = 0$. Let
    \begin{equation*}
        D_S^T = X^T C_1 + D_{S^c}^T C_2,\ \text{which implies}\ U_S^T = \Lambda^{-1} V^T X^T C_1 + U_{S^c}^T C_2.
    \end{equation*}
    Then $0 = \tilde{V}^T D_S^T = \tilde{V}^T X^T C_1 + 0 = \sqrt{n} V_1 \Lambda_1 U_1^T C_1$, which implies $U_1^T C_1 = 0$. So for $N = \Lambda^{-1} V^T X^T (I - U_1 U_1^T) / \sqrt{n}$, we have
    \begin{equation*}
        N C_1 = \Lambda^{-1} V^T X^T C_1 / \sqrt{n}.
    \end{equation*}
    Then $\mathrm{irr}(\infty) = 0 \Longleftrightarrow - U_{S^c} \tilde{U}'' \tilde{U}''^T U_S^T = 0 \Longleftrightarrow - U_{S^c} (I - M M^{\dag}) U_S^T = 0$. By $M = N N^T$, the equation is further equivalent to
    \begin{multline*}
        - U_{S^c} \left( I - N N^{\dag} \right) U_S^T = 0 \Longleftrightarrow - U_{S^c} \left( I - N N^{\dag} \right) \left( \Lambda^{-1} V^T X C_1 + U_{S^c}^T C_2 \right) = 0\\
        \Longleftrightarrow - U_{S^c} \left( I - N N^{\dag} \right) \left( \sqrt{n} N C_1 + U_{S^c}^T C_2 \right) = 0\\
        \Longleftrightarrow - U_{S^c} \left( I - N N^{\dag} \right) U_{S^c}^T C_2 = 0 \Longleftrightarrow C_2^T U_{S^c} \left( I - N N^{\dag} \right)\cdot \left( I - N N^{\dag} \right) U_{S^c}^T C_2 = 0\\
        \Longleftrightarrow \left( I - N N^{\dag} \right) U_{S^c}^T C_2 = 0 \Longleftrightarrow \mathrm{Im}(U_{S^c}^T C_2) \subseteq \mathrm{Im}(N).
    \end{multline*}
    It suffices to show that the last property holds if and only if $\ker(X) \subseteq \ker(D_S)$ or, equivalently, $\mathrm{Im}(D_S^T) \subseteq \mathrm{Im}(X^T)$. In fact, if $\mathrm{Im}(D_S^T) \subseteq \mathrm{Im}(X^T)$, then $C_2$ can be set $0$ in the beginning, and $\mathrm{Im}(U_{S^c}^T C_2) = \mathrm{Im}(0) \subseteq \mathrm{Im}(N)$. If $\mathrm{Im}(U_{S^c}^T C_2) \subseteq \mathrm{Im}(N)$, let $U_{S^c}^T C_2 = N C_3$, then
    \begin{multline*}
        D_{S^c}^T C_2 = V \Lambda U_{S^c}^T C_2 = V V^T X^T \left( I - U_1 U_1^T \right) C_3 / \sqrt{n}\\
        = \left( V V^T + \tilde{V} \tilde{V}^T \right) X^T \left( I - U_1 U_1^T \right) C_3 / \sqrt{n} = X^T \left( I - U_1 U_1^T \right) C_3 / \sqrt{n},
    \end{multline*}
    and hence $D_S^T = X^T C_1 + D_{S^c}^T C_2 = X^T (C_1 + (I - U_1 U_1^T) C_3 / \sqrt{n})$, which implies $\mathrm{Im}(D_S^T) \subseteq \mathrm{Im}(X^T)$. We have finished the proof of that $\mathrm{irr}(\infty) = 0$ if and only if $\ker(X) \subseteq \ker(D_S)$. \qed
\end{proof}

\section{Proof on Oracle Properties}
\label{sec:proof-orc}

\begin{proof}[Proof of \Cref{thm:slbiss-orc-gbi}]
    From the definition of oracle estimators \cref{eq:orc-def}, 
    \begin{equation}
        \label{eq:orc-1st-opt}
        \begin{split}
            \nabla_{\beta} \ell\left( \beta^o, \gamma^o \right) &= X^{*}\left( X\beta^o - y \right) + D^T \left( D\beta^o - \gamma^o \right) / \nu = 0,\\
            \nabla_{\gamma_S} \ell\left( \beta^o, \gamma^o \right) &= \left( \gamma_S^o - D_S \beta^o \right) / \nu = 0.
        \end{split}
    \end{equation}
    Adding \cref{eq:orc-1st-opt} to \cref{eq:slbiss-orc-b,eq:slbiss-orc-c}, we have
    \begin{align}
        \label{eq:slbiss-orc-b-reform}
        \begin{pmatrix} 0\\ \dot{\rho}_S'(t) \end{pmatrix} + \frac{1}{\kappa} \begin{pmatrix} \dot{\beta}'(t)\\ \dot{\gamma}_S'(t) \end{pmatrix} & = - H_{(\beta,S), (\beta,S)} \begin{pmatrix} d_\beta(t)\\ d_{\gamma,S}(t) \end{pmatrix}.
    \end{align}
    Besides, since
    \begin{equation*}
        \begin{pmatrix} \beta'(t)\\ \gamma'(t) \end{pmatrix},\ \begin{pmatrix} \beta^o\\ \gamma^o \end{pmatrix}\in L \oplus \mathbb{R}^s \oplus \{0\}^{m-s},
    \end{equation*}
    by \cref{eq:orc-def} and Pythagorean Theorem,
    \begin{multline}
        \label{eq:l-L-plus-c}
        \ell\left( \beta'(t), \gamma'(t) \right) = \frac{1}{2n} \left\| \begin{pmatrix} y\\ 0 \end{pmatrix} - \begin{pmatrix} X & 0\\ - \sqrt{n/\nu} D & I_m \end{pmatrix} \begin{pmatrix} \beta'(t)\\ \gamma'(t) \end{pmatrix} \right\|_2^2\\
        = \frac{1}{2n} \left\| \begin{pmatrix} X & 0\\ - \sqrt{n/\nu} D & I_m \end{pmatrix} \begin{pmatrix} \beta'(t)\\ \gamma'(t) \end{pmatrix} - \begin{pmatrix} X & 0\\ - \sqrt{n/\nu} D & I_m \end{pmatrix} \begin{pmatrix} \beta^o \\ \gamma^o \end{pmatrix} \right\|_2^2\\
        + \frac{1}{2n} \left\| \begin{pmatrix} y\\ 0 \end{pmatrix} - \begin{pmatrix} X & 0\\ - \sqrt{n/\nu} D & I_m \end{pmatrix} \begin{pmatrix} \beta^o\\ \gamma^o \end{pmatrix} \right\|_2^2\\
        = L(t) + \text{constant (independent of $t$)},
    \end{multline}
    where
    \begin{multline}
        \label{eq:L-def}
        L(t) := \frac{1}{2n} \left\| \begin{pmatrix} X & 0\\ - \sqrt{n/\nu} D & I_m \end{pmatrix} \begin{pmatrix} d_\beta(t)\\ d_\gamma(t) \end{pmatrix} \right\|_2^2 = \frac{1}{2} \left( d_\beta(t)^T, d_\gamma(t)^T \right) H \begin{pmatrix} d_\beta(t)\\ d_\gamma(t) \end{pmatrix}\\
        = \frac{1}{2} \left( d_\beta(t)^T, d_{\gamma,S}(t)^T \right) H_{(\beta,S),(\beta,S)} \begin{pmatrix} d_\beta(t)\\ d_{\gamma,S}(t) \end{pmatrix}.
    \end{multline}
    Noting $\gamma_j(t)\cdot \dot{\rho}_j(t) \equiv 0$ for each $j$, by \cref{eq:slbiss-orc-d,eq:slbiss-orc-b-reform,eq:L-def} we have
    \begin{equation}
        \label{eq:dPsi}
        \begin{split}
            \frac{\mathrm{d}}{\mathrm{d}t} \Psi(t) &= \left\langle - \gamma_S^o, \dot{\rho}_S'(t) \right\rangle + d_{\gamma,S}(t)^T \dot{\gamma}_S(t) / \kappa + d_\beta(t)^T \dot{\beta}'(t) / \kappa\\
            &= \left\langle \begin{pmatrix} d_\beta(t) \\ d_{\gamma,S}(t) \end{pmatrix},\ \begin{pmatrix} 0\\ \dot{\rho}_S'(t) \end{pmatrix} + \frac{1}{\kappa} \begin{pmatrix} \dot{\beta}'(t) \\ \dot{\gamma}_S'(t) \end{pmatrix} \right\rangle = - 2L(t).
        \end{split}
    \end{equation}
    Thus it suffices to show
    \begin{equation*}
        F\left( \frac{2}{\lambda_H} L(t) \right) \ge \Psi(t).
    \end{equation*}
    Since $\| \gamma_S^o \|_1 - \langle \gamma_S^o, \rho_S'(t) \rangle = 0$ if $\| \gamma_S'(t) - \gamma_S^o \|_2^2 < (\gamma_{\min}^o)^2$, and
    \begin{align*}
        \left\| \gamma_S^o \right\|_1 - \left\langle \gamma_S^o, \rho_S'(t) \right\rangle &\le 2 \sum_{j\in N(t)} \left| \gamma_j^o \right| \ \left( N(t) := \left\{ j:\ \mathrm{sign}\left( \gamma_j'(t) \right) \neq \mathrm{sign} \left( \gamma_j^o \right) \right\} \right)\\
        & \le
        \begin{cases}
            \displaystyle \frac{2}{\gamma_{\min}^o} \sum_{j\in N(t)} (\gamma_j^o)^2 \le \frac{2}{\gamma_{\min}^o} \left\| \gamma_S'(t) - \gamma_S^o \right\|_2^2 \\
            \displaystyle 2 \sqrt{s \sum_{j\in N(t)} (\gamma_j^o)^2}\le 2\sqrt{s \left\| \gamma_S'(t) - \gamma_S^o \right\|_2^2}.
        \end{cases}
    \end{align*}
    Thus
    \begin{equation*}
        \Psi(t) - \frac{1}{2\kappa} \left( \left\| d_{\gamma,S}(t) \right\|_2^2 + \left\| d_\beta(t) \right\|_2^2 \right) \le F\left( \left\| d_{\gamma,S}(t) \right\|_2^2 \right) - \frac{1}{2\kappa} \left\| d_{\gamma,S}(t) \right\|_2^2.
    \end{equation*}
    It suffice to show
    \begin{equation*}
        F\left( \frac{2}{\lambda_H} L(t) \right) \ge F\left( \left\| d_{\gamma,S}(t) \right\|_2^2 \right) + \frac{1}{2\kappa} \left\| d_\beta(t) \right\|_2^2,
    \end{equation*}
    which is true since by \Cref{thm:rsc}
    \begin{equation}
        \label{eq:L-ge-d}
        2 L(t) = \left( d_\beta(t)^T, d_{\gamma, S}(t)^T \right) \cdot H_{(\beta, S), (\beta, S)} \cdot \begin{pmatrix} d_\beta(t)\\ d_{\gamma, S}(t) \end{pmatrix} \ge \lambda_H \cdot d(t)^2,
    \end{equation}
    and by $F(\cdot + x) \ge F(\cdot) + x / (2\kappa)$
    \begin{equation*}
        F\left( d(t)^2 \right) = F\left( \left\| d_\beta(t) \right\|_2^2 + \left\| d_{\gamma, S}(t) \right\|_2^2 \right) \ge F\left( \left\| d_{\gamma, S}(t) \right\|_2^2 \right) + \frac{1}{2\kappa} \left\| d_\beta(t) \right\|_2^2. \eqno \mbox{\qed}
    \end{equation*}
\end{proof}

\begin{lemma}
    \label{thm:slbiss-orc-cstc}
    Under \Cref{thm:rsc}, let $\gamma_{\min}^o := \min(|\gamma_j^o|:\ \gamma_j^o\neq 0)$. For
    \begin{equation}
        \label{eq:tau-inf-def-slbiss}
        t \ge \tau_{\infty}(\mu) := \frac{1}{\kappa \lambda_H} \log \frac{1}{\mu} + \frac{2\log s + 4 + d(0) / \kappa}{\lambda_H \gamma_{\min}^o}\ (0<\mu<1),
    \end{equation}
    we have
    \begin{equation}
        \label{eq:slbiss-orc-cstc-sign}
        d(t) \le \mu \gamma_{\min}^o \left( \Longrightarrow \mathrm{sign}\left( \gamma_S'(t) \right) = \mathrm{sign}\left( \gamma_S^o \right),\ \text{if $\gamma_j^o \neq 0$ for $j\in S$} \right).
    \end{equation}
    For $t\ge 0$, we have
    \begin{equation}
        \label{eq:slbiss-orc-cstc-l2}
        d(t) \le \min \left( \frac{4 \sqrt{s} + d(0) / \kappa}{\lambda_H t},\ \sqrt{\frac{2\left( 1 + \nu \Lambda_X^2 + \Lambda_D^2 \right)}{\lambda_H \nu}} \cdot d(0) \right).
    \end{equation}
\end{lemma}

\begin{proof}[Proof of \Cref{thm:slbiss-orc-cstc}]
    Noting \cref{eq:l-L-plus-c} and that $\ell(\beta'(t), \gamma'(t))$ is \emph{non-increasing}, we know $L(t)$ is \emph{non-increasing}. \cref{eq:dPsi} tells that $\Psi(t)$ is non-increasing since $L(t) \ge 0$. If $L(t) = 0$ for $t = \tau_{\infty}(\mu)$, by \cref{eq:L-ge-d} and the fact that $L(t)$ is non-increasing, we have
    \begin{equation*}
        d(t)^2 \le \frac{2}{\lambda_H} L(t)^2 = 0\ \left( t\ge \tau_{\infty}(\mu) \right).
    \end{equation*}
    Therefore \cref{eq:slbiss-orc-cstc-sign} holds for $t\ge \tau_{\infty}(\mu)$. Now assume that $L(t) > 0$ for $t = \tau_{\infty}(\mu)$ (and hence for $0\le t\le \tau_{\infty}(\mu)$), then $\Psi(t)$ is \emph{strictly} decreasing on $[0, \tau_{\infty}(\mu)]$. Besides, $F$ is strictly increasing and continuous on $[(\gamma_{\min}^o)^2, +\infty)$. Moreover,
    \begin{align*}
        F\left( d(0)^2 \right) &\ge F\left( \left\| \gamma_S^o \right\|_2^2 \right) + \left\| \beta^o \right\|_2^2 / (2\kappa) \ge \Psi(0),\\
        d(0)^2 &\ge \left\| \gamma_S^o \right\|_2^2 \ge s\left( \gamma_{\min}^o \right)^2,
    \end{align*}
    If there does not exist some $t\le \tau_{\infty}(\mu)$ satisfying \cref{eq:slbiss-orc-cstc-sign}, then for $0\le t \le \tau_{\infty}(\mu)$,
    \begin{equation*}
        \Psi\left( t \right)
        \begin{cases}
            \ge d\left( t \right)^2 / (2\kappa) \ge \mu^2 \left( \gamma_{\min}^o \right)^2 / (2\kappa) >0,& \text{if $\kappa < +\infty$},\\
            > 0,& \text{if $\kappa = +\infty$},
        \end{cases}
    \end{equation*}
    which also implies that $F^{-1}(\Psi(t)) > 0$. By \Cref{thm:slbiss-orc-gbi},
    \begin{align*}
        & \lambda_H \tau_{\infty}(\mu) \le \int_0^{\tau_{\infty}(\mu)} \frac{- \frac{\mathrm{d}}{\mathrm{d}t} \Psi(t)}{F^{-1}\left( \Psi(t) \right)} \mathrm{d}t = \int_{\Psi\left( \tau_{\infty}(\mu) \right)}^{\Psi(0)} \frac{\mathrm{d}x}{F^{-1}(x)}\\
        \le {}& \left( \int_{\mu^2 \left( \gamma_{\min}^o \right)^2 / (2\kappa)}^{\left( \gamma_{\min}^o \right)^2 / (2\kappa)} + \int_{\left( \gamma_{\min}^o \right)^2 / (2\kappa)}^{F\left( \left( \gamma_{\min}^o \right)^2 \right)} + \int_{F\left( \left( \gamma_{\min}^o \right)^2 \right)}^{F\left( s\left( \gamma_{\min}^o \right)^2 \right)} + \int_{F\left( s\left( \gamma_{\min}^o \right)^2 \right)}^{F\left( d(0)^2 \right)} \right) \frac{\mathrm{d}x}{F^{-1}(x)}\\
        \le {}& \int_{\mu^2 \left( \gamma_{\min}^o \right)^2 / (2\kappa)}^{\left( \gamma_{\min}^o \right)^2 / (2\kappa)} \frac{\mathrm{d}x}{2\kappa x} + \int_{\left( \gamma_{\min}^o \right)^2 / (2\kappa)}^{F\left( \left( \gamma_{\min}^o \right)^2 \right)} \frac{1}{\left( \gamma_{\min}^o \right)^2} \mathrm{d}x + \int_{\left( \gamma_{\min}^o \right)^2}^{s\left( \gamma_{\min}^o \right)^2} \frac{\mathrm{d}F(x)}{x} + \int_{s\left( \gamma_{\min}^o \right)^2}^{d(0)^2} \frac{\mathrm{d}F(x)}{x}\\
        = {}& \frac{1}{2\kappa} \log \frac{1}{\mu^2} + \frac{2}{\gamma_{\min}^o} + \int_{\left( \gamma_{\min}^o \right)^2}^{s\left( \gamma_{\min}^o \right)^2} \left( \frac{1}{2\kappa x} + \frac{2}{\gamma_{\min}^o x} \right) \mathrm{d}x + \int_{s\left( \gamma_{\min}^o \right)^2}^{d(0)^2} \left( \frac{1}{2\kappa x} + \frac{\sqrt{s}}{x\sqrt{x}} \right) \mathrm{d}x\\
        < {}& \frac{1}{2\kappa} \log \frac{1}{\mu^2} + \frac{2}{\gamma_{\min}^o} + \frac{1}{2\kappa} \log \frac{d(0)^2}{\left( \gamma_{\min}^o \right)^2} + \frac{2\log s}{\gamma_{\min}^o} + \frac{2}{\gamma_{\min}^o}\\
        \le{}& \frac{1}{\kappa} \log \frac{1}{\mu} + \frac{2\log s + 4 + d(0) / \kappa}{\gamma_{\min}^o},
    \end{align*}
    contradicting with the definition of $\tau_{\infty}(\mu)$. Thus \cref{eq:slbiss-orc-cstc-sign} holds for some $0\le \tau\le \tau_{\infty}(\mu)$. If $\kappa = +\infty$, we see that for $t\ge \tau_{\infty}(\mu)$, $\Psi(t) \le \Psi(\tau) = 0$. Then $-2L(t)$, the derivative of $\Psi(t)$, is $0$ (which means $d(t) = 0$) when $t\ge \tau_{\infty}(\mu)$, and \cref{eq:slbiss-orc-cstc-sign} holds. If $\kappa < +\infty$, just note that for $t\ge \tau$,
    \begin{equation*}
        d(t)^2/(2\kappa) \le \Psi(t)\le \Psi(\tau) = d(\tau)^2/(2\kappa)\Longrightarrow d(t)\le d(\tau) \le \mu \gamma_{\min}^o.
    \end{equation*}
    So \cref{eq:slbiss-orc-cstc-sign} holds for $t\ge \tau_{\infty}(\mu)$.

    For any $t>0$, if $L(t) = 0$, then $d(t) = 0$ and \cref{eq:slbiss-orc-cstc-l2} holds. If $L(t) > 0$, let $C = \sqrt{2L(t) / \lambda_H} > 0$, then for any $0\le t'\le t$,
    \begin{equation*}
        \frac{\mathrm{d}}{\mathrm{d}t'}\Psi\left( t' \right) = - 2L\left( t' \right) \le - 2L(t) = - \lambda_H C^2.
    \end{equation*}
    Besides, for $\tilde{F}(x) = x / (2\kappa) + 2\sqrt{sx}\ge F(x)$, by \Cref{thm:slbiss-orc-gbi} we have
    \begin{equation*}
        \frac{\mathrm{d}}{\mathrm{d}t'}\Psi\left( t' \right) \le - \lambda_H F^{-1}\left( \Psi\left( t' \right) \right)\le - \lambda_H \tilde{F}^{-1} \left( \Psi\left( t' \right) \right).
    \end{equation*}
    By \cref{eq:dPsi} and the fact that
    \begin{align*}
        \tilde{F}\left( d(0)^2 \right) &\ge \tilde{F}\left( \left\| \gamma_S^o \right\|_2^2 \right) + \left\| \beta^o \right\|_2^2 / (2\kappa) \ge \Psi(0),
    \end{align*}
    we have that, if $d(0) > C$, then
    \begin{align*}
        \lambda_H t &\le \int_0^t \frac{- \frac{\mathrm{d}}{\mathrm{d}t'} \Psi\left( t' \right)}{\max\left( C^2, \tilde{F}^{-1}\left( \Psi\left( t' \right) \right) \right)} \mathrm{d}t' = \int_{\Psi(t)}^{\Psi(0)} \frac{\mathrm{d}x}{\max\left( C^2, \tilde{F}^{-1}(x) \right)}\\
        &\le \int_{\tilde{F}(0)}^{\tilde{F}\left( d(0)^2 \right)} \frac{\mathrm{d}x}{\max\left( C^2, \tilde{F}^{-1}(x) \right)} = \int_{\tilde{F}(0)}^{\tilde{F}\left( C^2 \right)} \frac{\mathrm{d}x}{C^2} + \int_{C^2}^{d(0)^2} \frac{\mathrm{d}\tilde{F}(x)}{x}\\
        &= \frac{C^2/(2\kappa)+2\sqrt{s}C}{C^2} + \int_{C^2}^{d(0)^2} \left( \frac{1}{2\kappa x} + \frac{\sqrt{s}}{x\sqrt{x}} \right) \mathrm{d}x\\
        &\le \frac{4\sqrt{s}}{C} + \frac{1}{2\kappa} \left( 1 + \log \frac{d(0)^2}{C^2} \right)\le \frac{4\sqrt{s} + d(0) / \kappa}{C}.
    \end{align*}
    If $d(0) \le C$, then similarly
    \begin{multline*}
        \lambda_H t \le \int_{\tilde{F}(0)}^{\tilde{F}\left( d(0)^2 \right)} \frac{\mathrm{d}x}{\max\left( C^2, \tilde{F}^{-1}(x) \right)} \le \int_{\tilde{F}(0)}^{\tilde{F}\left( d(0)^2 \right)} \frac{\mathrm{d}x}{C^2}\\
        = \frac{d(0)^2 / (2\kappa) + 2\sqrt{s}\cdot d(0)}{C^2} \le \frac{4\sqrt{s} + d(0) / \kappa}{C}.
    \end{multline*}
    Combining it with \cref{eq:L-ge-d}, we have
    \begin{equation*}
        d(t)^2 \le \frac{2}{\lambda_H} L(t) = \frac{2}{\lambda_H}\cdot \frac{\lambda_H C^2}{2} \le \left( \frac{4\sqrt{s} + d(0) / \kappa}{\lambda_H t} \right)^2.
    \end{equation*}
    Besides, noting \cref{eq:H-upper}, we have
    \begin{multline*}
        2 L(0) = \left( d_\beta(0)^T, d_{\gamma,S}(0)^T \right) H_{(\beta,S), (\beta,S)} \begin{pmatrix} d_\beta(0)\\ d_{\gamma,S}(0) \end{pmatrix} \\
        \le \left\| H \right\|_2 \cdot \left\| \begin{pmatrix} d_\beta(0) \\ d_\gamma(0) \end{pmatrix} \right\|_2^2 \le \frac{2 \left( 1 + \nu \Lambda_X^2 + \Lambda_D^2 \right)}{\nu} \cdot d(0)^2.
    \end{multline*}
    Thus
    \begin{equation*}
        d(t)^2 \le \frac{2}{\lambda_H} L(t) \le \frac{2}{\lambda_H} L(0) \le \frac{2\left( 1 + \nu \Lambda_X^2 + \Lambda_D^2 \right)}{\lambda_H \nu} \cdot d(0)^2.
    \end{equation*}
    Thus \cref{eq:slbiss-orc-cstc-l2} holds. \qed
\end{proof}

\section{Proof on Consistency of Split LBISS}
\label{sec:proof-slbiss-cstc}

Before proving \Cref{thm:slbiss-cstc} and \ref{thm:slbiss-rev-cstc}, we need the following lemmas.

\begin{lemma}[\textnormal{No-false-positive condition for Split LBISS}]
    \label{thm:slbiss-nfp}
    For the oracle dynamics \cref{eq:slbiss-orc}, if there is $\tau > 0$, such that for $0\le t\le \tau$ the inequality
    \begin{equation}
        \label{eq:slbiss-nfp}
        \left\| H_{S^c, (\beta, S)} {H_{(\beta,S), (\beta,S)}}^{\dag} \left( \begin{pmatrix} 0_p \\ \rho_S'(t) \end{pmatrix} + \frac{1}{\kappa} \begin{pmatrix} \beta'(t) \\ \gamma_S'(t) \end{pmatrix} - t \begin{pmatrix} X^{*} \epsilon\\ 0_s \end{pmatrix} \right) \right\|_{\infty} < 1
    \end{equation}
    holds, then the solution path of the original dynamics \cref{eq:slbiss} has no false-positive for $0\le t\le \tau$.
\end{lemma}

\begin{proof}[Proof of \Cref{thm:slbiss-nfp}]
    It is easy to see that
    \begin{equation}
        \label{eq:slbiss-comb}
        \begin{pmatrix} 0_p\\ \dot{\rho}(t) \end{pmatrix} + \frac{1}{\kappa} \begin{pmatrix} \dot{\beta}(t)\\ \dot{\gamma}(t) \end{pmatrix} = H \left( \begin{pmatrix} \beta(t)\\ \gamma(t) \end{pmatrix} - \begin{pmatrix} \beta^{\star}\\ \gamma^{\star} \end{pmatrix} \right) + \begin{pmatrix} X^{*} \epsilon \\ 0_m \end{pmatrix}.
    \end{equation}
    Now define the exit time of oracle subspace, 
    \begin{align*}
        \tau_{\mathrm{exit}} &:= \inf\left( t\ge 0:\ \left\| \rho_{S^c}(t) \right\|_{\infty} = 1 \right).
    \end{align*}
    It suffices to show $\tau_{\mathrm{exit}} > \tau$. For $0\le t < \tau_{\mathrm{exit}}$, we have $\gamma_{S^c}(t) = 0$, which also implies the paths of Split LBISS and oracle dynamics are identical, i.e. $\rho_S(t) = \rho_S'(t)$ and $\gamma_S(t) = \gamma_S'(t)$. Hence by \cref{eq:slbiss-comb} we have
    \begin{align}
        \label{eq:rho-S}
        \begin{pmatrix} 0_p \\ \dot{\rho}_S'(t) \end{pmatrix} + \frac{1}{\kappa} \begin{pmatrix} \dot{\beta}'(t)\\ \dot{\gamma}_S'(t) \end{pmatrix} &= - H_{(\beta, S), (\beta, S)} \left( \begin{pmatrix} \beta'(t)\\ \gamma_S'(t) \end{pmatrix} - \begin{pmatrix} \beta^{\star} \\ \gamma_S^{\star} \end{pmatrix} \right) + \begin{pmatrix} X^{*} \epsilon\\ 0_s \end{pmatrix},\\
        \dot{\rho}_{S^c}(t) &= - H_{S^c,(\beta,S)} \left( \begin{pmatrix} \beta'(t)\\ \gamma_S'(t) \end{pmatrix} - \begin{pmatrix} \beta^{\star} \\ \gamma_S^{\star} \end{pmatrix} \right).\notag
    \end{align}
    We claim that
    \begin{equation*}
        \begin{pmatrix} \beta'(t)\\ \gamma_S'(t) \end{pmatrix} - \begin{pmatrix} \beta^{\star} \\ \gamma_S^{\star} \end{pmatrix} \in L\oplus \mathbb{R}^s = \mathrm{Im} \left( {H_{(\beta, S), (\beta, S)}}^{\dag} \right)
    \end{equation*}
    (the equality above will be shown at last), so by \cref{eq:rho-S} we have
    \begin{align*}
        \begin{pmatrix} \beta'(t)\\ \gamma_S'(t) \end{pmatrix} - \begin{pmatrix} \beta^{\star} \\ \gamma_S^{\star} \end{pmatrix} = - {H_{(\beta,S), (\beta,S)}}^{\dag} \left( \begin{pmatrix} 0_p \\ \dot{\rho}_S'(t) \end{pmatrix} + \frac{1}{\kappa} \begin{pmatrix} \dot{\beta}'(t) \\ \dot{\gamma}_S'(t) \end{pmatrix} - \begin{pmatrix} X^{*} \epsilon\\ 0_s \end{pmatrix} \right),\\
        \Longrightarrow \dot{\rho}_{S^c}(t) = H_{S^c, (\beta, S)} {H_{(\beta,S), (\beta,S)}}^{\dag} \left( \begin{pmatrix} 0_p \\ \dot{\rho}_S'(t) \end{pmatrix} + \frac{1}{\kappa} \begin{pmatrix} \dot{\beta}'(t) \\ \dot{\gamma}_S'(t) \end{pmatrix} - \begin{pmatrix} X^{*} \epsilon\\ 0_s \end{pmatrix} \right).
    \end{align*}
    Integration on both sides leads to, for $0 \le t< \tau_{\mathrm{exit}}$
    \begin{equation*}
        \rho_{S^c}(t) = H_{S^c, (\beta, S)} {H_{(\beta,S), (\beta,S)}}^{\dag} \left( \begin{pmatrix} 0_p \\ \rho_S'(t) \end{pmatrix} + \frac{1}{\kappa} \begin{pmatrix} \beta'(t) \\ \gamma_S'(t) \end{pmatrix} - t \begin{pmatrix} X^{*} \epsilon\\ 0_s \end{pmatrix} \right)\ .
    \end{equation*}
    Due to the continuity of $\rho_{S^c}(t), \rho_S'(t)$ (and $\gamma_S'(t)$, if $\kappa < +\infty$), the equation above also holds for $t = \tau_{\mathrm{exit}}$. According to the definition of $\tau_{\mathrm{exit}}$, we know \cref{eq:slbiss-nfp} does not hold for $t = \tau_{\mathrm{exit}}$. Thus for $\tau<\tau_{\mathrm{exit}}$, the desired result follows.

    So it suffices to prove
    \begin{equation}
        \label{eq:L-oplus-Rs}
        L\oplus \mathbb{R}^s = \mathrm{Im} \left( {H_{(\beta, S),(\beta, S)}}^{\dag} \right).
    \end{equation}
    Actually, let $H_{(\beta, S), (\beta, S)} = U' \Lambda' U'^T$ where $U'^T U' = I$ and $\Lambda'$ is an invertible diagonal matrix. It suffices to show $L\oplus \mathbb{R}^s = \mathrm{Im}(U')$. First, by the definition of $H$, one can easily verify that
    \begin{equation*}
        \mathrm{Im}\left( U' \right) = \mathrm{Im}\left( H_{(\beta, S), (\beta, S)} \right) \subseteq \left( \mathrm{Im}\left( X^T \right) + \mathrm{Im} \left( D^T \right) \right) \oplus \mathbb{R}^s = L \oplus \mathbb{R}^s.
    \end{equation*}
    On the other hand, assume that $(U', \tilde{U}')$ is an orthogonal square matrix. For any $\zeta \in L\oplus \mathbb{R}^s$, since $P_{\mathrm{Im}(U')} \zeta \in \mathrm{Im}(U') \subseteq L \oplus \mathbb{R}^s$, we have $P_{\mathrm{Im}(\tilde{U}')} \zeta = \zeta - P_{\mathrm{Im}(U')} \zeta \in L\oplus \mathbb{R}^s$, and \cref{eq:rsc} tells us
    \begin{multline*}
        0 = \left\| \Lambda'^{1/2} U'^T P_{\mathrm{Im} \left( \tilde{U'} \right)} \zeta \right\|_2^2 \ge \lambda_H \left\| P_{\mathrm{Im}\left( \tilde{U}' \right)} \zeta \right\|_2^2 \Longrightarrow P_{\mathrm{Im}\left( \tilde{U}' \right)} \zeta = 0\\
        \Longrightarrow \zeta = P_{\mathrm{Im}(U')} \zeta + P_{\mathrm{Im}\left( \tilde{U}' \right)} \zeta = P_{\mathrm{Im}(U')} \zeta \in \mathrm{Im}(U').
    \end{multline*}
    Thus \cref{eq:L-oplus-Rs} holds. \qed
\end{proof}

\begin{lemma}
    \label{thm:orc-minus-star}
    Suppose $\Sigma_{S,S} \succeq \lambda_\Sigma I$. For $\beta^o \in L$ and $\gamma_S^o\in \mathbb{R}^s$ satisfying \cref{eq:orc-1st-opt}, we have
    \begin{equation}
        \label{eq:beta-orc-minus-star}
        \begin{split}
            \left\| \beta^o - \beta^{\star} \right\|_2^2 = &{} \left\| \delta^o - \delta^{\star} \right\|_2^2 + \left\| \xi^o - \xi^{\star} \right\|_2^2,\ \text{where}\\
            & \delta^o - \delta^{\star} := V^T \left( \beta^o - \beta^{\star} \right),\ \xi^o - \xi^{\star} = V_1^T \tilde{V}^T \left( \beta^o - \beta^{\star} \right),
        \end{split}
    \end{equation}
    and
    \begin{align}
        \label{eq:delta-orc-minus-star}
        \delta^o - \delta^{\star} &= \underbrace{\left( \nu B^{-1} + B^{-1} \Lambda U_S^T \Sigma_{S,S}^{-1} U_S \Lambda B^{-1} \right) V^T X^{*} \left( I - U_1 U_1^T \right)}_{\triangleq B_{\delta}} \epsilon,\ \text{with}\ \left\| B_\delta \right\|_2 \le \frac{\Lambda_X}{\sqrt{n}\cdot \lambda_\Sigma \lambda_D^2},\\
        \label{eq:xi-orc-minus-star}
        \xi^o - \xi^{\star} &= \underbrace{n^{-1/2} \Lambda_1^{-1} U_1^T (I - XV B_\delta)}_{\triangleq B_{\xi}} \epsilon,\ \text{with}\ \left\| B_\xi \right\|_2 \le \frac{\lambda_\Sigma \lambda_D^2 + \Lambda_X^2}{\sqrt{n}\cdot \lambda_1 \lambda_\Sigma \lambda_D^2}.
    \end{align}
    Besides, we have
    \begin{equation}
        \label{eq:gamma-orc-minus-star}
        \gamma_S^o - \gamma_S^{\star} = \underbrace{\Sigma_{S,S}^{-1} U_S \Lambda B^{-1} V^T X^{*} \left( I - U_1 U_1^T \right)}_{\triangleq B_\gamma} \epsilon,\ \text{with}\ \left\| B_\gamma \right\|_2 \le \frac{\Lambda_X}{\sqrt{n}\cdot \lambda_\Sigma \lambda_D}.
    \end{equation}
\end{lemma}

\begin{proof}
    By \Cref{thm:beta-decomposition} and $\beta^o - \beta^{\star}\in L$, we have \cref{eq:beta-orc-minus-star}. By \cref{eq:orc-1st-opt}, we have
    \begin{equation}
        \label{eq:gamma-delta-orc-minus-star}
        \gamma_S^o - \gamma_S^{\star} = D_S \left( \beta^o - \beta^{\star} \right) = U_S \Lambda \left( \delta^o - \delta^{\star} \right),
    \end{equation}
    and
    \begin{equation*}
        X^{*} \epsilon + D_S^T \left( \gamma_S^o - \gamma_S^{\star} \right) / \nu = \left( X^{*} X + D^T D / \nu \right) \left( \beta^o - \beta^{\star} \right),
    \end{equation*}
    i.e.
    \begin{multline}
        \label{eq:gamma-beta-orc-minus-star}
        X^{*} \epsilon + V \Lambda U_S^T \left( \gamma_S^o - \gamma_S^{\star} \right) / \nu = \left( X^{*} X + V \Lambda^2 V^T / \nu \right) \left( V \left( \delta^o - \delta^{\star} \right) + \tilde{V} V_1 \left( \xi^o - \xi^{\star} \right) \right)\\
        = \left( X^{*} X V + V \Lambda^2 / \nu \right) \left( \delta^o - \delta^{\star} \right) + \sqrt{n} X^{*} U_1 \Lambda_1 \left( \xi^o - \xi^{\star} \right).
    \end{multline}
    Left multiplying $\Lambda_1^{-2} V_1^T \tilde{V}^T$ on both sides of \cref{eq:gamma-beta-orc-minus-star} leads to
    \begin{equation}
        \label{eq:xi-delta-orc-minus-star}
        \xi^o - \xi^{\star} = \frac{1}{\sqrt{n}} \Lambda_1^{-1} U_1^T \left( \epsilon - XV \left( \delta^o - \delta^{\star} \right) \right).
    \end{equation}
    Then left multiplying $V^T$ on both sides of \cref{eq:gamma-beta-orc-minus-star} leads to
    \begin{multline*}
        V^T X^{*} \epsilon + \Lambda U_S^T \left( \gamma_S^o - \gamma_S^{\star} \right) / \nu\\
        = \left( V^T X^{*} X V + \Lambda^2 / \nu \right) \left( \delta^o - \delta^{\star} \right) + \sqrt{n} V^T X^{*} U_1 \Lambda_1 \cdot \frac{1}{\sqrt{n}} \Lambda_1^{-1} U_1^T \left( \epsilon - XV \left( \delta^o - \delta^{\star} \right) \right)\\
        = \left( V^T X^{*} \left( I - U_1 U_1^T \right) X V + \Lambda^2 / \nu \right) \left( \delta^o - \delta^{\star} \right) + V^T X^{*} U_1 U_1^T \epsilon.
    \end{multline*}
    Recalling the definition of $B$ in \Cref{thm:Adag}, the equation above implies
    \begin{equation}
        \label{eq:delta-gamma-orc-minus-star}
        \delta^o - \delta^{\star} = B^{-1} \Lambda U_S^T \left( \gamma_S^o - \gamma_S^{\star} \right) + \nu B^{-1} V^T X^{*} \left( I - U_1 U_1^T \right) \epsilon.
    \end{equation}
    Plugging it into \cref{eq:gamma-delta-orc-minus-star}, we obtain $\gamma_S^o - \gamma_S^{\star} = B_{\gamma} \epsilon$. Then noting $B \succeq \lambda_D^2 I$, we have
    \begin{equation*}
        \left\| B_\gamma \right\|_2 \le \left\| \Sigma_{S,S}^{-1} \right\|_2 \cdot 1 \cdot \left\| \Lambda B^{-1} \right\|_2 \cdot 1\cdot \left\| X^{*} \right\|_2 \cdot \left\| I - U_1 U_1^T \right\|_2 \le \frac{\Lambda_X}{\sqrt{n}\cdot \lambda_\Sigma \lambda_D}.
    \end{equation*}
    so \cref{eq:gamma-orc-minus-star} holds. Now by \cref{eq:delta-gamma-orc-minus-star} we have $\delta^o - \delta^{\star} = B_{\delta} \epsilon$. Noting \cref{eq:Sigma} and $\Sigma_{S,S}\succeq \lambda_\Sigma I$, we have
    \begin{multline*}
        U_S \Lambda B^{-1/2} \cdot B^{-1/2} \Lambda U_S^T \preceq (1 - \lambda_\Sigma \nu) I\\
        \Longleftrightarrow B^{-1/2} \Lambda U_S^T \cdot U_S \Lambda B^{-1/2} \preceq (1 - \lambda_\Sigma \nu) I \Longleftrightarrow \Lambda U_S^T U_S \Lambda \preceq (1 - \lambda_\Sigma \nu) B.
    \end{multline*}
    Thus
    \begin{equation*}
        \nu B^{-1} + B^{-1} \Lambda U_S^T \Sigma_{S,S}^{-1} U_S \Lambda B^{-1} \preceq \nu B^{-1} + \frac{1}{\lambda_\Sigma} B^{-1} \Lambda U_S^T U_S \Lambda B^{-1} \preceq \frac{1}{\lambda_\Sigma} B^{-1},
    \end{equation*}
    which immediately leads to \cref{eq:delta-orc-minus-star}. Finally, combining \cref{eq:xi-delta-orc-minus-star} with \cref{eq:delta-orc-minus-star} we have \cref{eq:xi-orc-minus-star}. \qed
\end{proof}

Now we are ready for proving the main theorems.

\begin{proof}[Proof of \Cref{thm:slbiss-cstc}]
    By \cref{eq:tail-subgau-2,eq:gamma-orc-minus-star,eq:xi-orc-minus-star}, we have that with probability not less than $1-4s/m^2\ge 1 - 4/m$,
    \begin{align}
        \label{eq:gamma-orc-minus-star-linf}
        \left\| \gamma_S^o - \gamma_S^\star \right\|_{\infty} &< \frac{2\sigma}{\lambda_H} \cdot \frac{\Lambda_X}{\lambda_D} \sqrt{\frac{\log m}{n}},\\
        \label{eq:xi-orc-minus-star-linf}
        \left\| \xi^o - \xi^{\star} \right\|_{\infty} &< \frac{2\sigma}{\lambda_H}\cdot \frac{\lambda_H \lambda_D^2 + \Lambda_X^2}{\lambda_1 \lambda_D^2} \sqrt{\frac{\log m}{n}}.
    \end{align}
    By \cref{eq:tail-subexp-2,eq:beta-orc-minus-star,eq:delta-orc-minus-star,eq:xi-orc-minus-star,eq:gamma-orc-minus-star}, with probability not less than $1 - 3 \exp(-4n/5)$,
    \begin{equation}
        \label{eq:gamma-delta-xi-orc-minus-star-l2}
        \begin{gathered}
            \left\| \epsilon \right\|_2 \le 2\sigma \sqrt{n},\ \text{which implies}\\
            \left\| \gamma_S^o - \gamma_S^{\star} \right\|_2 < \frac{2\sigma}{\lambda_H}\cdot \frac{\Lambda_X}{\lambda_D},\ \left\| \delta^o - \delta^{\star} \right\|_2 < \frac{2\sigma}{\lambda_H}\cdot \frac{\Lambda_X}{\lambda_D^2},\ \left\| \xi^o - \xi^{\star} \right\|_2 < \frac{2\sigma}{\lambda_H}\cdot \frac{\lambda_H \lambda_D^2 + \Lambda_X^2}{\lambda_1 \lambda_D^2}.
        \end{gathered}
    \end{equation}
    The inequalities above also imply
    \begin{equation}
        \label{eq:beta-orc-minus-star-l2}
        \left\| \beta^o - \beta^{\star} \right\|_2 \le \left\| \delta^o - \delta^{\star} \right\|_2 + \left\| \xi^o - \xi^{\star} \right\|_2 < \frac{2\sigma}{\lambda_H} \left( \frac{\Lambda_X}{\lambda_D^2} + \frac{\lambda_H \lambda_D^2 + \Lambda_X^2}{\lambda_1 \lambda_D^2} \right),
    \end{equation}
    and
    \begin{multline}
        \label{eq:d-0-upper}
        d(0) = \sqrt{\left\| \gamma_S^o \right\|_2^2 + \left\| \beta^o \right\|_2^2} \le \left\| \gamma_S^\star \right\|_2 + \left\| \beta^{\star} \right\|_2 + \left\| \gamma_S^o - \gamma_S^\star \right\|_2 + \left\| \beta^o - \beta^{\star} \right\|_2\\
        < \left( 1 + \Lambda_D \right) \left\| \beta^{\star} \right\|_2 + \frac{2\sigma}{\lambda_H} \left( \frac{\Lambda_X}{\lambda_D} + \frac{\Lambda_X}{\lambda_D^2} + \frac{\lambda_H \lambda_D^2 + \Lambda_X^2}{\lambda_1 \lambda_D^2} \right).
    \end{multline}
    From now, we assume all the inequalities above hold. The condition on $\kappa$ now tells us
    \begin{equation}
        \label{eq:kappa-ge-d0}
        \kappa \ge \frac{4}{\eta} \left( 1 + \frac{1}{\lambda_D} + \frac{\Lambda_X}{\lambda_1 \lambda_D} \right) \left( 1 + \sqrt{\frac{2 \left( 1 + \nu \Lambda_X^2 + \Lambda_D^2 \right)}{\lambda_H \nu}} \right) \cdot d(0)\ (\ge d(0)).
    \end{equation}

    Now we prove the \emph{No-false-positive} property. By \Cref{thm:slbiss-nfp}, it suffices to show that for $0\le t\le \bar{\tau}$, \cref{eq:slbiss-nfp} holds with probability not less that $1 - 2/m$. By \cref{eq:ic-equiv,eq:DAdag,eq:slbiss-orc-cstc-l2},
    \begin{multline*}
        \frac{1}{\kappa} \left\| H_{S^c, (\beta,S)} {H_{(\beta,S),(\beta,S)}}^{\dag} \begin{pmatrix} \beta'(t) \\ \gamma_S'(t) \end{pmatrix} \right\|_{\infty}\\
        = \left\| \left( - D_{S^c} A^{\dag} + \Sigma_{S^c, S} \Sigma_{S,S}^{-1} D_S \right) A^{\dag} \beta'(t) + \Sigma_{S^c, S} \Sigma_{S,S}^{-1} \gamma_S'(t) \right\|_{\infty} / \kappa\\
        \le \left\| D_{S^c} A^{\dag} \beta'(t) \right\|_{\infty} / \kappa + \left\| \Sigma_{S^c, S} \Sigma_{S,S}^{-1} D_S A^{\dag} \beta'(t) \right\|_{\infty} / \kappa + \left\| \gamma_S'(t) \right\|_{\infty} / \kappa\\
        \le 2 \left\| D A^{\dag} \right\|_2 \cdot \left\| \beta'(t) \right\|_2 / \kappa + \left\| \gamma_S'(t) \right\|_2 / \kappa \le \left( 2 \left( \frac{1}{\lambda_D} + \frac{\Lambda_X}{\lambda_D \lambda_1} \right) + 1 \right) \sqrt{\left\| \beta'(t) \right\|_2^2 + \left\| \gamma_S'(t) \right\|_2^2} / \kappa\\
        \le 2 \left( 1 + \frac{1}{\lambda_D} + \frac{\Lambda_X}{\lambda_D \lambda_1} \right) \left( d(0) + d(t) \right) / \kappa\\
        \le 2 \left( 1 + \frac{1}{\lambda_D} + \frac{\Lambda_X}{\lambda_D \lambda_1} \right) \left( 1 + \sqrt{\frac{2 \left( 1 + \nu \Lambda_X^2 + \Lambda_D^2 \right)}{\lambda_H \nu}} \right) d(0) / \kappa \le \frac{\eta}{2}.
    \end{multline*}
    Besides, by \cref{eq:ic-equiv} we have
    \begin{multline*}
        \left\| H_{S^c, (\beta,S)} {H_{(\beta,S), (\beta,S)}}^{\dag} \begin{pmatrix} X^{*} \epsilon \\ 0 \end{pmatrix} \right\|_{\infty} = \left\| \left( - D_{S^c} + \Sigma_{S^c, S} \Sigma_{S,S}^{\dag} D_S \right) A^{\dag} X^{*} \epsilon \right\|_{\infty}\\
        \le \left\| D_{S^c} A^{\dag} X^{*} \epsilon \right\|_{\infty} + \left\| D_S A^{\dag} X^{*} \epsilon \right\|_{\infty} \le 2 \left\| D A^{\dag} X^{*} \epsilon \right\|_{\infty}.
    \end{multline*}
    By \cref{eq:Sigma}, $D A^{\dag} D^T = U \Lambda B^{-1} \Lambda U^T$ and $\Lambda^2 \preceq B \preceq (1 + \nu \Lambda_X^2 / \lambda_D^2) \Lambda^2$, therefore $1$ is an upper bound of the largest eigenvalue of $D A^{\dag} D^T$, and $1 / (1 + \nu \Lambda_X^2 / \lambda_D^2)$ is a lower bound of the smallest \emph{nonzero} eigenvalue of $D A^{\dag} D^T$. Then
    \begin{multline*}
        D A^{\dag} X^{*} \left( D A^{\dag} X^{*} \right)^T = \frac{1}{n \nu} D A^{\dag} \left( A - D^T D \right) A^{\dag} D^T\\
        = \frac{1}{n \nu} \left( D A^{\dag} D^T - \left( D A^{\dag} D^T \right)^2 \right) \preceq \frac{1}{n \nu} \min \left( \frac{1}{4},\ \frac{\nu \Lambda_X^2 / \lambda_D^2}{\left( 1 + \nu \Lambda_X^2 / \lambda_D^2 \right)^2} \right) I \preceq \frac{\Lambda_X^2}{n\cdot \lambda_D^2} I.
    \end{multline*}
    By \cref{eq:tail-subgau-2}, with probability not less than $1 - 2/m$, for any $0\le t\le \bar{\tau}$,
    \begin{equation*}
        \left\| H_{S^c, (\beta,S)} {H_{(\beta,S), (\beta,S)}}^{\dag} \cdot t \begin{pmatrix} X^{*} \epsilon \\ 0 \end{pmatrix} \right\|_{\infty} \le 2 \bar{\tau} \left\| D A^{\dag} X^{*} \epsilon \right\|_{\infty} \le 2 \bar{\tau}\cdot 2\sigma \cdot \sqrt{\frac{\Lambda_X^2}{n\cdot \lambda_D^2}} \cdot \sqrt{\log m} < \frac{\eta}{2}.
    \end{equation*}
    Combining the results above with \Cref{thm:irr-nu}, we have for $0\le t\le \bar{\tau}$, \cref{eq:slbiss-nfp} holds with probability not less that $1 - 2/m$, and we have the No-false-positive property (which tells that $(\beta(t), \gamma_S(t))$ coincides with that of the oracle dynamics for $0 \le t\le \bar{\tau}$).

    Then we prove the \emph{sign consistency of $\gamma(t)$}. If the $\gamma_{\min}^{*}$ condition \cref{eq:slbiss-gammamin-cond} holds, by \cref{eq:gamma-orc-minus-star-linf},
    \begin{equation}
        \label{eq:gamma-orc-minus-star-cstc-sign}
        \left\| \gamma_S^o - \gamma_S^\star \right\|_{\infty} \le \frac{2\sigma}{\lambda_H} \cdot \frac{\Lambda_X}{\lambda_D} \sqrt{\frac{\log m}{n}} \le \frac{\gamma_{\min}^{\star}}{2}\Longrightarrow \gamma_{\min}^o \ge \frac{1}{2} \gamma_{\min}^{\star}.
    \end{equation}
    Thus $\mathrm{sign}(\gamma_S^o) = \mathrm{sign}(\gamma_S^\star)$, and
    \begin{equation*}
        \gamma_{\min}^o \ge \frac{1}{2} \gamma_{\min}^{\star} \ge \frac{2\log s + 5}{\lambda_H \bar{\tau}} > \frac{2\log s + 4 + d(0) / \kappa}{\lambda_H \bar{\tau}} \Longrightarrow \bar{\tau} > \frac{2\log s + 4 + d(0) / \kappa}{\lambda_H \gamma_{\min}^o}.
    \end{equation*}
    By \cref{eq:slbiss-orc-cstc-sign}, the sign consistency of $\gamma_S'(t)$ holds for 
    \begin{equation*}
        t > \inf_{0<\mu<1} \left( \frac{1}{\kappa \lambda_H} \log \frac{1}{\mu} + \frac{2\log s + 4 + d(0) / \kappa}{\lambda_H \gamma_{\min}^o} \right) = \frac{2\log s + 4 + d(0) / \kappa}{\lambda_H \gamma_{\min}^o},
    \end{equation*}
    thus also for $\bar{\tau}$. Then under the No-false-positive property,
    \begin{equation*}
        \mathrm{sign}\left( \gamma_S\left( \bar{\tau} \right) \right) = \mathrm{sign}\left( \gamma_S'\left( \bar{\tau} \right) \right) = \mathrm{sign}\left( \gamma_S^o \right) = \mathrm{sign}\left( \gamma_S^\star \right),
    \end{equation*}
    and
    \begin{equation*}
        \mathrm{sign}\left( \gamma_{S^c}'\left( \bar{\tau} \right) \right) = 0 = \mathrm{sign}\left( \gamma_{S^c}^{\star} \right).
    \end{equation*}

    Now we prove the \emph{$\ell_2$ consistency of $\gamma(t)$}. Under the No-false-positive property, for $0\le t\le \bar{\tau}$,
    \begin{multline*}
        \left\| \gamma(t) - D \beta^{\star} \right\|_2 = \left\| \gamma_S'(t) - \gamma_S^{\star} \right\|_2 \le \left\| d_{\gamma,S}(t) \right\|_2 + \left\| \gamma_S^o - \gamma_S^{\star} \right\|_2\\
        \le d(t) + \sqrt{s} \left\| \gamma_S^o - \gamma_S^{\star} \right\|_{\infty} \le \frac{4 \sqrt{s} + d(0) / \kappa}{\lambda_H t} + \frac{2\sigma}{\lambda_H} \cdot \frac{\Lambda_X}{\lambda_D} \sqrt{\frac{s\log m}{n}}\\
        \le \frac{5\sqrt{s}}{\lambda_H t} + \frac{2\sigma}{\lambda_H}\cdot \frac{\Lambda_X}{\lambda_D} \sqrt{\frac{s\log m}{n}}.
    \end{multline*}

    Finally, we prove the \emph{$\ell_2$ consistency} of $\beta(t)$. Under the No-false-positive property, for $0\le t\le \bar{\tau}$,
    \begin{equation*}
        \left\| \beta(t) - \beta^{\star} \right\|_2 = \left\| \beta'(t) - \beta^{\star} \right\|_2 \le d_{\beta}(t) + \left\| \beta^o - \beta^{\star} \right\|_2\le d(t) + \left\| \beta^o - \beta^{\star} \right\|_2.
    \end{equation*}
    By \Cref{thm:orc-minus-star} (especially noting \cref{eq:delta-gamma-orc-minus-star}), we have
    \begin{multline*}
        \left\| \beta^o - \beta^{\star} \right\|_2 \le \left\| \delta^o - \delta^{\star} \right\|_2 + \left\| \xi^o - \xi^{\star} \right\|_2\\
        \le \left\| \frac{1}{\sqrt{n}} \Lambda_1^{-1} U_1^T \epsilon \right\|_2 + \left( 1 + \left\| \frac{1}{\sqrt{n}} \Lambda_1^{-1} U_1^T X V \right\|_2 \right) \cdot \left\| \delta^o - \delta^{\star} \right\|_2 \le \sqrt{r'} \left\| \frac{1}{\sqrt{n}} \Lambda_1^{-1} U_1^T \epsilon \right\|_{\infty}\\
        + \left( 1 + \frac{\Lambda_X}{\lambda_1} \right) \left( \nu \left\| B^{-1} V^T X^{*} \left( I - U_1 U_1^T \right) \epsilon \right\|_2 + \left\| B^{-1} \Lambda U_S^T \right\|_2 \cdot \sqrt{s} \left\| \gamma_S^o - \gamma_S^{\star} \right\|_{\infty} \right)\\
        \le \sqrt{r'} \left\| \frac{1}{\sqrt{n}} \Lambda_1^{-1} U_1^T \epsilon \right\|_{\infty} + \left( 1 + \frac{\Lambda_X}{\lambda_1} \right) \left( \nu \cdot 2\sigma \cdot \frac{\Lambda_X}{\lambda_D^2} + \frac{1}{\lambda_D}\cdot \sqrt{s}\cdot \frac{2\sigma}{\lambda_H}\cdot \frac{\Lambda_X}{\lambda_D} \sqrt{\frac{\log m}{n}} \right).
    \end{multline*}
    By \cref{eq:tail-subgau-2}, with probability not less than $1 - 2/m$, we have
    \begin{equation*}
        \left\| \frac{1}{\sqrt{n}} \Lambda_1^{-1} U_1^T \epsilon \right\|_{\infty} \le 2\sigma \left\| \frac{1}{\sqrt{n}} \Lambda_1^{-1} U_1^T \right\|_2 \sqrt{\log m} \le \frac{2\sigma}{\lambda_1} \sqrt{\frac{\log m}{n}}.
    \end{equation*}
    In this case, combining the inequalities above with $d(t)\le 5\sqrt{s}/(\lambda_H t)$, the desired result follows. \qed
\end{proof}

\begin{proof}[Proof of \Cref{thm:slbiss-rev-cstc}]
    By the proof details of \Cref{thm:slbiss-cstc}, we know that with probability not less than $1 - 6/m - 3\exp(-4n/5)$, \cref{eq:gamma-orc-minus-star-linf,eq:xi-orc-minus-star-linf,eq:gamma-delta-xi-orc-minus-star-l2,eq:beta-orc-minus-star-l2,eq:d-0-upper} hold, meanwhile the solution path has no false-positive for $0\le t\le \overline{\tau}$. From now, we assume that these properties are all valid.

    First we prove the \emph{sign consistency of $\tilde{\beta}(t)$}. If the $\gamma_{\min}^{\star}$ condition \cref{eq:slbiss-gammamin-cond} holds, then by \Cref{thm:slbiss-cstc}, $S(\bar{\tau}) = S$ holds, and we have
    \begin{equation*}
        D_{S^c} P_{S\left( \bar{\tau} \right)} = D_{S^c} \left( I - D_{S^c}^{\dag} D_{S^c} \right) = 0 \Longrightarrow \mathrm{sign}\left( D_{S^c} \tilde{\beta}\left( \bar{\tau} \right) \right) = 0 = \mathrm{sign} \left( D_{S^c} \beta^{\star} \right).
    \end{equation*}
    To prove $\mathrm{sign}(D_S \tilde{\beta}(\bar{\tau})) = \mathrm{sign}(D_S \beta^{\star})$, note that
    \begin{align*}
        & \left\| D_S \tilde{\beta}\left( \bar{\tau} \right) - D_S \beta^{*} \right\|_{\infty} = \left\| D_S \left( I - D_{S^c}^{\dag} D_{S^c} \right) \left( \beta'\left( \bar{\tau} \right) - \beta^{\star} \right) \right\|_{\infty}\\
        \le{}& \left\| D_S \left( I - D_{S^c}^{\dag} D_{S^c} \right) d_\beta\left( \bar{\tau} \right) \right\|_{\infty} + \left\| D_S \left( 1 - D_{S^c}^{\dag} D_{S^c} \right) \left( \beta^o - \beta^{\star} \right) \right\|_{\infty}\\
        \le{}& \left\| D_S \left( I - D_{S^c}^{\dag} D_{S^c} \right) d_\beta\left( \bar{\tau} \right) \right\|_{\infty} + \left\| \gamma_S^o - \gamma_S^{\star} \right\|_{\infty} + \left\| D_S D_{S^c}^{\dag} D_{S^c} \left( \beta^o - \beta^{\star} \right) \right\|_{\infty}.
    \end{align*}
    First, by \cref{eq:kappa-ge-d0}, $\kappa \ge d(0) \ge \left\| \gamma_S^o \right\|_2 \ge \gamma_{\min}^o$, and
    \begin{equation*}
        \bar{\tau} \ge \frac{\log(8 \Lambda_D)}{\lambda_H \gamma_{\min}^o} + \frac{2 \log s + 5}{\lambda_H \gamma_{\min}^o} \ge \frac{1}{\kappa \lambda_H} \log \left( 8 \Lambda_D \right) + \frac{2\log s + 4 + d(0) / \kappa}{\lambda_H \gamma_{\min}^o}.
    \end{equation*}
    By \cref{eq:slbiss-orc-cstc-sign}, we have $d \left( \bar{\tau} \right) \le \gamma_{\min}^o / (8 \Lambda_D)$, and thus
    \begin{multline*}
        \left\| D_S \left( I - D_{S^c}^{\dag} D_{S^c} \right) d_\beta\left( \bar{\tau} \right) \right\|_{\infty}\\
        \le \left\| D_S \right\|_2 \cdot \left\| I - D_{S^c}^{\dag} D_{S^c} \right\|_2 \cdot \left\| d_\beta\left( \bar{\tau} \right) \right\|_2 \le \Lambda_D \cdot d\left( \bar{\tau} \right) \le \frac{\gamma_{\min}^o}{8}\le \frac{\gamma_{\min}^{\star}}{4}.
    \end{multline*}
    Besides, by \cref{eq:delta-orc-minus-star}, we have
    \begin{equation*}
        D_S D_{S^c}^{\dag} D_{S^c} \left( \beta^o - \beta^{\star} \right) = U_S \Lambda V^T D_{S^c}^{\dag} U_{S^c} \Lambda \left( \delta^o - \delta^{\star} \right) = U_S \Lambda V^T D_{S^c}^{\dag} U_{S^c} \Lambda B_{\delta} \epsilon
    \end{equation*}
    with
    \begin{equation*}
        \left\| U_S \Lambda V^T D_{S^c}^{\dag} U_{S^c} \Lambda B_\delta \right\|_2\le \Lambda_D \left\| D_{S^c}^{\dag} \cdot U_{S^c} \Lambda V^T \right\|_2 \cdot \left\| B_\delta \right\|_2 \le \frac{\Lambda_X \Lambda_D}{\sqrt{n}\cdot \lambda_H \lambda_D^2}.
    \end{equation*}
    By \cref{eq:tail-subgau-2}, with probability not less than $1 - 2/m$,
    \begin{equation*}
        \left\| D_S D_{S^c}^{\dag} D_{S^c} \left( \beta^o - \beta^{\star} \right) \right\|_{\infty} < \frac{2\sigma}{\lambda_H} \cdot \frac{\Lambda_X \Lambda_D}{\lambda_D^2} \sqrt{\frac{\log m}{n}} \le \frac{\gamma_{\min}^{\star}}{4}.
    \end{equation*}
    Finally, we note \cref{eq:gamma-orc-minus-star-cstc-sign}. Then $\mathrm{sign}(D_S \tilde{\beta}(\bar{\tau})) = \mathrm{sign}(D_S \beta^{\star})$ holds, since
    \begin{equation*}
        \left\| D_S \left( \tilde{\beta}\left( \bar{\tau} \right) - \beta^{\star} \right) \right\|_{\infty} < \frac{\gamma_{\min}^{\star}}{4} + \frac{\gamma_{\min}^{\star}}{2} + \frac{\gamma_{\min}^{\star}}{4} = \left( D_S \beta^{\star} \right)_{\min}.
    \end{equation*}

    Then we prove the \emph{$\ell_2$ consistency of $\tilde{\beta}(t)$}. For any $0\le t\le \bar{\tau}$, $S(t) \subseteq S$, which implies $D_{S^c} \tilde{\beta}(t) = D_{S^c} \beta^{\star} = 0$. Then
    \begin{multline*}
        \left\| \tilde{\beta}(t) - \beta^{\star} \right\|_2 \le \left\| V^T \left( \tilde{\beta}(t) - \beta^{\star} \right) \right\|_2 + \left\| V_1^T \tilde{V}^T \left( \tilde{\beta}(t) - \beta^{\star} \right) \right\|_2\\
        \le \left( \left\| V^T P_{S(t)} \left( \beta'(t) - \beta^{\star} \right) \right\|_2 + \left\| V^T \left( I - P_{S(t)} \right) \beta^{\star} \right\|_2 \right)\\
        + \left( \left\| V_1^T \tilde{V}^T P_{S(t)} \left( \beta'(t) - \beta^{\star} \right) \right\|_2 + \left\| V_1^T \tilde{V}^T \left( I - P_{S(t)} \right) \beta^{\star} \right\|_2 \right)\\
        \le \left\| V^T P_{S(t)} \left( \beta'(t) - \beta^{\star} \right) \right\|_2 + \left\| V_1^T \tilde{V}^T P_{S(t)} \left( \beta'(t) - \beta^{\star} \right) \right\|_2 + 2 \left\| D_{S(t)^c}^{\dag} D_{S(t)^c \cap S} \beta^{\star} \right\|_2.
    \end{multline*}
    The first and second term of the right hand side are respectively not greater than
    \begin{multline*}
        \left\| V^T P_{S(t)} d_\beta(t) \right\|_2 + \left\| V^T P_{S(t)} \left( \beta^o - \beta^{\star} \right) \right\|_2 \le \left\| d_\beta(t) \right\|_2 + \frac{1}{\lambda_D} \left\| D P_{S(t)} \left( \beta^o - \beta^{\star} \right) \right\|_2\\
        \le d(t) + \frac{1}{\lambda_D} \left\| D_{S(t)} P_{S(t)} \left( \beta^o - \beta^{\star} \right) \right\|_2\\
        = d(t) + \frac{1}{\lambda_D} \left\| U_{S(t)} \Lambda \left( 1 - V^T D_{S(t)^c}^{\dag} U_{S(t)^c} \Lambda \right) \left( \delta^o - \delta^{\star} \right) \right\|_2
    \end{multline*}
    (here we use the fact that $D_{S(t)^c} P_{S(t)} = 0$), and
    \begin{multline*}
        \left\| V_1^T \tilde{V}^T P_{S(t)} d_\beta(t) \right\|_2 + \left\| V_1^T \tilde{V}^T P_{S(t)} \left( \beta^o - \beta^{\star} \right) \right\|_2\\
        \le \left\| d_\beta(t) \right\|_2 + \left\| \left( \xi^o - \xi^{\star} \right) - V_1^T \tilde{V}^T D_{S(t)^c}^{\dag} D_{S(t)^c} \left( \beta^o - \beta^{\star} \right) \right\|_2\\
        \le d(t) + \left\| \xi^o - \xi^{\star} \right\|_2 + \left\| V_1^T \tilde{V}^T D_{S(t)^c}^{\dag} U_{S(t)^c} \Lambda \left( \delta^o - \delta^{\star} \right) \right\|_2.
    \end{multline*}
    Noting \cref{eq:slbiss-orc-cstc-l2,eq:xi-orc-minus-star-linf}, as well as applying the definition of $B_\delta$ in \Cref{thm:orc-minus-star}, now we only need to show that with probability not less than $1 - 2/m - 2r'/m^2$,
    \begin{align*}
        \left\| U_{S(t)} \Lambda \left( I - V^T D_{S(t)^c}^{\dag} U_{S(t)^c} \Lambda \right) B_\delta \epsilon \right\|_{\infty} &\le \frac{2\sigma}{\lambda_H} \cdot \frac{\Lambda_D \Lambda_X}{\lambda_D^2} \sqrt{\frac{\log m}{n}},\\
        \left\| V_1^T \tilde{V}^T D_{S(t)^c}^{\dag} U_{S(t)^c} \Lambda B_\delta \epsilon \right\|_{\infty} &\le \frac{2\sigma}{\lambda_H} \cdot \frac{\Lambda_X}{\lambda_D^2} \sqrt{\frac{\log m}{n}},
    \end{align*}
    which are both true, according to \cref{eq:tail-subgau-2}, as well as \cref{eq:delta-orc-minus-star} which leads to
    \begin{multline*}
        \left\| U_{S(t)} \Lambda \left( I - V^T D_{S(t)^c}^{\dag} U_{S(t)^c} \Lambda \right) B_\delta \right\|_2\\
        \le \Lambda_D \left( 1 + \left\| V^T D_{S(t)^c}^{\dag} \cdot U_{S(t)^c} \Lambda V^T \right\|_2 \right) \left\| B_\delta \right\|_2 \le \frac{2\Lambda_X\Lambda_D}{\sqrt{n}\cdot \lambda_H \lambda_D^2},
    \end{multline*}
    and
    \begin{equation*}
        \left\| V_1^T \tilde{V}^T D_{S(t)^c}^{\dag} U_{S(t)^c} \Lambda B_\delta \right\|_2 \le \left\| D_{S(t)^c} \cdot U_{S(t)^c} \Lambda V^T \right\|_2 \cdot \left\| B_\delta \right\|_2 \le \frac{\Lambda_X}{\sqrt{n}\cdot \lambda_H \lambda_D^2}. \eqno \mbox{\qed}
    \end{equation*}
\end{proof}

\section{Proof on Consistency of Split LBI}
\label{sec:proof-slbi-cstc}

\begin{proof}[Proof of \Cref{thm:slbi-cstc} and \ref{thm:slbi-rev-cstc}]
    They are merely discrete versions of proofs of \Cref{thm:slbiss-cstc} and \ref{thm:slbiss-rev-cstc}. In the proofs, \Cref{thm:slbi-orc-gbi} and \ref{thm:slbi-orc-cstc} stated below are applied, instead of \Cref{thm:slbiss-orc-gbi} and \ref{thm:slbiss-orc-cstc}. \qed
\end{proof}

Specifically, one can define the \emph{oracle iteration} of Split LBI as an \emph{oracle} version of Split LBI \cref{eq:slbi} (with $S$ known and $\rho_{k,S^c}, \gamma_{k,S^c}$ set to be $0$), resembling the idea of \emph{oracle dynamics} of Split LBISS. Define
\begin{equation*}
    \Psi_k := \| \gamma_S^o \|_1 - \langle \gamma_S^o, \rho_{k,S} \rangle + \| \gamma_{k,S} - \gamma_S^o \|_2^2 / (2\kappa) + \| \beta_k - \beta^o \|_2^2 / (2\kappa).
\end{equation*}
Then we have

\begin{lemma}[\textnormal{Discrete Generalized Bihari's inequality}]
    \label{thm:slbi-orc-gbi}
    Under \Cref{thm:rsc}, suppose $\kappa \alpha \| H \|_2 < 2$ and $\lambda_H' = \lambda_H (1 - \kappa \alpha \|H\|_2 / 2)$. For all $k$ we have
    \begin{equation*}
        \Psi_{k+1} - \Psi_k \le - \alpha \lambda_H' F^{-1}\left( \Psi_k \right),
    \end{equation*}
    where $\gamma_{\min}^o,\ F(x),\ F^{-1}(x)$ are defined the same as in \Cref{thm:slbiss-orc-gbi}.
\end{lemma}

\begin{proof}[Proof of \Cref{thm:slbi-orc-gbi}]
    The proof is almost a discrete version of the continuous case. The only non-trivial thing is to show that
    \begin{gather*}
        \Psi_{k+1} - \Psi_k \le - 2 \alpha \left( 1 - \kappa \alpha \| H \|_2 / 2 \right) L_k,\ \text{where}\\
        L_k := \frac{1}{2} \left( d_{k,\beta}^T, d_{k,\gamma,S}^T \right) H_{(\beta,S), (\beta,S)} \begin{pmatrix} d_{k,\beta}\\ d_{k,\gamma,S} \end{pmatrix},\ \begin{pmatrix} d_{k,\beta}\\ d_{k,\gamma,S} \end{pmatrix} := \begin{pmatrix} \beta_k' - \beta^o\\ \gamma_{k,S}' - \gamma_S^o \end{pmatrix}.
    \end{gather*}
    By \cref{eq:slbi}, we have
    \begin{equation*}
        - \alpha H_{(\beta,S), (\beta,S)} \begin{pmatrix} d_{k,\beta}\\ d_{\gamma,k,S} \end{pmatrix} = \begin{pmatrix} 0\\ \rho_{k+1, S}' - \rho_{k,S}' \end{pmatrix} + \frac{1}{\kappa} \begin{pmatrix} \beta_{k+1}' - \beta_k' \\ \gamma_{k+1,S}' - \gamma_{k,S}' \end{pmatrix}.
    \end{equation*}
    Noting $(\rho_{k+1,S}' - \rho_{k,S}')^T \gamma_{k+1,S}' \ge 0$ and multiplying $(d_{k,\beta}^T, d_{\gamma,k,S}^T)$ on both sides, we have
    \begin{multline*}
        - 2 \alpha L_k = d_{\gamma,k,S}^T \left( \rho_{k+1,S}' - \rho_{k,S}' \right) + \frac{1}{\kappa} \begin{pmatrix} d_{k,\beta}\\ d_{k,\gamma,S} \end{pmatrix}^T \begin{pmatrix} \beta_{k+1}' - \beta_k' \\ \gamma_{k+1,S}' - \gamma_{k,S}' \end{pmatrix}\\
        \ge - \left( \rho_{k+1,S}' - \rho_{k,S}' \right)^T \left( \gamma_{k+1,S}' - \gamma_{k,S}' \right) - \left( \rho_{k+1,S}' - \rho_{k,S}' \right)^T \gamma_S^o\\
        + \frac{1}{\kappa} \begin{pmatrix} d_{k,\beta}\\ d_{k,\gamma,S} \end{pmatrix}^T \begin{pmatrix} \beta_{k+1}' - \beta_k' \\ \gamma_{k+1,S}' - \gamma_{k,S}' \end{pmatrix}.
    \end{multline*}
    Thus
    \begin{multline*}
        \Psi_{k+1} - \Psi_k = - \left( \rho_{k+1,S}' - \rho_{S,k}' \right)^T \gamma_S^o + \frac{1}{2\kappa} \left( \left\| \begin{pmatrix} d_{k+1,\beta}\\ d_{k+1,\gamma,S} \end{pmatrix} \right\|_2^2 - \left\| \begin{pmatrix} d_{k,\beta}\\ d_{k,\gamma,S} \end{pmatrix} \right\|_2^2 \right)\\
        = - \left( \rho_{k+1,S}' - \rho_{S,k}' \right)^T \gamma_S^o + \frac{1}{2\kappa} \begin{pmatrix} \beta_{k+1}' - \beta_k' \\ \gamma_{k+1,S}' - \gamma_{k,S}' \end{pmatrix}^T \left( \begin{pmatrix} \beta_{k+1}' - \beta_k' \\ \gamma_{k+1,S}' - \gamma_{k,S}' \end{pmatrix} + 2 \begin{pmatrix} d_{k,\beta} \\ d_{k,\gamma,S} \end{pmatrix} \right)\\
        \le - 2\alpha L_k + \left( \rho_{S,k+1}' - \rho_{S,k}' \right)^T \left( \gamma_{k+1,S}' - \gamma_{k,S}' \right) + \frac{1}{2\kappa} \left\| \begin{pmatrix} \beta_{k+1}' - \beta_k' \\ \gamma_{k+1,S}' - \gamma_{k,S}' \end{pmatrix} \right\|_2^2\\
        \le - 2\alpha L_k + \frac{\kappa}{2} \left\| \begin{pmatrix} 0\\ \rho_{k+1,S}' - \rho_{k,S}' \end{pmatrix} + \frac{1}{\kappa} \begin{pmatrix} \beta_{k+1}' - \beta_k' \\ \gamma_{k+1,S}' - \gamma_{k,S}' \end{pmatrix} \right\|_2^2\\
        = - \left( d_{k,\beta}^T, d_{k,\gamma,S}^T \right) \left( \alpha H_{(\beta,S), (\beta,S)} - \frac{\kappa \alpha^2}{2} H_{(\beta,S), (\beta,S)}^2 \right) \begin{pmatrix} d_{k,\beta}\\ d_{k,\gamma,S} \end{pmatrix}\\
        \le - \alpha \left( 1 - \frac{\kappa\alpha}{2} \left\| H_{(\beta,S), (\beta,S)} \right\|_2 \right) \left( d_{k,\beta}^T, d_{k,\gamma,S}^T \right) H_{(\beta,S), (\beta,S)} \begin{pmatrix} d_{k,\beta}\\ d_{k,\gamma,S} \end{pmatrix}\\
        \le - 2\alpha\left( 1 - \kappa\alpha \|H\|_2 / 2 \right) L_k. \hfill\mbox{\qed}
    \end{multline*}
\end{proof}

\begin{lemma}
    \label{thm:slbi-orc-cstc}
    Under \Cref{thm:rsc}, suppose $\kappa \alpha \| H \|_2 < 2$ and $\lambda_H' = \lambda_H (1 - \kappa \alpha \|H\|_2 / 2)$. Let
    \begin{gather*}
        \gamma_{\min}^o := \min(|\gamma_j^o|:\ \gamma_j^o\neq 0),\\
        d_{k,\beta} = \beta_k' - \beta^o,\ d_{k,\gamma} = \gamma_k' - \gamma^o,\ d_k = \sqrt{\left\| d_{k,\beta} \right\|_2^2 + \left\| d_{k,\gamma,S} \right\|_2^2}.
    \end{gather*}
    Then for any $k$ such that
    \begin{equation}
        \label{eq:tau-inf-def-slbi}
        k \alpha \ge \tau_{\infty}'(\mu) := \frac{1}{\kappa \lambda_H'} \log \frac{1}{\mu} + \frac{2\log s + 4 + d_0 / \kappa}{\lambda_H' \gamma_{\min}^o} + 4\alpha\ (0<\mu<1),
    \end{equation}
    we have
    \begin{equation}
        \label{eq:slbi-orc-cstc-sign}
        d_k \le \mu \gamma_{\min}^o \left( \Longrightarrow \mathrm{sign}\left( \gamma_{k,S}' \right) = \mathrm{sign}\left( \gamma_S^o \right),\ \text{if $\gamma_j^o \neq 0$ for $j\in S$} \right).
    \end{equation}
    For any $k$, we have
    \begin{equation}
        \label{eq:slbi-orc-cstc-l2}
        d_k \le \min \left( \frac{4 \sqrt{s} + d_0 / \kappa}{\lambda_H' k \alpha},\ \sqrt{\frac{2\left( 1 + \nu \Lambda_X^2 + \Lambda_D^2 \right)}{\lambda_H' \nu}} \cdot d_0 \right).
    \end{equation}
\end{lemma}

\begin{proof}[Proof of \Cref{thm:slbi-orc-cstc}]
    The proof is almost a discrete version of the continuous case. The only non-trivial thing is described as follows. First, suppose there does not exist $k \le \tau_{\infty}'(\mu) / \alpha$ satisfying \cref{eq:slbi-orc-cstc-sign}, then for any $0\le k \alpha \le \tau_{\infty}'(\mu)$, we have $\Psi_k > \mu^2 (\gamma_{\min}^o)^2 / (2\kappa)$. Letting $k_0 = 0$, then $\Psi_{k_0} = \Psi_0 \le F(d_0^2)$. Suppose that
    \begin{multline*}
        F\left( d_0^2 \right) \ge \Psi_{k_0}, \ldots, \Psi_{k_1-1} > F\left( s \left( \gamma_{\min}^o \right)^2 \right) \ge \Psi_{k_1}, \ldots, \Psi_{k_2-1} > F\left( \left( \gamma_{\min}^o \right)^2 \right)\\
        \ge \Psi_{k_2}, \ldots, \Psi_{k_3-1} > \left( \gamma_{\min}^o \right)^2 / (2\kappa) \ge \Psi_{k_3}, \ldots, \Psi_{k_4-1} > \mu^2 \left( \gamma_{\min}^o \right)^2 / (2\kappa)\ge \Psi_{k_4}, \ldots
    \end{multline*}
    Then $k_4 \alpha > \tau_{\infty}'(\mu)$. Besides, by \Cref{thm:slbi-orc-gbi},
    \begin{equation*}
        \alpha \le \frac{\Psi_k - \Psi_{k+1}}{\lambda_H' F^{-1}(\Psi_k)}\ (0\le k\alpha\le \tau_{\infty}'(\mu)).
    \end{equation*}
    Thus $\lambda_H' (k_4-4) \alpha$ is not greater than
    \begin{align*}
        & \left( \sum_{k=k_3}^{k_4-2} + \sum_{k=k_2}^{k_3-2} + \sum_{k=k_1}^{k_2-2} + \sum_{k=k_0}^{k_1-2} \right) \frac{\Psi_k - \Psi_{k+1}}{F^{-1}(\Psi_k)} \le \sum_{k=k_3}^{k_4-2} \frac{\Psi_k - \Psi_{k+1}}{2\kappa \Psi_k} + \sum_{k=k_2}^{k_3-2} \frac{\Psi_k - \Psi_{k+1}}{\left( \gamma_{\min}^o \right)^2}\\
        & + \sum_{k=k_1}^{k_2-2} \frac{F(\Delta_k) - F(\Delta_{k+1})}{\Delta_k} + \sum_{k=k_0}^{k_1-2}\frac{F(\Delta_k) - F(\Delta_{k+1})}{\Delta_k}\ \left( \Delta_k := F^{-1}(\Psi_k) \right)\\
        ={} & \sum_{k=k_3}^{k_4-2} \frac{\Psi_k - \Psi_{k+1}}{2\kappa \Psi_k} + \sum_{k=k_2}^{k_3-2} \frac{\Psi_k - \Psi_{k+1}}{\left( \gamma_{\min}^o \right)^2} + \sum_{k=k_1}^{k_2-2} \left( \frac{\Delta_k - \Delta_{k+1}}{2\kappa \Delta_k} + \frac{2(\Delta_k - \Delta_{k+1})}{\gamma_{\min}^o \Delta_k} \right)\\
        & + \sum_{k=k_0}^{k_1-2} \left( \frac{\Delta_k - \Delta_{k+1}}{2\kappa \Delta_k} + \frac{2\sqrt{s}\left( \sqrt{\Delta_k} - \sqrt{\Delta_{k+1}} \right)}{\Delta_k} \right).
    \end{align*}
    By $(u-v)/u \le \log (u/v)$ and $(\sqrt{u} - \sqrt{v})/u \le 1/\sqrt{v} - 1/\sqrt{u}$ for $u \ge v > 0$, the quantity above is not greater than
    \begin{multline*}
        \frac{\log \left( \Psi_{k_3} / \Psi_{k_4-1} \right)}{2\kappa} + \frac{\Psi_{k_2} - \Psi_{k_3-1}}{\left( \gamma_{\min}^o \right)^2}\\
        + \frac{\log \left( \Delta_{k_0} / \Delta_{k_2-1} \right)}{2\kappa} + \frac{2 \log \left( \Delta_{k_1} / \Delta_{k_2-1} \right)}{\gamma_{\min}^o} + 2\sqrt{s} \left( \frac{1}{\sqrt{\Delta_{k_1-1}}} - \frac{1}{\sqrt{\Delta_{k_0}}} \right)\\
        < \frac{\log \left( 1 / \mu^2 \right)}{2\kappa} + \frac{2 \gamma_{\min}^o}{\left( \gamma_{\min}^o \right)^2} + \frac{\log \left( d_0^2 / \left( \gamma_{\min}^o \right)^2 \right)}{2\kappa} + \frac{2\log s}{\gamma_{\min}^o} + \frac{2\sqrt{s}}{\sqrt{s \left( \gamma_{\min}^o \right)^2}}.
    \end{multline*}
    Therefore we get
    \begin{equation*}
        \lambda_H' \left( \tau_{\infty}'(\mu) - 4\alpha \right) < \lambda_H' \left( k_4-4 \right)\alpha < \frac{1}{\kappa} \log \frac{1}{\mu} + \frac{2\log s + 4 + d_0 / \kappa}{\gamma_{\min}^o},
    \end{equation*}
    a contradiction with the definition of $\tau_{\infty}'(\mu)$. So there exists some $k \le \tau_{\infty}'(\mu) / \alpha$ satisfying \cref{eq:slbi-orc-cstc-sign}. Then continue to imitate the proof in the continous version, we obtain \cref{eq:slbi-orc-cstc-sign} for all $k\ge \tau_{\infty}'(\mu) / \alpha$. The proof of \cref{eq:slbi-orc-cstc-l2} follows the same spirit. \qed
\end{proof}

\section*{References}

\bibliographystyle{elsarticle-harv}

\bibliography{ref_local}

\end{document}